\documentclass[11pt]{article}
\usepackage[utf8]{inputenc} 
\usepackage[T1]{fontenc} 

\usepackage{url}            
\usepackage{booktabs}       
\usepackage{amsfonts}       
\usepackage{nicefrac}       
\usepackage{microtype}      
\usepackage[dvipsnames]{xcolor}

\usepackage{natbib}
\usepackage{enumitem}
\usepackage{amsthm}
\usepackage{amssymb}
\usepackage{amsmath}
\usepackage{mathrsfs}

\usepackage{hyperref}
\hypersetup{
    pdftoolbar=true,        
    pdfmenubar=true,        
    pdffitwindow=false,     
    pdfstartview={FitH},    
    pdfnewwindow=true,      
    colorlinks=true,       
    linkcolor=blue,          
    citecolor=blue,        
    filecolor=magenta,      
    urlcolor=red,           
    breaklinks=false,
}

\usepackage{graphicx}
\usepackage{cleveref}

\usepackage{fullpage}

\numberwithin{equation}{section}

\newcommand{\e}{\varepsilon}

\newcommand{\diag}{\textnormal{diag}}

\newcommand{\E}{\mathbf{E}}

\DeclareMathOperator{\argmin}{argmin}

\DeclareMathOperator{\rank}{rank}

\DeclareMathOperator{\tr}{tr}

\DeclareMathOperator*{\esssup}{ess\,sup}

\newcommand{\ceil}[1]{\lceil #1 \rceil}
\newcommand{\bigceil}[1]{\left\lceil #1 \right\rceil}
\newcommand{\T}{\mathsf{T}}

\newcommand{\calN}{\mathcal{N}}
\newcommand{\calZ}{\mathcal{Z}}

\newcommand{\calE}{\mathcal{E}}
\newcommand{\scrA}{\mathscr{A}}

\newcommand{\scrF}{\mathscr{F}}
\newcommand{\scrP}{\mathscr{P}}
\newcommand{\sfX}{\mathsf{X}}
\newcommand{\sfY}{\mathsf{Y}}
\newcommand{\sfZ}{\mathsf{Z}}
\newcommand{\sfP}{\mathsf{P}}
\newcommand{\sfM}{\mathsf{M}}
\newcommand{\Px}{\sfP_{X}}
\newcommand{\Pxbar}{\sfP_{\bar{X}}}
\newcommand{\floor}[1]{\lfloor #1 \rfloor}
\newcommand{\dx}{d_{\mathsf{X}}}
\newcommand{\dy}{d_{\mathsf{Y}}}
\newcommand{\dv}{d_{\mathsf{V}}}
\newcommand{\R}{\mathbb{R}}
\newcommand{\N}{\mathbb{N}}
\renewcommand{\Pr}{\mathbf{P}}

\newcommand{\iid}{iid}
\newcommand{\Gammadep}{\Gamma_{\mathsf{dep}}}

\newcommand{\abs}[1]{| #1 |}
\newcommand{\bigabs}[1]{\left| #1 \right|}

\newcommand{\norm}[1]{\lVert #1 \rVert}
\newcommand{\bignorm}[1]{\left\lVert #1 \right\rVert}
\newcommand{\ind}{\mathbf{1}}
\newcommand{\ip}[2]{\ensuremath{\langle #1, #2 \rangle}}

\newcommand{\tvnorm}[1]{\| #1 \|_{\mathsf{TV}}}
\newcommand{\opnorm}[1]{\| #1 \|_{\mathsf{op}}}
\newcommand{\bigopnorm}[1]{\left\| #1 \right\|_{\mathsf{op}}}
\newcommand{\mumin}{\underline{\mu}}

\newcommand{\rmd}{\mathrm{d}}

\newtheorem{definition}{Definition}[section] 
\newtheorem{theorem}{Theorem}[section] 
\newtheorem{proposition}{Proposition}[section] 
\newtheorem{corollary}{Corollary}[section] 
\newtheorem{assumption}{Assumption}[section] 

\newtheorem{lemma}{Lemma}[section]

\title{Learning with little mixing}

\usepackage{authblk}
\author[1]{Ingvar Ziemann}
\author[2]{Stephen Tu}
\affil[1]{KTH Royal Institute of Technology}
\affil[2]{Google Brain Robotics}
\date{June 16, 2022, Revised: \today}

\begin{document}

\maketitle

\begin{abstract}
We study square loss in a realizable time-series framework with martingale difference noise. Our main result is a fast rate excess risk bound
which shows that whenever a
\emph{trajectory hypercontractivity} condition holds,
the risk of the least-squares estimator on dependent data matches the 
\iid\ rate order-wise after a burn-in time.
In comparison, many existing results in learning from dependent data 
have rates where the effective sample size is deflated by a factor of the mixing-time of the underlying process, even after the burn-in time.
Furthermore, our results allow the covariate process to exhibit long range correlations which are substantially weaker than geometric ergodicity.
We call this phenomenon \emph{learning with little mixing},
and present several examples for when it occurs:
bounded function classes for
which the $L^2$ and $L^{2+\e}$ norms are equivalent, 
ergodic finite state
Markov chains,
various parametric models,
and a broad family of infinite dimensional 
$\ell^2(\mathbb{N})$ ellipsoids. 
By instantiating our main result to system identification
of nonlinear dynamics with generalized linear model 
transitions, we obtain a nearly minimax optimal 
excess risk bound after only a polynomial burn-in time.

%


\end{abstract}
\section{Introduction}

Consider regression in the context of the time-series model:
\begin{align}
    \label{eq:ts}
    Y_t = f_\star(X_t) + W_t, \qquad t=0,1,2,\dots.
\end{align}
Such models are ubiquitous in applications of machine learning, signal processing, econometrics, and control theory. In our setup, the learner is given access to
$T \in \N_+$ pairs $\{(X_t, Y_t)\}_{t=0}^{T-1}$ drawn from the model \eqref{eq:ts}, and is asked to output a hypothesis $\widehat{f}$ from a hypothesis class $\mathscr{F}$
which best approximates the (realizable) regression function $f_\star \in \mathscr{F}$ in terms of square loss.

In this work, we study the least-squares estimator (LSE). This procedure  minimizes the empirical risk associated to the square loss over the class $\mathscr{F}$. When each pair of observations $(X_t,Y_t)$ is drawn \iid\ from some fixed distribution,
this procedure is minimax optimal over a
broad set of hypothesis classes~\citep{tsybakov2009introduction,lecue2013learning,mendelson2014learning,wainwright2019high}. However, 
much less is known about the optimal rate of convergence
for the general time-series model~\eqref{eq:ts},
as correlations across time in the covariates $\{X_t\}$ complicate the analysis.

With this in mind, we seek to extend our understanding of the minimax optimality of the LSE for the time-series model \eqref{eq:ts}. We show that for a broad class of function
spaces and covariate processes,
the effects of data dependency across time enter the LSE excess risk only
as a higher order term, whereas the leading term in the excess risk remains order-wise identical
to that in the \iid\ setting. Hence, after a sufficiently long,
but finite \emph{burn-in time}, the LSE's excess risk scales as if all $T$ samples are independent. 
This behavior applies to processes that
exhibit correlations which decay slower than 
geometrically. We refer to this double phenomenon,
where the mixing-time only enters as a burn-in time, and
where the mixing requirement is mild, as
\emph{learning with little mixing}.

Our result stands in contrast to a long line of work
on learning from dependent data (see e.g.,\ \citep{yu1994mixing,modha1996regression,mohri2008rademachermixing,steinwart2009fastlearningmixing,zou2009mixing,agarwal2012generalization,duchi2012ergodicmd,kuznetsov2017generalization,roy2021dependent,sancetta2021rkhs} and the references within), 
where the blocking technique \citep{yu1994mixing} is used to create independence
amongst the dependent covariates, so that tools to analyze
independent learning can be applied.
While these aforementioned works differ in their specific setups,
the main commonality is that the resulting dependent data rates mimic
the corresponding independent rates, but with the caveat that the
sample size is replaced by an ``effective'' sample size that is \emph{decreased} in some way by the mixing-time, even after any necessary burn-in time. Interestingly,
the results of \citet{ziemann2022single} studying the LSE
on the model \eqref{eq:ts}
also suffer from such sample degradation, but do not rely on the blocking technique.





The model \eqref{eq:ts} captures learning dynamical systems
by setting $Y_{t} = X_{t+1}$, so that the regression function
$f_\star$ describes the dynamics of the state variable $X_t$.
Recent progress in system identification
shows that the lack of ergodicity does not necessarily degrade
learning rates.
Indeed, when the states evolve as a 
linear dynamical system (i.e., the function $f_\star$ is linear),
learning rates are not deflated by any mixing times, and match
existing rates for \iid\ linear regression~\citep{simchowitz2018learning,faradonbeh2018finite,sarkar2019near,tu2022learning}.
%
\citet{kowshik2021near,gao2021simam} extend results of this flavor to parameter
recovery of dynamics driven by a generalized linear model.
The extent to which this phenomenon---less ergodicity not impeding learning---generalizes beyond linear and generalized linear models is a
key motivation for our work.


\paragraph{Contributions}

We consider the realizable setting, where $f_\star$ is assumed to 
be contained in a known function space $\scrF$.
Our results rest on two assumptions regarding both the covariate process
$\{X_t\}$ and the function space $\scrF$.
The first assumption posits that
the process $\{X_t\}$ exhibits some mild form of ergodicity
(that is significantly weaker than the typical geometric 
ergodicity assumption). The second assumption is a hypercontractivity condition that holds uniformly in $\scrF$
along the trajectory $\{X_t\}$, extending contractivity
assumptions for \iid\ learning~\citep{mendelson2014learning} 
to dependent processes.

Informally, our main result (\Cref{thm:themainthm}, presented in \Cref{sec:results}),
shows that under these two assumptions, 
letting $\textnormal{comp}(\scrF)$ denote some (inverse) measure of complexity of $\scrF$,
the LSE $\widehat{f}$ satisfies:
\begin{align}
\label{eq:thm1inbrief}
    \E \|\widehat f -f_\star\|_{L^2}^2 \lesssim \left( \frac{\textnormal{dimensional factors}\times\sigma_W^2}{T} \right)^{\textnormal{comp}(\scrF)} + \textnormal{higher order $o(1/T^{\textnormal{comp}(\scrF)})$ terms}.
\end{align}
The first term in \eqref{eq:thm1inbrief} matches
existing LSE risk bounds for \iid\ learning order-wise, and most importantly,
does not include any dependence on the mixing-time of the process.
Indeed, all mixing-time dependencies enter only in the higher order term.
Since this term scales as $o(1/T^{\textnormal{comp}(\scrF)})$,
it becomes negligible after a finite burn-in time. 
This captures the crux of our results:
on a broad class of problems, given enough data,
the LSE applied to time-series model \eqref{eq:ts} behaves as if all samples are independent.

\Cref{sec:examples} provides several examples for which the trajectory
hypercontractivity assumption holds.
When the covariate process $\{X_t\}$
is generated by a finite-state irreducible and aperiodic Markov chain, 
then any function class $\scrF$ satisfies the requisite condition.
More broadly, the condition is satisfied for any bounded function classes for which the $L^2$ and $L^{2+\e}$ 
norms (along trajectories)
are equivalent. Next, we show that many infinite dimensional function
spaces based on $\ell^2(\N)$ ellipsoids satisfy our hypercontractivity condition, demonstrating that
our results are not inherently limited to finite-dimensional
hypothesis classes.

To demonstrate the broad applicability of our framework, \Cref{sec:results:comparison} instantiates our main result on two
system identification problems that have received recent attention
in the literature: linear dynamical systems (LDS),
and systems with generalized linear model (GLM) transitions. 
For stable LDS, after a polynomial burn-in time, we recover an excess risk
bound that matches the \iid\ rate. A more general form of this result was recently established by   \citet{tu2022learning}. 
For stable GLMs, also after a polynomial burn-in time,
we obtain the first excess risk bound for this problem
which matches the \iid\ rate, up to logarithmic factors in various
problem constants including the mixing-time.
In both of these settings, our excess risk bounds also yield
nearly optimal rates for parameter recovery, matching
known results for LDS~\citep{simchowitz2018learning} and
GLMs~\citep{kowshik2021near} in the stable case.
In \Cref{sec:experiments}, we show experimentally, 
using the stable GLM model, that
the trends predicted by our theory are indeed realized
in practice.


\section{Related work}
\label{sec:related}

While regression is a fundamental problem studied across many
disciplines, our work draws its main inspiration from \citet{mendelson2014learning} and \citet{simchowitz2018learning}.
\citet{mendelson2014learning} shows that for nonparametric \iid\ regression, only minimal assumptions are required for one-sided isometry, and thus learning. \citet{simchowitz2018learning} build on this intuition and provide mixing-free rates for linear regression over trajectories generated by a linear dynamical system. 
We continue this trend, by leveraging one-sided isometry to show that mixing only enters as a higher order term in the rates of the nonparametric LSE. More broadly, the technical developments we follow synthesize techniques from two lines of work: nonparametric regression with \iid\ data, and learning from dependent data. 

\paragraph{Nonparametric regression with \iid\ data}
Beyond the seminal work of \citet{mendelson2014learning}, the works \citep{audibert2011robust,lecue2013learning,mendelson2017aggregation} all study \iid\ regression with square loss under various moment equivalence conditions.  
In addition to moment equivalence, we build on
the notion of offset Rademacher complexity defined by \citet{liang2015learning} in the context of \iid\ regression. Indeed, we show that a martingale analogue of the offset complexity (described in~\citet{ziemann2022single}) characterizes the LSE rate in~\eqref{eq:ts}.
%


\paragraph{Learning from dependent data}
As discussed previously, many existing results for learning
from dependent data reduce the problem to independent
learning via the blocking technique~\citep{yu1994mixing},
at the expense of sample complexity deflation by the mixing-time.
\citet{nagaraj2020least} prove a lower bound for linear regression stating that in a worst case agnostic model,
this deflation is unavoidable.
Moreover, if the linear regression problem is realizable,
\citet{nagaraj2020least} provide upper and lower bounds showing
that the mixing-time only affects
the burn-in time, but not the final risk. We note that their upper bound is 
an algorithmic result that holds only for a
specific modification of SGD.
Our work can be interpreted as 
an upper bound in the more general
nonparametric setting,
where we put forth sufficient conditions to
recover the \iid\ rate after a burn-in time. Our result is algorithm agnostic and directly
applies to the empirical risk minimizer.
\citet{ziemann2022single} also study the model~\eqref{eq:ts}, and provide an information-theoretic analysis of the nonparametric LSE.
However, their approach fundamentally reduces to showing two-sided
concentration---something our work evades---and therefore their bounds incur worst case dependency on the mixing-time.  
%
%
\citet{roy2021dependent} extend the results from \citet{mendelson2014learning} to the dependent data setting. 
While following Mendelson's argument
allows their results to handle
non-realizability and heavy-tailed noise, 
their proof ultimately still 
relies on
two-sided concentration for both the
``version space'' and the ``noise interaction''. Hence,
their rates end up 
degrading for slower mixing processes.
We note that this is actually expected in the non-realizable setting 
in light of the lower bounds in \citet{nagaraj2020least}.


The measure of dependencies we use 
for the process $\{X_t\}$ is due to \citet{samson2000concentration}.
Recently, \citet{dagan2019learning} use a similar measure
to study learning when the covariates have no obvious 
sequential ordering (e.g.,\ a graph structure or Ising model).
However, our results are not directly comparable, other than 
noting that their risk bounds 
degrade as the measure of correlation increases.

Results in linear system identification show that
lack of ergodicity does not degrade parameter recovery rates~\citep{simchowitz2018learning,faradonbeh2018finite,rantzer2018adaptive,sarkar2019near,jedra2020finite,tu2022learning}.
Beyond linear system identification, 
\citet{sattar2020non,foster2020learning,kowshik2021near,gao2021simam}
prove parameter recovery bounds for dynamical systems
driven by a generalized linear model (GLM) transition. Most relevant
are \citet{kowshik2021near} and \citet{gao2021simam}, who again show that the lack of
ergodicity does not hamper rates. Indeed, \citet{gao2021simam} even manage to do so in 
a semiparametric setting with an unknown link function.
As mentioned previously, our main result instantiated to these problems
in the stable case matches existing excess risk and parameter recovery
bounds for linear system identification, and actually provides 
the sharpest known excess risk bound for the GLM setting
(when the link function is known).
A more detailed comparison
to existing LDS results is given in \Cref{sec:results:lds},
and to existing GLM results in \Cref{sec:appendix:glm:comparison}.

\section{Problem formulation}
\label{sec:probform}




The time-series (\ref{eq:ts}) evolves on two subsets of Euclidean space, $\mathsf{X}\subset \mathbb{R}^{\dx}$ and $\mathsf{Y} \subset \mathbb{R}^{\dy}$, with $X_t \in \mathsf{X}$ and $Y_t,W_t \in \mathsf{Y}$. 
Expectation (resp.\ probability) with respect to all the randomness of the underlying
probability space is denoted by $\E$ (resp.\ $\Pr$).
The Euclidean norm on $\mathbb{R}^{d}$ is denoted $\|\cdot\|_2$,
and the unit sphere in $\R^d$ is denoted $\mathbb{S}^{d-1}$.
For a matrix $M \in \R^{d_1 \times d_2}$,
$\opnorm{M}$ denotes the largest singular value,
$\sigma_{\min}(M)$ the smallest non-zero singular value,
and $\mathrm{cond}(M) = \opnorm{M}/\sigma_{\min}(M)$ the condition
number. When the matrix $M$ is symmetric, $\lambda_{\min}(M)$ will be used
to denote its minimum eigenvalue.

We assume there exists a filtration $\{\mathcal{F}_t \}$ such that (a) $\{W_t\}$ is a square integrable martingale difference sequence (MDS) with respect to this filtration, and (b) $\{X_t\}$ is adapted to $\{\mathcal{F}_{t-1}\}$. Further tail conditions on this MDS will be imposed as necessary later on. 

Let $\scrF$ be a hypothesis space of functions mapping
$\R^{\dx}$ to $\R^{\dy}$. We assume that the true regression function is an element of $\mathscr{F}$
(i.e.,\ $f_\star \in \scrF$), and that $\scrF$ is known to the learner.
Given two compatible function spaces $\scrF_1, \scrF_2$, let
$\scrF_1 - \scrF_2 \triangleq \{ f_1 - f_2 \mid f_1 \in \scrF_1, f_2 \in \scrF_2 \}$. A key quantity in our analysis is the shifted function class
$\scrF_\star \triangleq \scrF - \{ f_\star \}$. Our results will be stated under
the assumption that $\scrF_\star$ is \emph{star-shaped},\footnote{A function class $\scrF$ is star-shaped if for any $\alpha \in [0, 1]$, $f \in \scrF$ implies $\alpha f \in \scrF$.
}
although we will see that this is not too restrictive.
For any function $f : \sfX \rightarrow \R^{\dy}$, we define
$\norm{f}_\infty \triangleq \sup_{x \in \sfX} \norm{f}_2$.
A function $f$ is $B$-bounded if $\| f\|_{\infty} \leq B$. Similarly, a hypothesis class is $B$-bounded if each of its elements is $B$-bounded.
For a bounded class $\scrF$ and
resolution $\varepsilon > 0$, the quantity
$\calN_\infty(\scrF, \varepsilon)$ denotes the size of the minimal
$\varepsilon$-cover of $\scrF$ (contained in $\scrF$) in the $\norm{\cdot}_\infty$-norm.

We fix a $T \in \N_+$, indicating the number of labeled observations $\{(X_t, Y_t)\}_{t=0}^{T-1}$ from the
time-series \eqref{eq:ts} that are available to the learner.
The joint distribution of $X_{0:T-1}\triangleq (X_0,\dots,X_{T-1})$
is denoted $\Px$. For $p \geq 1$, we endow $\mathscr{F}-\mathscr{F}$ with $L^p(\mathsf{P}_X)$ norms: $\|f-g\|_{L^p}^p\triangleq \frac{1}{T} \sum_{t=0}^{T-1}\E \|f(X_t)-g(X_t)\|_2^p$, where expectation is taken with respect to $\mathsf{P}_X$. 
We will mostly be interested in $L^2(\mathsf{P}_X)$, hereafter often just referred to as $L^2$.  This is the $L^2$ space associated to the law of the uniform mixture over $X_{0:T-1}$ and thus, for \iid\ data, coincides with the standard $L^2$ space often considered in \iid\ regression. 
For a radius $r > 0$, we let 
$B(r)$ denote the closed ball of $\scrF_\star$ with radius $r$ in $L^2$,
and we let $\partial B(r)$ denote its boundary:
\begin{align*}
    B(r) \triangleq \left\{ f \in \scrF_\star \,\bigg|\, \frac{1}{T} \sum_{t=0}^{T-1} \E\norm{f(X_t)}_2^2 \leq r^2 \right\}, \:\: \partial B(r) \triangleq \left\{ f \in \scrF_\star \,\bigg|\, \frac{1}{T} \sum_{t=0}^{T-1} \E\norm{f(X_t)}_2^2 = r^2 \right\}.
\end{align*}

The learning task is to produce an estimate $\widehat f$ of $f_\star$, which renders the excess risk $\|\widehat f-f_\star\|_{L^2}^2$ as small as possible. We emphasize that $\|\widehat f-f_\star\|_{L^2}^2 = \frac{1}{T} \sum_{t=0}^{T-1}\E_{\tilde X_{0:T-1}} \|\widehat f(\tilde X_t)-f_\star(\tilde X_t)\|_2^2$ where $\tilde X_{0:T-1}$ is a fresh, statistically independent, sample with the same law $\mathsf{P}_X$ as $X_{0:T-1}$. Namely, $\|\widehat f-f_\star\|_{L^2}^2$ is a random quantity, still depending on the internal randomness of the learner
and that of the sample $X_{0:T-1}$ used to generate $\widehat f$. We
study the performance of the least-squares estimator (LSE) defined as
$\widehat f \in \argmin_{f\in \mathscr{F}} \left\{\frac{1}{T}\sum_{t=0}^{T-1} \| Y_t -f(X_t)\|_2^2\right\}$,
and measure the excess risk
$\E \norm{\widehat{f} - f_\star}_{L^2}^2$.

\section{Results}
\label{sec:results}

This section presents our main result.
We first detail the definitions behind our main assumptions
in \Cref{sec:prels}.
The main result and two corollaries are then presented in~\Cref{sec:littlemix}. 

\subsection{Hypercontractivity and the dependency matrix}
\label{sec:prels}

\paragraph{Hypercontractivity} We first 
state our main trajectory hypercontractivity condition,
which we will use to establish lower isometry.
The following definition is heavily inspired by recent work on learning without concentration~\citep{mendelson2014learning,mendelson2017aggregation}.

\begin{definition}[Trajectory $(C,\alpha)$-hypercontractivity]
\label{def:hypconrelax}
Fix constants $C > 0$ and $\alpha \in [1,2]$. We say that the tuple $(\mathscr{F},\mathsf{P}_X)$ satisfies the \emph{trajectory $(C,\alpha)$-hypercontractivity} condition if
\begin{align}
\label{cond:hyperconttrajrelax}
  \E\left[ \frac{1}{T}\sum_{t=0}^{T-1} \| f(X_t)\|^4_2\right] \leq C \left(\E \left[ \frac{1}{T}\sum_{t=0}^{T-1}  \| f(X_t)\|_2^2\right]\right)^\alpha \textnormal{for all $f\in\mathscr{F}$}.
\end{align}
Here, the expectation is with respect to $\mathsf{P}_X$, the joint law of $X_{0:T-1}$. 
\end{definition}

Condition (\ref{cond:hyperconttrajrelax}) 
interpolates between boundedness and small-ball behavior. Indeed, if the class $\mathscr{F}$ is $B$-bounded, then it satisfies trajectory $(B^2, 1)$-hypercontractivity trivially. 
On the other hand, for $\alpha=2$,  \eqref{cond:hyperconttrajrelax} asks that 
$\| f\|_{L^4} \leq C^{1/4} \|f\|_{L^2}$ for trajectory-wise $L^p$-norms; by the Paley-Zygmund inequality, this 
implies that
a small-ball condition holds. Moreover, if for some $\e \in (0, 2)$, the trajectory-wise $L^2$ and $L^{2+\e}$ norms are equivalent on $\mathscr{F}$, then \Cref{prop:momeqv} 
(\Cref{sec:examples}) shows that the condition holds for a nontrivial $\alpha=1+\e/2 \in (1, 2)$. 
More examples are given in \Cref{sec:examples}.

Our main results assume that $(\scrF_\star, \Px)$
(or a particular subset of $\scrF_\star$) satisfies the trajectory
$(C, \alpha)$-hypercontractivity condition with
$\alpha > 1$, which
we refer to as the \emph{hypercontractive} regime.
The condition $\alpha > 1$ is required in our analysis
for the lower order excess risk term 
to not depend on the mixing-time.
Our results instantiated for the $\alpha=1$ case
directly correspond to existing work by \citet{ziemann2022single}, and
exhibit a lower order term that depends on the mixing-time.



\paragraph{Ergodicity via the dependency matrix}

We now state the main definition we use to measure 
the stochastic dependency of a process.
Recall that for two measures $\mu, \nu$ on the same measurable space
with $\sigma$-algebra $\mathcal{A}$,
the total-variation norm is defined as $\tvnorm{\mu - \nu} \triangleq \sup_{A \in \mathcal{A}} | \mu(A) - \nu(A) |$. 
%
%
\begin{definition}[{Dependency matrix, \citet[Section 2]{samson2000concentration}}]
\label{def:depmat}
The \emph{dependency matrix}
of a process $\{Z_t\}_{t=0}^{T-1}$ with distribution $\mathsf{P}_{Z}$ is the (upper-triangular) matrix $\Gamma_{\mathsf{dep}}(\mathsf{P}_{Z})=\{\Gamma_{ij}\}_{i,j=0}^{T-1} \in \mathbb{R}^{T\times T}$ defined as follows.
Let $\mathcal{Z}_{0:i}$ denote the $\sigma$-algebra 
generated by $\{Z_t\}_{t=0}^{i}$.
For indices $i < j$, let 
\begin{align}
    \Gamma_{ij} = \sqrt{ 2 \sup_{A \in \mathcal{Z}_{0:i}} \tvnorm{ \sfP_{Z_{j:T-1}}(\cdot \mid A) - \sfP_{Z_{j:T-1}} } }. \label{eq:dep_mat_coeffs}
\end{align}
For the remaining indices $i \geq j$, let
$\Gamma_{ii} =1$ and $\Gamma_{ij}=0$ when $i > j$ (below the diagonal).
\end{definition}
%
%
%
Given the dependency matrix from \Cref{def:depmat}, 
we measure the dependency of 
the process $\Px$ by the quantity 
$\opnorm{\Gammadep(\Px)}$. Notice that this quantity always satisfies
$1 \leq \opnorm{\Gammadep(\Px)} \lesssim T$.
The lower bound indicates that the process $\Px$ is independent
across time. The upper bound indicates that the process is fully
dependent, e.g., $X_{t+1} = X_t$ for all $t \in \N$.

Our results apply to cases where $\opnorm{\Gammadep(\Px)}^2$ 
grows sub-linearly in $T$-- the exact requirement depends on  the specific function class $\scrF$.
If the process $\{X_t\}$ is geometrically $\phi$-mixing, then $\opnorm{\Gammadep(\Px)}^2$ 
is upper bounded by a constant that depends on the 
mixing-time of the process, and is
independent of $T$ \citep[Section 2]{samson2000concentration}. 
Other examples, such as processes satisfying Doeblin's condition~\citep{meyn1993markov}, are given in \citet[Section 2]{samson2000concentration}. %
When $\{X_t\}$ is a stationary time-homogenous Markov chain
with invariant distribution $\pi$, the coefficients $\Gamma_{ij}$ simplify to
$\Gamma_{ij}^2 = 2 \sup_{A \in \mathcal{X}_\infty} \tvnorm{\sfP_{X_{j-i}}(\cdot \mid A) - \pi}$
for indices $j > i$,
where $\mathcal{X}_\infty$ is the $\sigma$-algebra generated
by $X_\infty \sim \pi$ (cf.~\Cref{prop:tv_reduction_markov}).
Hence, the requirement
$\opnorm{\Gammadep(\Px)}^2 \lesssim T^\beta$ for $\beta \in (0, 1)$
then corresponds to
$\sup_{A\in\mathcal{X}_\infty}\tvnorm{\sfP_{X_{t}}(\cdot \mid A) - \pi} \lesssim 1/t^{1-\beta}$ for $t \in \N_+$.
\citet{jarner2002polynomial} give various examples and conditions to check polynomial convergence rates for Markov chains. We also provide further means to verify $\opnorm{\Gamma_{\mathsf{dep}}(\mathsf{P}_X)} = O(1)$ in \Cref{sec:basictoolsdepmat} and \Cref{sec:appendix:truncated_gaussian}.




\subsection{Learning with little mixing}
\label{sec:littlemix}

A key quantity appearing in our bounds is
a martingale variant of the notion of Gaussian complexity.
\begin{definition}[{Martingale offset complexity, cf.\ \citet{liang2015learning}, \ \citet{ziemann2022single}}]
\label{eq:martingale_complexity}
For the regression problem (\ref{eq:ts}), the martingale offset complexity of a function space $\mathscr{F}$ is given by:
\begin{align}
\label{eq:thecomplexitymeasure}
\mathsf{M}_T(\mathscr{F}) \triangleq   \sup_{f\in \mathscr{F}} \left\{\frac{1}{T}\sum_{t=0}^{T-1} 4 \langle W_t,  f(X_t) \rangle - \|f(X_t)\|_2^2 \right\}.
\end{align}
\end{definition}


Recall that $\scrF_\star = \scrF - \{f_\star\}$ is the centered
function class
and $\partial B(r) = \{ f \in \scrF_\star \mid \norm{f}_{L^2} = r \}$ is the boundary of the $L^2$ ball $B(r)$.
The following theorem is the main result of this paper.
\begin{theorem}
\label{thm:themainthm}
Fix $B > 0$, 
$C : (0, B] \to \mathbb{R}_+$,
$\alpha \in [1,2]$, and
$r \in (0, B]$. 
Suppose that $\mathscr{F}_\star$ is star-shaped and $B$-bounded.
Let $\scrF_r \subset \mathscr{F}_\star$ be a $r/\sqrt{8}$-net of $\partial B(r)$ in the supremum norm $\norm{\cdot}_\infty$, and suppose that $(\mathscr{F}_r,\mathsf{P}_X)$ satisfies the trajectory $(C(r),\alpha)$-hypercontractivity condition  (cf.~\Cref{def:hypconrelax}). Then:
\begin{align}
    \E \|\widehat f - f_\star\|_{L_2}^2 \leq 8 \E \mathsf{M}_T(\mathscr{F}_\star)   +r^2+ B^2|\mathscr{F}_r| \exp \left( \frac{-T r^{4-2\alpha} }{8C(r) \opnorm{ \Gamma_{\mathsf{dep}}(\mathsf{P}_X)}^2 } \right). \label{eq:mainthm}
\end{align}
\end{theorem}

The assumption that $\scrF_\star$ is star-shaped
in \Cref{thm:themainthm} is not particularly restrictive. 
Indeed, \Cref{thm:themainthm} still holds if we replace $\scrF_\star$ by its star-hull $\mathrm{star}(\scrF_\star) \triangleq \{ \gamma f \mid \gamma \in [0, 1], \:\: f \in \scrF_\star \}$, and $\partial B(r)$ with the boundary of
the $r$-sphere of $\mathrm{star}(\scrF_\star)$.
In this case, we note that (a) 
the metric entropy of $\mathrm{star}(\scrF_\star)$ 
is well controlled by the metric entropy of $\scrF_\star$,\footnote{Specifically, $\log\calN_\infty(\mathrm{star}(\scrF_\star), \e) \leq \log(2B/\e) + \log\calN_\infty(\scrF_\star, \e/2)$~\citep[Lemma 4.5]{mendelson2002improving}.} and (b) the trajectory hypercontractivity conditions over a class $\scrF_\star$ and its star-hull $\mathrm{star}(\scrF_\star)$ are equivalent. Hence, at least whenever we are able to verify hypercontractivity over the entire class $\scrF_\star$, little generality is lost. While most of our examples are star-shaped, we will need the observations above when we work with generalized linear model dynamics in \Cref{sec:results:glm}.



To understand \Cref{thm:themainthm}, we will proceed in a series of steps. We first need to understand the martingale complexity term
$\E \mathsf{M}_T(\scrF_\star)$.
Since $\mathscr{F}_\star$ is $B$-bounded,
if one further imposes the tail conditions that the noise process $\{W_t$\} is a $\sigma_W^2$-sub-Gaussian MDS,\footnote{
That is, for any $u \in \mathbb{S}^{\dy-1}$,
$\lambda \in \R$,
and $t \in \N$,
we have $\E[\exp(\lambda \langle W_t, u\rangle) \mid \mathcal{F}_{t-1} ]\leq \exp(\lambda^2\sigma_W^2/2)$. 
} a chaining argument detailed in \citet[Lemma 4]{ziemann2022single} shows that:
\begin{equation}
\label{eq:theformofthebound}
\begin{aligned}
     \E \mathsf{M}_T(\mathscr{F}_\star) \lesssim \inf_{\substack{\gamma>0,\\ \delta \in [0, \gamma]}} \Bigg\{\gamma^2 + \frac{\sigma^2_W\log \mathcal{N}_\infty(\scrF_\star,\gamma)}{T}
     + \sigma_W \sqrt{\dy} \delta
     + \frac{\sigma_W } {\sqrt{T}} \int_{\delta}^\gamma \sqrt{\log \mathcal{N}_\infty(\scrF_\star,s)} \,\rmd s  \Bigg\}.
\end{aligned}
\end{equation}
In particular, this bound only depends on $\scrF_\star$ and is \emph{independent} of 
$\opnorm{\Gammadep(\Px)}^2$. Furthermore, \eqref{eq:theformofthebound} coincides with the corresponding risk bound for the LSE with \iid\ covariates \citep{liang2015learning}.


Given that $\E \mathsf{M}_T(\scrF_\star)$ corresponds to the
rate of learning from $T$ \iid\ covariates, the form of 
\eqref{eq:mainthm} suggests that we choose $r^2 \lesssim \E \sfM_T(\scrF_\star)$, so that the
dominant term in \eqref{eq:mainthm} is equal to $\E \sfM_T(\scrF_\star)$ in scale.
Given that $r$ has been set, the only remaining degree of freedom
in \eqref{eq:mainthm} is to set $T$ large enough (the burn-in time) 
so that the third term is dominated by $r^2$. 
Thus, it is this third term in \eqref{eq:mainthm}
that captures the interplay between
the function class $\scrF_\star$ and the dependency measure
$\opnorm{\Gammadep(\Px)}$.
We will now consider specific examples to illustrate how
the burn-in time can be set.

Our first example supposes that
(a) $\scrF_\star$ satisfies the trajectory $(C,2)$-hypercontractivity
condition, and that (b)  $\scrF_\star$
is nonparametric, but not too large:
\begin{align}
    \exists\, p > 0,\, q \in (0, 2) \textnormal{ s.t. } \log{\calN_{\infty}(\scrF_\star, \varepsilon)} \leq p \left(\frac{1}{\varepsilon}\right)^q \textnormal{ for all } \varepsilon \in (0, 1). \label{eq:nonparametric_simple}
\end{align}
Covering numbers of the form \eqref{eq:nonparametric_simple}
are typical for sufficiently smooth function classes,
e.g.\ the space of $k$-times continuously differentiable
functions mapping $\sfX \to \sfY$ for any $k \geq \ceil{\dx/2}$~\citep{tikhomirov1993entropy}.
If condition \eqref{eq:nonparametric_simple} holds and the noise process $\{W_t\}$ is a sub-Gaussian MDS, inequality~(\ref{eq:theformofthebound}) yields
$\E \sfM_T(\scrF_\star) \lesssim T^{-\frac{2}{2+q}}$,
and hence we want to set $r^2 = o(T^{-\frac{2}{2+q}})$. Carrying out this program yields the following corollary.
\begin{corollary}
\label{corr:Cconst_v2}
Fix $B \geq 1$, 
$C > 0$, 
$p > 0$,
$q \in (0, 2)$, 
and $\gamma \in (0, \frac{q}{2+q})$.
Suppose that $\scrF_\star$
is star-shaped, $B$-bounded, satisfies \eqref{eq:nonparametric_simple},
and $(\scrF_\star, \Px)$ satisfies the trajectory $(C, 2)$-hypercontractivity condition. Suppose that $T$ satisfies:
\begin{align}
    T \geq \max\left\{\left[8 (32p+1)C \opnorm{\Gammadep(\Px)}^2 \right]^{\frac{1}{1-\frac{q}{2}\left(\frac{2}{2+q} + \gamma\right)}}, \left[\log{B} + \frac{4}{q} \log\left(\frac{8}{q}\right)\right]^{\frac{1}{\frac{q}{2}\left(\frac{2}{2+q} + \gamma\right)}} \right\}. \label{eq:Cconst_burnin_time}
\end{align}
Then, we have that:
\begin{align}
\label{eq:stmtnonparcorr}
    \E \norm{\widehat{f}-f_\star}_{L^2}^2 \leq 8 \E \sfM_T(\scrF_\star) + 2T^{-\left(\frac{2}{2+q}+\gamma\right)}.
\end{align}
\end{corollary}

The rate \eqref{eq:stmtnonparcorr} of \Cref{corr:Cconst_v2} highlights the fact that the first order term of the excess risk is bounded by the martingale offset complexity $\E\mathsf{M}_T(\mathscr{F}_\star)$. This behavior arises since the dependency matrix $\Gammadep(\Px)$
only appears as the burn-in requirement \eqref{eq:Cconst_burnin_time}. Here, the value of $q$  constrains how fast $\opnorm{\Gammadep(\Px)}^2$ is allowed to grow.
In particular, condition \eqref{eq:Cconst_burnin_time} requires that
$\opnorm{\Gammadep(\Px)}^2 = o(T^{1 - \frac{q}{2+q}})$,
otherwise the burn-in condition cannot be satisfied 
for any $\gamma \in (0, \frac{q}{2+q})$.

In our next example, we consider both
a variable hypercontractivity parameter $C(r)$
that varies with the covering radius $r$,
and also allow $\alpha \in (1,2]$ to vary.
Since our focus is on the interaction of
the parameters in the hypercontractivity definition,
we will consider smaller function classes with logarithmic metric entropy. This includes parametric classes but also bounded subsets of certain reproducing kernel Hilbert spaces. For such function spaces,
one expects $\E \sfM_T(\scrF_\star) \leq \tilde{O}(T^{-1})$,
and hence we set $r^2 = o(T^{-1})$.
%
\begin{corollary}
\label{corr:gammagrowandeps}
Fix $B \geq 1$, $C:(0,1] \to \mathbb{R}_+$, $\alpha \in (1,2]$, 
$b_1 \in [0, 1)$, 
$b_2 \in [0, 2)$, 
$\gamma \in (0, 1)$,
and $p,q \geq 1$.
Suppose that $\mathscr{F}_\star$ is star-shaped and $B$-bounded, and that for every $r \in (0, 1)$,
there exists a $r$-net $\scrF_r$ of $\partial B(r)$ in the $\norm{\cdot}_\infty$-norm such that
(a) $\log |\mathscr{F}_r| \leq p\log^q\left(\frac{1}{r}\right)$ and
(b) $(\scrF_r, \Px)$ satisfies the
trajectory $(C(r), \alpha)$-hypercontractivity
condition.
%
%
Next, suppose the growth conditions hold:
\begin{align*}
   \opnorm{\Gamma_{\mathsf{dep}}(\Px)}^2 \leq T^{b_1}, \:\: C(r) \leq (1/r)^{b_2} \, \forall r \in (0,1).
\end{align*}
As long as the constants $\alpha$, $b_1$, $b_2$, and $\gamma$ satisfy:
\begin{align}
    \psi := 1 - b_1 - \frac{(1+\gamma)(4-2\alpha+b_2)}{2} > 0, \label{eq:psi_constraint}
\end{align}
then for any $T \geq \mathsf{poly}_{\frac{q}{\psi}}\left( p, \log{B}, \frac{q}{\psi} \right)$,
we have:
\begin{align*}
    \E \|\widehat f - f_\star\|_{L_2}^2 \leq 8 \E \mathsf{M}_T(\mathscr{F}_\star)  +2\left(\frac{1}{T}\right)^{1+\gamma}.
\end{align*}
\end{corollary}
Here $\mathsf{poly}_{q/\psi}$ denotes
a polynomial of degree $O(q/\psi)$ in its arguments--
the exact expression is given in the proof.
\Cref{prop:infdimex} in \Cref{sec:examples} 
gives an example of an $\ell^2(\N)$ ellipsoid which satisfies the assumptions in \Cref{corr:gammagrowandeps}.
\Cref{corr:gammagrowandeps} illustrates the
interplay between the function class $\scrF_\star$,
the data dependence of the covariate process
$\{X_t\}$, and the hypercontractivity
constant $\alpha$.
Let us consider a few cases.
First, let us suppose that
the process $\{X_t\}$ is geometrically ergodic
and that $C(r)$ is a constant, so that we can set
$b_1$ and $b_2$ arbitrarily close to zero (at the expense of a longer burn-in time). Then, inequality
\eqref{eq:psi_constraint} simplifies to
$\alpha > 2 - \frac{1}{1+\gamma}$.
This illustrates that in
the hypercontractivity regime ($\alpha > 1$),
there exists a valid setting of $(b_1,b_2,\gamma)$ that
satisfies inequality \eqref{eq:psi_constraint}.
Next, let us consider the case where $C(r)$ is again a constant, but $\{X_t\}$ is not geometrically ergodic. Setting $b_2$ and $\gamma$ arbitrarily close to zero, we have that \eqref{eq:psi_constraint} simplifies to
$b_1 < \alpha - 1$.
Compared to \Cref{corr:Cconst_v2}, we see that
in the case when $\alpha=2$, 
the parametric nature of $\scrF_\star$
allows the dependency requirement to be
less strict:
$o(T)$ in the parametric case versus $o(T^{1-\frac{q}{2+q}})$ in the nonparametric case.
%




We conclude with noting that when $\alpha=1$, 
it is not possible to remove the dependence
on $\opnorm{\Gammadep(\Px)}^2$ in the lowest
order term. In this situation, our results
recover existing risk bounds 
from \citet{ziemann2022single}. We refer to \Cref{sec:appendix:bdd} 
for details. 

\section{Proof techniques}
\label{sec:proofs}

In this section, we highlight the main ideas behind the proof of \Cref{thm:themainthm}. The full details
can be found in the appendix.
We start with a key insight from \citet{mendelson2014learning}: establishing a one-sided inequality between the empirical versus
true risk is substantially easier than the corresponding two-sided inequality.
Recall that the closed ball of radius $r$ is
$B(r) = \left\{ f \in \mathscr{F}_\star \,\Big|\, \frac{1}{T}\sum_{t=0}^{T-1} \E\| f(X_t) \|_2^2 \leq r^2\right\}$.
We identify conditions which depend mildly on $\opnorm{\Gammadep(\Px)}$, so that with high probability we have:
\begin{equation}
    \forall f \in \scrF_\star \setminus B(r), \:\: \E\left[\frac{1}{T} \sum_{t=0}^{T-1} \norm{f(X_t)}_2^2\right] \lesssim \frac{1}{T} \sum_{t=0}^{T-1} \norm{f(X_t)}_2^2. \label{eq:lower_isometry_simple}
\end{equation}
Once this \emph{lower isometry} condition \eqref{eq:lower_isometry_simple} holds, we bound the empirical excess risk, the RHS of inequality (\ref{eq:lower_isometry_simple}), by a version of the basic inequality of least squares \citep{liang2015learning,ziemann2022single}. This leads to an upper bound of the empirical excess risk by the 
martingale complexity term (\Cref{eq:martingale_complexity}). 
As our innovation mainly lies in establishing the lower isometry
condition (\ref{eq:lower_isometry_simple}), we focus on this component for the remainder of the proof outline.

\paragraph{Lower isometry}
The key tool we use is
the following exponential inequality, which  
controls the lower tail of sums of non-negative 
dependent random variables via the dependency matrix $\Gamma_{\mathsf{dep}}(\mathsf{P}_X)$. 
\begin{theorem}[{\citet[Theorem 2]{samson2000concentration}}]
\label{prop:samsonexpineq}
Let $g : \sfX \to \R$ be non-negative. Then for any $\lambda > 0$:
\begin{align}
\label{eq:samsonexpineq}
\E \exp \left(- \lambda \sum_{t=0}^{T-1} g(X_t) \right) \leq \exp\left(- \lambda \sum_{t=0}^{T-1} \E g(X_t)+ \frac{\lambda^2 \opnorm{\Gamma_{\mathsf{dep}}(\Px)}^2 \sum_{t=0}^{T-1} \E g^2(X_t) }{2}  \right).
\end{align}
\end{theorem}
We note that 
\citet[Theorem 2]{samson2000concentration} actually proves a 
much stronger statement than \Cref{prop:samsonexpineq} (a Talagrand-style 
uniform concentration inequality),
from which \Cref{prop:samsonexpineq} is a byproduct.
With \Cref{prop:samsonexpineq} in hand, 
the following \namecref{lem:lowertailmix} allows us to relate
hypercontractivity to lower isometry.
\begin{proposition}
\label{lem:lowertailmix}
Fix $C > 0$ and $\alpha \in [1,2]$.
Let $g : \sfX \rightarrow \R$ be a non-negative function satisfying
\begin{align}
  \E \left[\frac{1}{T}\sum_{t=0}^{T-1} g^2(X_t)\right] \leq C\left(\E \left[\frac{1}{T}\sum_{t=0}^{T-1}  g(X_t)\right]\right)^\alpha. \label{eq:contractivity_g}
\end{align}
Then we have:
\begin{align*}
    \mathbf{P} \left( \sum_{t=0}^{T-1} g(X_t)  \leq \frac{1}{2}\sum_{t=0}^{T-1} \E g(X_t)   \right) \leq  \exp \left( -\frac{T}{8 C \opnorm{\Gamma_{\mathsf{dep}}(\Px)}^2}  \left(\frac{1}{T}\sum_{t=0}^{T-1}\E g(X_t)\right)^{2-\alpha}\right).
\end{align*}
\end{proposition}

Now fix an $f \in \scrF_\star \setminus B(r)$,
and put $g(x) = \norm{f(x)}_2^2$.
Substituting $g$ into \eqref{eq:contractivity_g}
yields the trajectory hypercontractivity condition \eqref{cond:hyperconttrajrelax}
in \Cref{def:hypconrelax}.
Thus, \Cref{lem:lowertailmix} establishes
the lower isometry condition \eqref{eq:lower_isometry_simple} for any fixed function. Hence, it remains to take a union bound
over a supremum norm cover of $\scrF_\star \setminus B(r)$ at resolution $r$.
It turns out that it suffices to instead cover the boundary
$\partial B(r)$ since $\scrF_\star$ is star-shaped.
Carrying out these details leads
to the main lower isometry result.


\begin{theorem}
\label{thm:lucemthm}
Fix constants $\alpha \in [1,2]$ and $C,r > 0$. Let $\mathscr{F}_\star$ be star-shaped, and suppose that there exists a
$r/\sqrt{8}$-net $\mathscr{F}_r$ of $\partial B(r)$ in the $\norm{\cdot}_\infty$-norm 
such that $(\mathscr{F}_r,\mathsf{P}_X)$  satisfies the trajectory $(C,\alpha)$-hypercontractivity condition. Then the following lower isometry holds:
\begin{align*}
     \mathbf{P} \left( \exists f \in \mathscr{F}_\star\setminus B(r) \:\Bigg \vert\: \frac{1}{T} \sum_{t=0}^{T-1} \| f(X_t) \|_2^2 \leq  \E  \frac{1}{8T} \sum_{t=0}^{T-1} \| f(X_t) \|_2^2 \right) \leq|\mathscr{F}_r |\exp \left( \frac{-T r^{4-2\alpha} }{8C \opnorm{ \Gamma_{\mathsf{dep}}(\mathsf{P}_X)}^2 } \right).
\end{align*}
%
\end{theorem}

\subsection{Handling unbounded trajectories}
\label{sec:proofs:unbounded}
Our main result \Cref{thm:themainthm} requires boundedness
of both the hypothesis class $\scrF$ and the covariate
process $\{X_t\}$ to hold. However, when this does not hold,
\Cref{thm:themainthm} can often still be applied via a 
careful truncation argument.
In this section, we outline the key ideas of this
argument, with the full details given in  \Cref{sec:appendix:truncated_gaussian}.

For concreteness, let us consider a Markovian process driven by Gaussian noise.
Let $\{W_t\}_{t \geq 0}$ and $\{W'_t\}_{t \geq 0}$ be sequences of \iid\ $N(0, I)$ vectors in $\R^{\dx}$.
Fix a dynamics function $f : \R^{\dx} \to \R^{\dx}$
and truncation radius $R > 0$. Define the truncated noise process $\{\bar{W}_t\}_{t \geq 0}$ as $\bar{W}_t \triangleq W'_t \ind\{\norm{W'_t}_2 \leq R \}$, and
denote the original process and its truncated process by:
\begin{subequations}
\begin{align}
    X_{t+1} &= f(X_t) + HW_t, \:\: X_0 = HW_0, &&\text{\textcolor{gray}{(original process)}} \label{eq:dynamic_process_gaussian_main_text} \\
    \bar{X}_{t+1} &= f(\bar{X}_t) + H\bar{W}_t, \:\: \bar{X}_0 = H\bar{W}_0. &&\text{\textcolor{gray}{(truncated process)}} \label{eq:dynamic_process_trunc_main_text}
\end{align}
\end{subequations}
Setting $R$ appropriately, it is clear that 
the original process \eqref{eq:dynamic_process_gaussian_main_text}
coincides with the truncated process \eqref{eq:dynamic_process_trunc_main_text}
with high probability by standard Gaussian concentration inequalities.
Furthermore, the truncated noise process $\{H \bar{W}_t\}$ remains a martingale difference sequence due to the symmetry of the truncation.
Additionally, since $\{H \bar{W}_t\}$ is bounded,
if $f$ is appropriately Lypaunov stable then the process $\{\bar{X}_t\}$ becomes bounded.
In turn any class $\scrF$ containing continuous functions is bounded as well on \eqref{eq:dynamic_process_trunc_main_text}.
Hence, the LSE $\hat{f}$ on \eqref{eq:dynamic_process_gaussian_main_text}
can be controlled by the LSE $\bar{f}$ on \eqref{eq:dynamic_process_trunc_main_text}.

So far, this is a straightforward reduction. However, a subtle point arises in applying \Cref{thm:themainthm} to the LSE $\bar{f}$ on \eqref{eq:dynamic_process_trunc_main_text}: the dependency 
matrix $\Gammadep$ now involves the truncated process \eqref{eq:dynamic_process_trunc_main_text}
instead of the original process \eqref{eq:dynamic_process_gaussian_main_text}. 
This is actually \emph{necessary} for 
this strategy to work, as the supremum in the dependency matrix coefficients \eqref{eq:dep_mat_coeffs}
is now over the truncated process $\{\bar{X}_t\}$, instead of the original process $\{X_t\}$ which is unbounded.
However, there is a trade-off, as bounding the coefficients 
for $\{\bar{X}_t\}$ is generally more complex
than for $\{X_t\}$.\footnote{The clearest
example of this is when the dynamics function $f$ is linear:
in this case, $\{X_t\}$ is jointly Gaussian (and hence \eqref{eq:dep_mat_coeffs} can
be bounded by closed-form expressions), whereas $\{\bar{X}_t\}$ is not due 
to the truncation operator.} 
Nevertheless, a coupling argument
allows us to switch back to bounding the dependency matrix coefficients for $\{X_t\}$, but crucially keep the supremum over the truncated $\{\bar{X}_t\}$. This reduction
substantially broadens the scope of \Cref{thm:themainthm} without any modification to the proof.
\section{Examples of trajectory hypercontractivity}
\label{sec:examples}

In this section, we detail a few examples of
trajectory hypercontractivity.
Let us begin by considering the simplest possible example: a finite hypothesis class. Let $|\mathscr{F}| < \infty$. Define for any fixed $f \in \mathscr{F}_\star$ the 
constant
$c_f \triangleq \E\left[ \frac{1}{T}\sum_{t=0}^{T-1} \| f(X_t)\|^4_2\right]/\left(\E \left[ \frac{1}{T}\sum_{t=0}^{T-1}  \| f(X_t)\|_2^2\right]\right)^2$,
where the ratio $0/0$ is taken to be $1$.
Then the class $\mathscr{F}_\star$ is trajectory $(\max_{f \in \scrF_\star} c_f, 2)$-hypercontractive.

Similarly, processes evolving on a finite state space can
also be verified to be hypercontractive.
\begin{proposition}
\label{prop:mcmoments}
Fix a $\mumin > 0$.
Let $\{\mu_t\}_{t=0}^{T-1}$ denote the marginal distributions of $\Px$.
Suppose that the $\mu_t$'s all share a common support
of a finite set of atoms
$\{\psi_1, \dots, \psi_K\} \subset \R^{\dx}$,
and that $\min_{0 \leq t \leq T-1} \min_{1 \leq k \leq K} \mu_t(\psi_k) \geq \mumin$.
For any class of functions $\scrF$ mapping
$\{\psi_1,\dots,\psi_K\} \to \R^{\dy}$, we have that
$\scrF$ satisfies the trajectory $(1/\mumin, 2)$-hypercontractivity condition.
\end{proposition}

We remark that when
$\Px$ is an aperiodic and irreducible Markov chain over
a finite state space, the condition $\mumin > 0$ is always
valid even as $T \to \infty$ \citep{levin2008markov}.
%
In this case, our findings are related to \citet[Theorem 3.1]{wolfer2021statistical}, who show that in the high accuracy regime (i.e.,\ after a burn-in time), the minimax rate of estimating the transition probabilities of such a chain is not affected by the mixing time (in their case the pseudo-spectral gap).

The examples considered thus far rely on the fact that under a certain degree of finiteness, the fourth and second moment can be made uniformly equivalent. The next proposition relaxes this assumption. 
Namely, if for some $\e\in(0,2]$ the $L^{2}$ and $L^{2+\e}$  norms are equivalent on a bounded class $\mathscr{F}$, this class then satisfies a nontrivial hypercontractivity constant, $\alpha>1$ (cf. \citet{mendelson2017aggregation}). 

\begin{proposition}
\label{prop:momeqv}
Fix $\e \in (0,2]$ and $c>0$. Suppose that $\scrF$ is $B$-bounded and that $\|f\|_{L^{2+\e}} \leq c \|f\|_{L^2}$ for all $f \in \scrF$. Then $\scrF$ satisfies the trajectory $(B^{2-\e}c^{2+\e},1+\e/2)$-hypercontractivity condition.
\end{proposition}

Next, we show that for processes $\{X_t\}$ which converge fast enough to a stationary distribution, it suffices to verify the hypercontractivity condition
only over the stationary distribution.
This mimics existing results in \iid\ learning,
where hypercontractivity is assumed over the covariate distribution \citep{liang2015learning,mendelson2014learning}.
We first recall the definition of the $\chi^2$ divergence between two measures.
Let $\mu$ and $\nu$ be two measures over the same probability space,
and suppose that $\mu$ is absolutely continuous w.r.t.\ $\nu$.
The $\chi^2(\mu, \nu)$ divergence is defined as
$\chi^2(\mu, \nu) \triangleq \E_{\nu}\left[ \left( \frac{\rmd\mu}{\rmd\nu} - 1\right)^2\right]$.

\begin{proposition}
\label{prop:eqvfromstat_v2}
Fix positive $r$, $C_{\chi^2}$, $C_{\mathsf{TV}}$, and $C_{8 \to 2}$.
Suppose that the process $\{X_t\}$ has 
a stationary distribution $\pi$.
Let $\{\mu_t\}$ denote the marginal distributions of
$\{X_t\}$, and suppose that the marginals $\{\mu_t\}$ are absolutely continuous w.r.t.\ $\pi$.
Assume the process is ergodic in the sense that:
\begin{align}
    \sup_{t \in \N} \chi^2(\mu_t, \pi) \leq C_{\chi^2}, \:\:  \frac{1}{T} \sum_{t=0}^{T-1} \tvnorm{\mu_t - \pi} \leq C_{\mathsf{TV}} r^2. \label{eq:chisq_ergodicity_v2}
\end{align}
Suppose also that for all $f \in \scrF_\star$:
\begin{align}
    \E_{\pi} \norm{f(X)}_2^8 \leq C_{8 \rightarrow 2} (\E_{\pi} \norm{f(X)}_2^2)^4.\label{eq:L8_L2_v2}
\end{align}
The set $\partial B(r)$
satisfies 
$(C, 2)$-trajectory hypercontractivity
with $C = (1 + \sqrt{C_{\chi^2}}) \sqrt{C_{8 \to 2}} (1+ C_{\mathsf{TV}} B^2)^2$.
\end{proposition}


Let us discuss the ergodicity conditions in \Cref{prop:eqvfromstat_v2}.
The condition $\sup_{t \in \N} \chi^2(\mu_t, \pi) < \infty$ from
\eqref{eq:chisq_ergodicity_v2} is quite mild.
To illustrate this point, suppose that $\{X_t\}$ are regularly spaced samples in time
from the It{\^{o}} stochastic differential equation:
\begin{align*}
    \rmd Z_t = f(Z_t)\, \rmd t + \sqrt{2}\, \rmd B_t, 
\end{align*}
where $(B_t)$ is standard Brownian motion in $\R^{\dx}$.
Assume the process $(Z_t)$ admits a stationary distribution $\pi$, and
let $\rho_t$ denote the measure of $Z_t$ at time $t$.
A standard calculation \citep[Theorem 4.2.5]{bakry2014book} shows that $\frac{\rmd}{\rmd t} \chi^2(\rho_t, \pi) = -2\E_\pi \bignorm{\nabla\left(\frac{\rho_t}{\pi}\right)}^2_2 \leq 0$, and hence $\sup_{t \geq 0} \chi^2(\rho_t, \pi) \leq \chi^2(\rho_0, \pi)$.
Thus, as long as the initial measure $\rho_0$ has finite divergence with $\pi$, then 
this condition holds.
One caveat is that $\chi^2(\rho_0, \pi)$ can scale as
$e^{\dx}$, resulting in a hypercontractivity constant that scales 
exponentially in dimension. This however only affects the burn-in time and not the final rate.

The second condition in \eqref{eq:chisq_ergodicity_v2} is
$\frac{1}{T} \sum_{t=0}^{T-1} \tvnorm{\mu_t - \pi} \lesssim r^2$.
A typical setting is
$r^2 \asymp 1/T^{\beta}$ for some
$\beta \in (0, 1]$, where $\beta$ is dictated by the
function class $\scrF$. Hence, this requirement reads:
\begin{align*}
    \frac{1}{T} \sum_{t=0}^{T-1} \tvnorm{\mu_t - \pi} \lesssim \frac{1}{T^\beta} \Longleftrightarrow \sum_{t=0}^{T-1} \tvnorm{\mu_t - \pi} \lesssim T^{1-\beta}.
\end{align*}
Therefore, the setting of $\beta$ determines the level of ergodicty
required. For example, if $\beta=1$ (which corresponds to the parametric function case), then this condition necessitates
geometric ergodicity, since it requires that
$\sum_{t=0}^{T-1} \tvnorm{\mu_t - \pi} = O(1)$.
On the other hand, suppose that $\beta \in (0, 1)$.
Then this condition is satisfied if
$\tvnorm{\mu_t-\pi} \lesssim 1/t^\beta$, allowing for slower mixing rates.


\subsection{Ellipsoids in $\ell^2(\mathbb{N})$}


Given that equivalence of norms is typically a finite-dimensional phenomenon, one may wonder whether examples of
hypercontractivity exist in an infinite-dimensional setting. 
Here we show that such examples are actually rather abundant. The key is that hypercontractivity need only be satisfied on an $\e$-cover of $\mathscr{F}_\star$. 
As discussed above, every finite hypothesis class (and thus every finite cover) is automatically $(C,2)$-hypercontractive for some $C>0$. The issue is to ensure that this constant does not grow too fast as one refines the cover. The next result shows that the growth can be controlled for  $\ell^2(\mathbb{N})$ ellipsoids of orthogonal expansions. By Mercer's theorem, these ellipsoids correspond to unit balls in reproducing kernel Hilbert spaces~\citep[Corollary 12.26]{wainwright2019high}.
\begin{proposition}
\label{prop:infdimex}
Fix positive constants $\beta$, $B$, $K$, and $q$. Fix a base measure $\lambda$ on $\mathsf{X}$ and suppose that $\{\phi_n\}_{n\in \mathbb{N}_+}$ is an orthonormal system in $L^2(\lambda)$ satisfying $\|\phi_n \|_{\infty} \leq B n^q,\, \forall n \in \mathbb{N}$. Suppose $\mu_j \leq e^{-2 \beta j}$ and define the ellipsoid:
\begin{align*}
\mathscr{P} \triangleq \left\{ f= \sum_{j=1}^\infty \theta_j \phi_j \Bigg| \sum_{j=1}^\infty \frac{\theta_j^2}{\mu_j}\leq 1\right\}.
\end{align*}
%
Fix $\e > 0$, and let $m_\e$ denote the smallest positive integer solution to $m \geq \frac{2}{\beta} \left| \log \left(\frac{8B}{\beta \e}\right) \right|$ subject to $\frac{m}{\log m} \geq \frac{q}{\beta}$.
Let $P \subset \scrP$ be an arbitrary subset.
There exists an $\e$-cover $P_\e$ of $P$ in the $\norm{\cdot}_\infty$-norm satisfying:
\begin{align*}
    \log{|P_\e|} \leq m_\e \log\left(1+\frac{8Bm^{q}_\e}{\e}\right).
\end{align*}
%
%
Furthermore, let $\{\mu_t\}_{t=0}^{T-1}$ be the marginal distributions of $\Px$, and suppose that
the $\mu_t$'s and $\lambda$ are all mutually absolutely continuous. Assume that 
$\max_{0 \leq t \leq T-1} \max\left\{ \frac{\rmd \mu_t}{\rmd\lambda}, \frac{\rmd\lambda}{\rmd\mu_t} \right\} \leq K$.
Then, as long as $\e \leq \inf_{f \in P} \norm{f}_{L^2(\Px)}$,
then
$(P_\e,\Px)$ is trajectory $(C_\e,2)$-hypercontractive with 
$C_\e = (1 + 7K^3 B^4 m_\varepsilon^{4q+2})$.
\end{proposition}
%
\Cref{prop:infdimex} states that 
when $\scrF_\star \subseteq \scrP$, then
$(\partial B(r),\mathsf{P}_X)$ is $(C(r),2)$-hypercontractive where $C(r)=C_r$ only grows \emph{poly-logarithmically} in $1/r$. 
In particular, this example verifies the assumptions of \Cref{corr:gammagrowandeps}. 





\section{System identification in parametric classes}
\label{sec:results:comparison}

To demonstrate the sharpness of 
our main result, we instantiate \Cref{thm:themainthm} on two 
parametric system identification problems which have received recent
attention in the literature:
linear dynamical systems (LDS) and generalized linear model (GLM) 
dynamics.


\subsection{Linear dynamical systems}
\label{sec:results:lds}

Consider the setting where the process $\{X_t\}_{t \geq 0}$ is
described by a linear dynamical system:
\begin{align}
    X_{t+1} = A_\star X_t + H V_t, \:\: X_0 = H V_0, \:\: V_t \in \R^{\dv}, \:\: V_t \sim N(0, I), \:\: V_t \perp V_{t'} \:\forall\, t \neq t'. \label{eq:LDS_dynamics}
\end{align}
In this setting, the system identification problem is to
recover the dynamics matrix $A_\star$ from 
$\{X_t\}_{t=0}^{T-1}$ evolving according to \eqref{eq:LDS_dynamics}.
We derive rates for recovering $A_\star$ by first deriving an
excess risk bound on the least-squares estimator via \Cref{thm:themainthm}, and then
converting the risk bound to a parameter error bound.
Since \Cref{thm:themainthm} relies on the process being ergodic, we consider the case when $A_\star$ is stable.
We start by stating a few standard definitions.
\begin{definition}
\label{def:k_step_controllable}
Fix a $k \in \{1, \dots, \dx\}$.
The pair $(A, H)$ is \emph{$k$-step controllable} if 
\begin{align*}
    \rank \left( \begin{bmatrix} H & AH & A^{2}H & \dots &A^{k -1}H \end{bmatrix}\right)= \dx.
\end{align*}
\end{definition}

For $t \in \N$, let the $t$-step \emph{controllability gramian} be defined as 
$\Gamma_t \triangleq \sum_{k=0}^{t} A^k HH^\T (A^k)^\T$.
Since the noise in \eqref{eq:LDS_dynamics}
serves as the ``control'' in this setting, the
controllability gramian also coincides with the covariance at time $t$, i.e., $\E[X_tX_t^\T] = \Gamma_t$.

\begin{definition}
\label{def:lds_strongly_stable}
Fix a $\tau \geq 1$ and $\rho \in (0, 1)$.
A matrix $A$ is called \emph{$(\tau,\rho)$-stable}
if for all $k \in \N$ we have $\opnorm{A^k} \leq \tau \rho^k$.
\end{definition}
With these definitions in place, we now state our result for
linear dynamical system.
\begin{theorem}
\label{prop:ldsprop}
Suppose that the matrix $A_\star$ in \eqref{eq:LDS_dynamics} is $(\tau, \rho)$-stable (cf.~\Cref{def:lds_strongly_stable}), and that the pair
$(A_\star,H)$ is $\kappa$-step controllable (cf.~\Cref{def:k_step_controllable}).
Suppose also that $\norm{A_\star}_F \leq B$
for some $B \geq 1$.
Consider the linear hypothesis class and true regression function:
\begin{align}
\mathscr{F} \triangleq \{ f(x) = Ax \mid A \in \mathbb{R}^{\dx \times \dx}, \:\: \|A\|_F \leq B\}, \:\: f_\star(x) = A_\star x. \label{eq:LDS_function_class}
\end{align}
Suppose that model \eqref{eq:ts} 
follows the process described in \eqref{eq:LDS_dynamics}
with $Y_{t} = X_{t+1}$.
There exists $T_0$ such that the LSE with hypothesis class $\mathscr{F}$ achieves for all $T \geq T_0$:
\begin{align}
    \E\| \widehat f- f_\star\|_{L^2}^2 \leq 8\E \mathsf{M}_T(\mathscr{F}_\star) + \frac{4\opnorm{H}^2 \dx^2}{T}. \label{eq:LDS_final_rate}
\end{align}
Furthermore, for a universal positive constant $c_0$, $T_0$ may be chosen as:
\begin{align}
    T_0 = c_0 \frac{\tau^4 \opnorm{H}^4 \dx^2}{(1-\rho)^2\lambda_{\min}(\Gamma_{\kappa-1})^2}\left[ \kappa^2 \vee \frac{1}{(1-\rho)^2}\right] \mathrm{polylog}\left( B, \dx, \tau, \opnorm{H}, \frac{1}{\lambda_{\min}(\Gamma_{\kappa-1})}, \frac{1}{1-\rho},\right). \label{eq:LDS_T0_burnin}
\end{align}
\end{theorem}
We first discuss the rate \eqref{eq:LDS_final_rate} prescribed by
\Cref{prop:ldsprop}.
A simple computation shows that the martingale complexity $\E \sfM_T(\scrF_\star)$ can be upper bounded by 
$1/T$ times the self-normalized martingale term which typically
appears in the analysis of least-squares~\citep{abbasiyadkori2011selfnormalized}.
Specifically, when the empirical covariance matrix $\sum_{t=0}^{T-1} X_tX_t^\T $ is invertible:
\begin{align*}
    \E \sfM_T(\scrF_\star) \leq \frac{4}{T} \E\bignorm{\left(\sum_{t=0}^{T-1} X_tX_t^\T \right)^{-1/2} \sum_{t=0}^{T-1} X_t V_t^\T H^\T}_F^2.
\end{align*}
A sharp analysis of this self-normalized martingale term \citep[Lemma 4.1]{tu2022learning} shows that $\E \sfM_T(\scrF_\star) \lesssim \frac{\opnorm{H}^2 \dx^2}{T}$, and hence \eqref{eq:LDS_final_rate} yields the minimax optimal rate
up to constant factors after a polynomial burn-in time.\footnote{
While the burn-in time is polynomial in the problem constants listed in \eqref{eq:LDS_T0_burnin}, \citet{tsiamis2021linear}
show that these constants (specifically $1/\lambda_{\min}(\Gamma_{\kappa-1})$) can scale
exponentially in $\kappa$, the controllability index of the system.}
This is unlike the chaining bound \eqref{eq:theformofthebound}
which yields extra logarithmic factors~\citep[see e.g.][Lemma 4]{ziemann2022single}.
Note that the burn-in time of $\tilde{O}(\dx^2)$ 
given by our result is sub-optimal by a factor of $\dx$. This extra factor
comes from the union bound over a Frobenius norm ball of
$\dx \times \dx$ matrices in \Cref{thm:themainthm}.



To convert \eqref{eq:LDS_final_rate}
into a parameter recovery bound, we simply lower bound the excess risk:
\begin{align}
    \E \norm{\hat{f} - f_\star}_{L^2}^2 \geq \E \norm{\hat{A}-A_\star}_F^2 \lambda_{\min}(\bar{\Gamma}_T) \Longrightarrow \E\norm{\hat{A}-A_\star}_F^2 \lesssim \frac{\opnorm{H}^2 \dx^2}{T \lambda_{\min}(\bar{\Gamma}_T)}, \label{eq:LDS_parameter_rate}
\end{align}
where $\bar{\Gamma}_T \triangleq \frac{1}{T} \sum_{t=0}^{T-1} \E[X_tX_t^\T]$ is the average covariance matrix.
The rate \eqref{eq:LDS_parameter_rate} recovers, after the polynomial burn-in time, existing results~\citep{simchowitz2018learning,sarkar2019near,jedra2020finite,tu2022learning} for stable systems, with a few caveats.
First, most of the existing
results are given in operator instead of Frobenius norm.
We ignore this issue, since the only difference is the extra
unavoidable factor of $\dx$ in the rate for the Frobenius norm
compared to the operator norm rate.
Second, 
since \Cref{thm:themainthm} ultimately relies on some degree of ergodicity for the covariate process $\{X_t\}_{t \geq 0}$, we cannot handle the marginally stable case (where $A_\star$ is allowed to have spectral radius equal to one) as in \citet{simchowitz2018learning,sarkar2019near,tu2022learning},
nor the unstable case as in \citet{faradonbeh2018finite,sarkar2019near}.

We conclude with a short discussion on the proof of \Cref{prop:ldsprop}.
As the LDS process \eqref{eq:LDS_dynamics} is unbounded, we
use the truncation argument outlined in \Cref{sec:proofs:unbounded}
so that \Cref{thm:themainthm} still applies.
Furthermore, since the process \eqref{eq:LDS_dynamics} is jointly Gaussian,
the dependency matrix coefficients are simple to bound, resulting
in polynomial rates \eqref{eq:LDS_T0_burnin} for the burn-in time.
A much wider variety of non-Gaussian noise distributions
can be handled via ergodic theory for Markov chains~\citep[see e.g.][Chapter 15]{meyn1993markov}. While these results typically do not
offer explicit expressions for the mixing coefficients,
both \citet{douc2004rates} and \citet{hairer2011ergodic} 
provide a path forward for deriving explicit bounds. 
We however omit these calculations in the interest of simplicity.

\subsection{Generalized linear models}
\label{sec:results:glm}

We next consider the following non-linear dynamical system:
\begin{align}
    X_{t+1} = \sigma(A_\star X_t) + H V_t, \:\: X_0 = H V_0, \:\: V_t \in \R^{\dx}, \:\: V_t \sim N(0, I), \:\: V_t \perp V_{t'} \:\forall\,t \neq t' \label{eq:GLM_dynamics}.
\end{align}
Here, $A_\star \in \R^{\dx \times \dx}$ is the dynamics matrix and
$\sigma : \R^{\dx} \to \R^{\dx}$ is a coordinate wise link function.
The notation $\sigma$ will also be overloaded to refer to the individual coordinate function mapping $\R \to \R$.
We study the system identification problem where the link function $\sigma$ is assumed to be known, but the dynamics matrix $A_\star$ 
is unknown and to be recovered from $\{X_t\}_{t=0}^{T-1}$.
We will apply \Cref{thm:themainthm} to derive a nearly optimal excess risk bound for the LSE on this problem in the stable case.

We start by stating a few assumptions that are again standard in the literature.
\begin{assumption}
\label{assumption:glm}
Suppose that $A_\star$, $H$, and $\sigma$ from the GLM process \eqref{eq:GLM_dynamics} satisfy:
\begin{enumerate}
    \item (One-step controllability). The matrix $H \in \R^{\dx \times \dx}$ is full rank.
    \item (Link function regularity). The link function $\sigma : \R \to \R$ 
    is $1$-Lipschitz, satisfies $\phi(0) = 0$,
    and there exists a $\zeta \in (0, 1]$ such that
    $|\sigma(x) - \sigma(y)| \geq \zeta |x-y|$ for all $x,y\in \R$.
    \item (Lyapunov stability). There exists a positive definite diagonal matrix $P_\star \in \R^{\dx \times \dx}$ satisfying $P_\star \succcurlyeq I$ and a $\rho \in (0, 1)$ such that $A_\star^\T P_\star A_\star \preccurlyeq \rho P_\star$.
\end{enumerate}
\end{assumption}

Several remarks on \Cref{assumption:glm} are in order.
First, the rank condition on $H$ ensures that the noise process $\{H V_t\}_{t \geq 0}$ is non-degenerate. Viewing \eqref{eq:GLM_dynamics}
as a control system mapping
$\{V_t\}_{t \geq 0} \mapsto \{X_t\}_{t \geq 0}$, this condition ensures that this system is one-step controllable.
%
Next, the link function assumption is standard in the literature (see e.g.~\citet{sattar2020non,foster2020learning,kowshik2021near}).
The expansiveness condition $|\sigma(x) - \sigma(y)| \geq \zeta |x-y|$ ensures that the link function is increasing at a uniform
rate.
For efficient parameter recovery, some extra assumption other than
Lipschitzness and monotonicity is needed~\citep[Theorem 4]{kowshik2021near}, and expansiveness yields a sufficient
condition. However, for excess risk, it is unclear if any extra requirements are necessary. We leave resolving this issue
to future work.
%
%
%
Finally, the Lyapunov stability condition is due to 
\citet[Proposition 2]{foster2020learning}, and yields a certificate
for global exponential stability (GES) to the origin.
It is weaker than requiring that $\opnorm{A_\star} < 1$, which amounts to taking $P_\star = I$.
The assumption $P_\star \succcurlyeq I$ is without loss of generality by rescaling $P_\star$.

With our assumptions in place, we are ready to state our main result concerning the excess risk for the LSE applied to
the process \eqref{eq:GLM_dynamics}.
\begin{theorem}
\label{thm:glmthm}
Suppose the model \eqref{eq:ts} follows the process
described in \eqref{eq:GLM_dynamics} with $Y_t = X_{t+1}$.
Assume that the process \eqref{eq:GLM_dynamics} satisfies
\Cref{assumption:glm}. Fix a $B \geq 1$, and suppose that
$\norm{A_\star}_F \leq B$.
Consider the hypothesis class and true regression function:
\begin{align}
\mathscr{F} \triangleq \{ f(x) = \sigma(Ax) \mid A \in \mathbb{R}^{\dx \times \dx}, \:\: \|A\|_F \leq B\}, \:\: f_\star(x) = \sigma(A_\star x). \label{eq:GLM_function_class}
\end{align}
There exists a $T_0$
and a universal positive constant $c_0$ 
such that the LSE with hypothesis class $\mathscr{F}$ achieves for all $T \geq T_0$:
\begin{align}
    \E\| \widehat f- f_\star\|_{L^2}^2 \leq c_0  \frac{\opnorm{H}^2 \dx^2}{T} \log\left(\max\left\{T, B, \dx, \opnorm{P_\star}, \opnorm{H}, \frac{1}{1-\rho}\right\}\right). \label{eq:GLM_final_rate}
\end{align}
Furthermore, for a universal constant $c_1>0$, we may choose $T_0$ as:
\begin{align}
    T_0 = c_1\max \frac{\opnorm{P_\star}^{2} \mathrm{cond}(H)^4  \dx^4}{\zeta^4 (1-\rho)^6} \mathrm{polylog}\left( B, \dx, \opnorm{P_\star}, \mathrm{cond}(H), \frac{1}{\zeta}, \frac{1}{1-\rho} \right).
    \label{eq:glm_T0_burnin}
\end{align}
\end{theorem}

\Cref{thm:glmthm} states that after a polynomial burn-in time 
(which scales quite sub-optimally as $\tilde{O}(\dx^4)$
in the dimension), the excess risk 
scales as the minimax rate $\opnorm{H}^2 \dx^2/T$ times a logarithmic factor
of various problem constants.
To the best of our knowledge, this is the sharpest excess risk bound for this problem
in the literature and is nearly minimax optimal.
As noted previously, the logarithmic factor enters via the chaining inequality \eqref{eq:theformofthebound} when bounding the martingale offset complexity. 
We leave to future work
a more refined analysis that removes this logarithmic dependence,
and also improves the polynomial dependence of $T_0$ on $\dx$.
The extra $\dx^2$ factor in \eqref{eq:glm_T0_burnin} 
over the LDS burn-in time \eqref{eq:LDS_T0_burnin} comes from our 
analysis of the trajectory hypercontractivity constant
for this problem, and should be removable.


Much like in \eqref{eq:LDS_parameter_rate}, we can use the 
link function expansiveness in \Cref{assumption:glm} to 
convert the excess risk bound
\eqref{eq:GLM_final_rate} to a parameter recovery rate:
\begin{align}
    \E\norm{\hat{A} - A_\star}_F^2 \leq \tilde{O}(1) \frac{\opnorm{H}^2 \dx^2}{\zeta^2 T \lambda_{\min}(\bar{\Gamma}_T)}, \label{eq:GLM_parameter_rate}
\end{align}
where again $\bar{\Gamma}_T \triangleq \frac{1}{T} \sum_{t=0}^{T-1} \E[X_tX_t^\T]$ is the average covariance matrix of the GLM process \eqref{eq:GLM_dynamics}.
Note that the one-step controllability assumption in \Cref{assumption:glm} ensures
that the covariance matrix $\E[ X_tX_t^\T ] \succcurlyeq HH^\T$ is invertible for every $t \in \N$.
A detailed comparison of the excess risk rate \eqref{eq:GLM_final_rate} and parameter recovery rate \eqref{eq:GLM_parameter_rate}
with existing bounds in the literature is given in
\Cref{sec:appendix:glm:comparison}.

Let us briefly discuss the proof of \Cref{thm:glmthm}.
As in the LDS case, we use the truncation argument
described in \Cref{sec:proofs:unbounded} that allows us 
apply \Cref{thm:themainthm} while still
using bounds on the dependency matrix coefficients of the original unbounded process 
\eqref{eq:GLM_dynamics}.
%
%
However, an additional complication arises compared to the LDS case, as the covariates are not jointly Gaussian due to the presence
of the link function. 
While at this point we could appeal to ergodic theory,
we instead develop an alternative approach that still allows us to
compute explicit constants.
Building on the work of \citet{chae2020wasserstein}, we use the smoothness of the Gaussian transition kernel to
upper bound the TV distance by the $1$-Wasserstein distance.
This argument is where our analysis crucially relies on 
the non-degeneracy of $H$ in \Cref{assumption:glm}, 
as the transition kernel corresponding
to multiple steps of \eqref{eq:GLM_dynamics} is no longer Gaussian.
The $1$-Wasserstein distance is then controlled by using the \emph{incremental stability} \citep{tran2016incremental} properties of
the deterministic dynamics $x_+ = \sigma(A_\star x)$.
Since this technique only depends on the GLM dynamics through  incremental stability, it is of independent interest as it
applies much more broadly.

\section{Numerical experiments}
\label{sec:experiments}

We conduct a simple numerical simulation to illustrate
the phenomenon of learning with little mixing empirically.
We consider system identification of the GLM dynamics described
in \Cref{sec:results:glm}.

We first describe how the covariate process $\{X_t\}$ is generated.
We set $\dx=25$.
The true dynamic matrix $A_\star$ is
randomly sampled from the distribution described in 
Section~7 of \citet{kowshik2021near}.
Specifically, $A_\star = U \Sigma U^\T$, where 
$U$ is uniform from the Haar measure on the space of 
orthonormal $\dx \times \dx$ matrices,
and $\Sigma = \diag( \underbrace{\rho, \dots, \rho}_{\floor{\dx/2} \textrm{ times}}, \rho/3, \dots, \rho/3 )$.
We vary $\rho \in \{0.9, 0.99\}$ for this experiment.
Next, we set the activation function 
$\sigma$ to be the LeakyReLU with slope $0.5$, i.e.,
$\sigma(x) = 0.5 x \ind\{ x < 0 \} + x \ind\{x \geq 0\}$.
Observe that these dynamics satisfy
\Cref{assumption:glm} with $\zeta = 0.5$, and
where the Lyapunov matrix $P$ can be taken
to be identity, since $\opnorm{A_\star} = \rho < 1$.
Next, we generate $X_0 \sim N(0, I_{\dx})$,
and $X_{t+1} = \sigma(A_\star X_t) + W_t$ with $W_t \sim N(0, 0.01 I_{\dx})$ and $W_t \perp W_{t'}$ for $t \neq t'$.
From this trajectory, the labelled dataset is 
$\{(X_t, Y_t)\}_{t=0}^{T-1}$ with
$Y_t = X_{t+1}$.

To study the effects of the correlation
from a single trajectory $\{X_t\}$
for learning, we consider the following 
\emph{independent baseline}
motivated by the Ind-Seq-LS baseline 
described in \citet{tu2022learning}.
Let $\mu_t$ denote the marginal distribution of $X_t$.
We sample $\bar{X}_t \sim \mu_t$ independently across time,
and sample $\bar{Y_t} \mid \bar{X}_t$ from the conditional distribution
$N(\sigma(A_\star \bar{X}_t), 0.01 I_{\dx})$;
the labelled dataset is $\{(\bar{X}_t, \bar{Y}_t)\}_{t=0}^{T-1}$.
This ensures that the $L^2$ risk of a fixed hypothesis
$f(x) = \sigma(Ax)$ is the same under both the independent
baseline and the single trajectory distribution, so that our 
experiment singles out the effect of learning
from correlated data.
In practice, each $\bar{X}_t$ is sampled from a new
independent rollout up to time $t$.

Given a dataset $\{(X_t,Y_t)\}_{t=0}^{T-1}$, we 
search for the empirical risk minimizer (ERM) of the loss
\begin{align}
    \hat{A} = \argmin_{A \in \R^{\dx \times \dx}} \left\{\frac{1}{T}\sum_{t=0}^{T-1} \norm{\sigma(A X_t) - Y_t}_2^2\right\} \label{eq:experiment_ERM}
\end{align}
by running \texttt{scipy.optimize.minimize} with the 
\texttt{L-BFGS-B} method, using the default linesearch
and termination criteria options.
To calculate the $L^2$ excess risk
$\frac{1}{T}\sum_{t=0}^{T-1} \E\norm{\sigma(\hat{A} X_t) - \sigma(A_\star X_t)}_2^2$ of a hypothesis $\hat{A}$, 
we draw $1000$ new trajectories
and average the excess risk over these trajectories.
The experimental code is implemented with
\texttt{jax}~\citep{jax2018github}, and run using the CPU backend
with \texttt{float64} precision on a single machine.\footnote{Code available at: \href{https://github.com/google-research/google-research/tree/master/learning_with_little_mixing}{https://github.com/google-research/google-research/tree/master/learning{\_}with{\_}little{\_}mixing}}

\begin{figure}[htb]
    \centering
    \includegraphics[width=0.85\columnwidth]{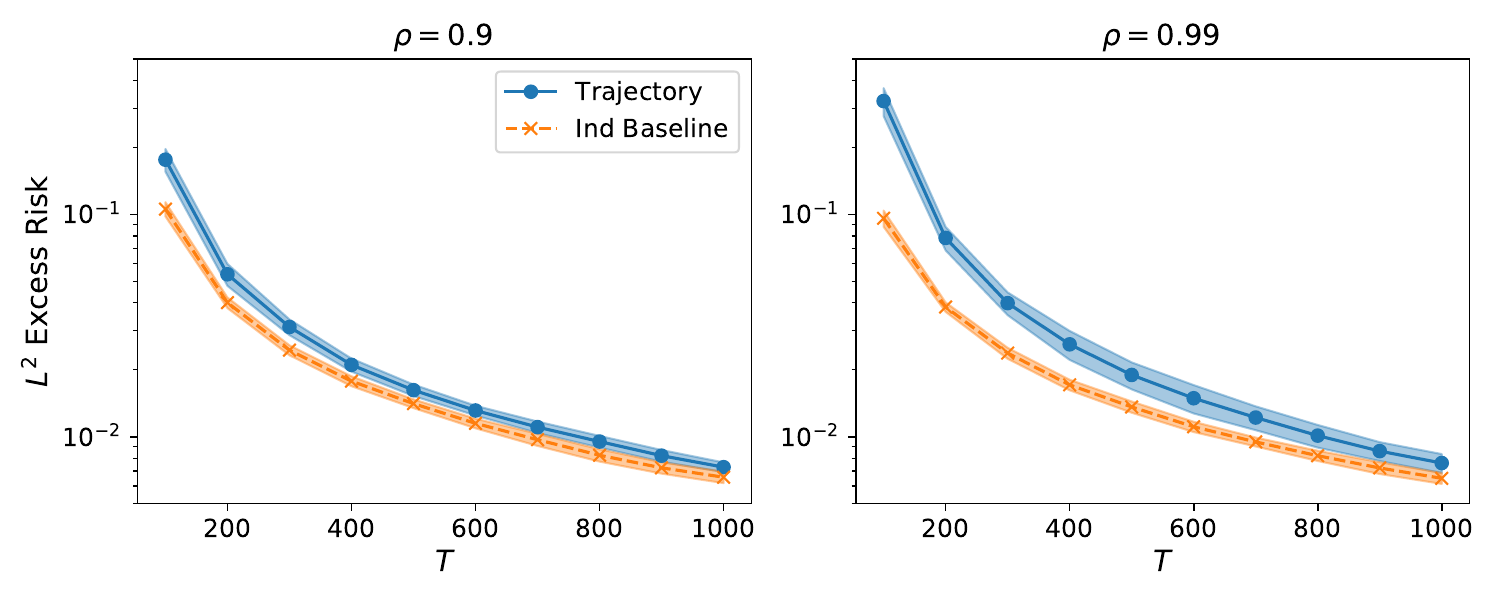}
    \caption{$L^2$ excess risk as a function of dataset length $T$ of the empirical risk minimizer on the single trajectory (Trajectory) dataset versus the independent baseline (Ind Baseline) dataset.}
    \label{fig:glm_little_mixing_risks}
\end{figure}
\begin{figure}[htb]
    \centering
    \includegraphics[width=0.85\columnwidth]{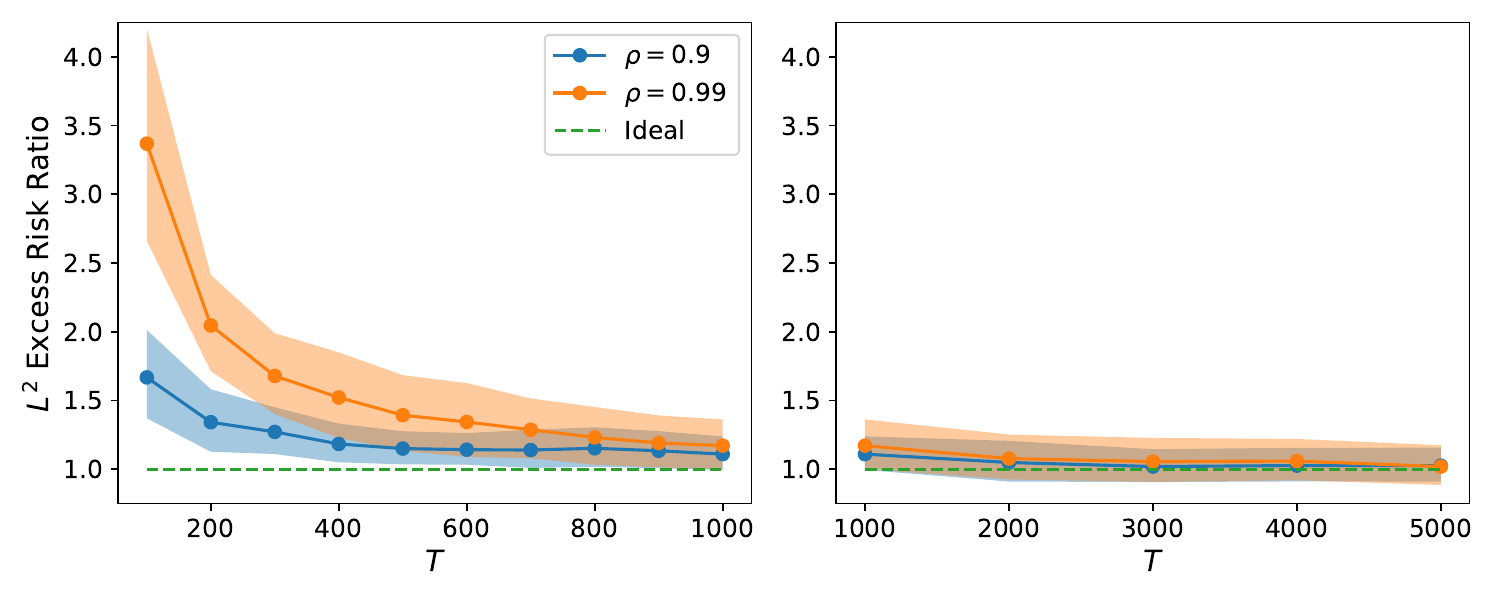}
    \caption{Ratio of the $L^2$ excess risk as a function of dataset
    length $T$ of the empirical risk minimizer (ERM) on the single trajectory dataset over the ERM on the independent baseline dataset.
    The dashed green curve (Ideal) marks a ratio of exactly one.}
    \label{fig:glm_little_mixing_ratio}
\end{figure}

The results of this experiment are shown in
\Cref{fig:glm_little_mixing_risks} and \Cref{fig:glm_little_mixing_ratio}. 
In \Cref{fig:glm_little_mixing_risks}, we plot the $L^2$ excess risk
of the ERM $\hat{A}$ from \eqref{eq:experiment_ERM}
on both the trajectory dataset $\{(X_t, Y_t)\}$ and
the independent baseline dataset $\{(\bar{X}_t, \bar{Y}_t)\}$,
varying $\rho \in \{0.9, 0.99\}$. The shaded region indicates
$\pm$ one standard deviation from the mean over $20$ training datasets.
In \Cref{fig:glm_little_mixing_ratio}, we plot the $L^2$
excess risk \emph{ratio} of the estimator $\hat{A}$
from the single trajectory dataset
over the estimator $\hat{A}$ from the independent baseline
trajectory, again varying $\rho \in \{0.9, 0.99\}$.
Here, the shaded region is constructed using $\pm$ one standard deviation
of the numerator and denominator taken over $20$ training datasets.

\Cref{fig:glm_little_mixing_risks} and \Cref{fig:glm_little_mixing_ratio} illustrate two different trends, which are both predicted by our theory. 
First, for a fixed $\rho$, as
$T$ increases, the $L^2$ excess risk of the ERM on the
trajectory dataset approaches that of the ERM on the independent
dataset. This illustrates the learning with little mixing
phenomenon, where despite correlations in the covariates $\{X_t\}$
of the trajectory dataset across time, the statistical behavior
of the ERM approaches that of the ERM on the independent
dataset where the covariates $\{\bar{X}_t\}$ are independent across 
time.
Next, for a fixed $T$, the burn-in time increases as $\rho$ approaches one. That is, systems that mix slower have longer burn-in times. 

\section{Conclusion}
\label{sec:conclusion}
We developed a framework for showing
when the mixing-time of the covariates plays a relatively small role in the rate of convergence of the least-squares estimator. 
In many situations, after a finite burn-in time, this learning procedure exhibits an excess risk that scales as if all the samples were independent (\Cref{thm:themainthm}). 
As a byproduct of our framework, by instantiating our results to system
identification for dynamics with generalized linear model transitions (\Cref{sec:results:glm}), 
we derived the sharpest known excess risk rate for this problem; our rates are nearly minimax optimal after only a polynomial burn-in time. 


To arrive at \Cref{thm:themainthm}, we leveraged insights from \citet{mendelson2014learning} via a one-sided concentration inequality (\Cref{thm:lucemthm}). As mentioned in \Cref{sec:prels}, 
hypercontractivity is closely related to the small-ball condition~\citep{mendelson2014learning}. Such conditions can be understood as quantitative identifiability conditions by providing control of the ``version space'' (cf.~\citet{mendelson2014learning}). Given that identifiability conditions also play a key role in linear system identification---a setting in which a similar phenomenon as studied here had already been reported---this suggests an interesting direction for future work: are such conditions actually necessary for learning with little mixing?


\section*{Acknowledgements}
We thank Dheeraj Nagaraj for helpful discussions regarding the results in \citet{nagaraj2020least}, and Abhishek Roy for clarifying the results in \citet{roy2021dependent}.
We also thank Mahdi Soltanolkotabi for an informative discussion about the
computational aspects of empirical risk minimization for generalized linear models
under the square loss.

\bibliographystyle{unsrtnat}
\bibliography{main.bib}

\newpage
\appendix

\tableofcontents

\newpage

\section{Proof of Theorem~\ref{thm:themainthm}}

\subsection{Proof of \Cref{lem:lowertailmix}}

Let us abbreviate $\Gamma = \Gamma_{\mathsf{dep}}(\Px)$.
A Chernoff argument yields
\begin{align*}
    &\mathbf{P} \left( \sum_{t=0}^{T-1} g(X_t)  \leq \frac{1}{2}\sum_{t=0}^{T-1} \E g(X_t)   \right)\\
    &\leq \inf_{\lambda \geq 0} \E\exp\left(  \frac{\lambda}{2}\sum_{t=0}^{T-1} \E g(X_t)-\lambda\sum_{t=0}^{T-1} g(X_t)   \right)&&\textnormal{(Chernoff)}\\
    &\leq  \inf_{\lambda \geq 0} \exp\left(  -\frac{\lambda}{2}\sum_{t=0}^{T-1} \E g(X_t) +\frac{\lambda^2\opnorm{\Gamma}^2 \sum_{t=0}^{T-1} \E g^2(X_t) }{2} \right) && (\textnormal{\Cref{prop:samsonexpineq}})\\
    &=\exp \left( -\frac{\left(\sum_{t=0}^{T-1}\E g(X_t) \right)^2}{8 \opnorm{\Gamma}^2\sum_{t=0}^{T-1} \E g^2(X_t)} \right) && \left(\lambda = \frac{ \sum_{t=0}^{T-1} \E g(X_t)}{2 \opnorm{\Gamma}^2 \sum_{t=0}^{T-1} \E g^2(X_t) } \right)\\
    &\leq \exp \left( -\frac{T}{8 C \opnorm{\Gamma}^2} \times \left(\frac{1}{T}\sum_{t=0}^{T-1}\E g(X_t)\right)^{2-\alpha}  \right), && \textnormal{(Using \eqref{eq:contractivity_g})}
\end{align*}
as per requirement. 
\hfill $\blacksquare$

\subsection{Proof of Theorem~\ref{thm:lucemthm}}

The hypothesis that $\mathscr{F}_\star$ is star-shaped allows us to rescale, so it suffices to prove the result for $f \in \partial B(r)$. Namely, if $f\in \mathscr{F}_\star \setminus B(r)$ then $\frac{r}{\|f\|_{L^2}}<1$ and so $r f / \|f\|_{L^2} \in \partial B(r)$ by the star-shaped hypothesis.
%
%
Recall that $\scrF_r \subset \partial B(r)$ is a $r/\sqrt{8}$-net of $\partial B(r)$ in the supremum norm.
Hence,
by construction and parallellogram, for every $f \in \partial B(r)$, there exists $f_i \in \scrF_r$ such that:
\begin{align}
\label{eq:discretizepartial}
  \frac{1}{T} \sum_{t=0}^{T-1} \| f(X_t) \|_2^2 \geq \frac{1}{2T} \sum_{t=0}^{T-1} \| f_i(X_t) \|_2^2 - \frac{r^2}{8}.
\end{align}
%
Define the event:
\begin{align*}
  \mathcal{E} \triangleq \bigcup_{f \in \scrF_r} \left\{  \frac{1}{T} \sum_{t=0}^{T-1} \| f(X_t) \|_2^2 \leq   \E \frac{1}{2T} \sum_{t=0}^{T-1} \| f(X_t) \|_2^2  \right\}.
\end{align*}
Invoking Lemma~\ref{lem:lowertailmix} with $g(x) = \| f(x) \|_2^2$ for $f \in \scrF_r$, 
by a union bound it clear that
\begin{align}
\label{eq:mixinglowerisometryunionbound}
    \mathbf{P}(\mathcal{E}) \leq |\scrF_r| \exp \left( \frac{- Tr^{4-2\alpha} }{8C \opnorm{\Gammadep(\Px)}^2} \right).
\end{align}
Fix now arbitrary $f \in \partial B(r)$. On the complement $\mathcal{E}^c$ it is true that
\begin{equation*}
    \begin{aligned}
     \frac{1}{T} \sum_{t=0}^{T-1} \| f(X_t) \|_2^2 &\geq \frac{1}{2T} \sum_{t=0}^{T-1} \| f_i(x_t) \|_2^2 - \frac{r^2}{8} && \textnormal{(we may find such an $f_i$ by observation (\ref{eq:discretizepartial}))}\\
     & \geq  \E \frac{1}{4T} \sum_{t=0}^{T-1} \| f_i(X_t) \|_2^2 - \frac{r^2}{8} && \textnormal{(by definition of $ \mathcal{E}$)}\\
     & = \frac{r^2}{4}- \frac{r^2}{8} && (f_i \in \partial B(r)) \\
     &\geq \frac{r^2}{8}.
    \end{aligned}
\end{equation*}
Since $f \in \partial B(r)$ was arbitrary, by virtue of the estimate (\ref{eq:mixinglowerisometryunionbound}) we have that:
\begin{align*}
     \mathbf{P} \left( \exists f \in \mathscr{F}_\star\cap \partial B(r) \:\Bigg \vert\: \frac{1}{T} \sum_{t=0}^{T-1} \| f(X_t) \|_2^2  \leq \frac{r^2}{8}\right) \leq|\mathscr{F}_r |\exp \left( \frac{-T r^{4-2\alpha} }{8C \opnorm{ \Gamma_{\mathsf{dep}}(\mathsf{P}_X)}^2 } \right).
\end{align*}
The result follows by rescaling.
\hfill $\blacksquare$

\subsection{Proof of Theorem~\ref{thm:themainthm}}

Define the event:
\begin{align*}
    \mathcal{B}_r\triangleq  \mathbf{P} \left( \exists f \in \mathscr{F}_\star\setminus B(r) \:\Bigg \vert\: \frac{1}{T} \sum_{t=0}^{T-1} \| f(X_t) \|_2^2 \leq  \E  \frac{1}{8T} \sum_{t=0}^{T-1} \| f(X_t) \|_2^2 \right)
\end{align*}
By definition, on the complement of $\mathcal{B}_r$ we have that:
\begin{align}
\label{eq:lowerisometryused}
    \|\widehat f - f_\star\|_{L^2}^2 \leq r^2 \vee \frac{8}{T}\sum_{t=0}^{T-1} \|\widehat f(X_t)-f_\star(X_t)\|_2^2 \leq r^2 + \frac{8}{T}\sum_{t=0}^{T-1} \|\widehat f(X_t)-f_\star(X_t)\|_2^2  .
\end{align}
Therefore, we can decompose $\E \| \widehat f - f_\star \|_{L^2}^2$ as
follows:
\begin{equation}
\label{eq:decompforthm}
\begin{aligned}
    \E \|\widehat f - f_\star\|_{L^2}^2 &= \E \mathbf{1}_{\mathcal{B}_r}\|\widehat f - f_\star\|_{L^2}^2+\E \mathbf{1}_{\mathcal{B}_r^c}\|\widehat f - f_\star\|_{L^2}^2\\
    &\leq B^2 \mathbf{P}(\mathcal{B}_r) +r^2+8 \E \left[\frac{1}{T}\sum_{t=0}^{T-1} \|\widehat f(X_t)-f_\star(X_t)\|_2^2\right]. && (\textnormal{$B$-bdd \& ineq. (\ref{eq:lowerisometryused})})
\end{aligned}
\end{equation}
\Cref{thm:lucemthm} informs us that:
\begin{align}
\label{eq:crudebd}
    \mathbf{P}(\mathcal{B}_r) \leq  |\mathscr{F}_r | \exp \left( \frac{-T r^{4-2\alpha} }{8C \opnorm{\Gammadep(\Px)}^2 } \right).
\end{align}
On the other hand, we have by the basic inequality (as in \cite{liang2015learning}): 
\begin{align}
\label{eq:basicineqused}
    \frac{1}{T}\sum_{t=0}^{T-1} \|\widehat f(X_t)-f_\star(X_t)\|_2^2 \leq \frac{1}{T}\sup_{f\in \mathscr{F}_\star} \sum_{t=0}^{T-1} 4\langle W_t, f(X_t)\rangle - \|f(X_t)\|_2^2
\end{align}
Combining inequalities (\ref{eq:decompforthm}), (\ref{eq:crudebd}) and (\ref{eq:basicineqused}) we conclude:
\begin{align*}
    \E \|\widehat f - f_\star\|_{L^2}^2 &\leq 8\E\left[ \sup_{f\in \mathscr{F}_\star} \frac{1}{T}\sum_{t=0}^{T-1}4 \langle W_t, f(X_t)\rangle - \|f(X_t)\|_2^2\right]\\
    &\qquad + r^2 + B^2|\mathscr{F}_r |\exp \left( \frac{-T r^{4-2\alpha} }{8C \opnorm{\Gammadep(\Px)}^2 } \right),
\end{align*}
as per requirement.
\hfill$\blacksquare$


\section{Proofs for corollaries in \Cref{sec:results}}

\subsection{Proof of \Cref{corr:Cconst_v2}}
We set $r^2 = \frac{1}{T^{\frac{2}{2+q}+\gamma}}$.
We first 
use \citet[Exercise 4.2.10]{vershynin2018} 
followed by \eqref{eq:nonparametric_simple} to bound:
\begin{align*}
    \log\calN_\infty(\partial B(r), r/\sqrt{8}) \leq \log\calN_\infty(\scrF_\star, r/(2\sqrt{8})) \leq p\left(\frac{2\sqrt{8}}{r}\right)^q.
\end{align*}
Therefore:
\begin{align*}
    B^2 \calN_{\infty}(\partial B(r), r/\sqrt{8}) \exp\left( \frac{-T}{8C \opnorm{\Gammadep(\Px)}^2} \right) \leq B^2 \exp\left( 32p T^{\frac{q}{2+q} + \frac{q\gamma}{2}} - \frac{T}{8C \opnorm{\Gammadep(\Px)}^2} \right).
\end{align*}
We want to solve for $T$ such that:
\begin{align*}
    B^2 \exp\left( 32p T^{\frac{q}{2+q} + \frac{q\gamma}{2}} - \frac{T}{8C \opnorm{\Gammadep(\Px)}^2} \right) \leq \frac{1}{T^{\frac{2}{2+q}+\gamma}}.
\end{align*}
To do this, we first require that:
\begin{align*}
    32p T^{\frac{q}{2+q} + \frac{q\gamma}{2}} - \frac{T}{8C \opnorm{\Gammadep(\Px)}^2} \leq - T^{\frac{q}{2+q} + \frac{q\gamma}{2}} &\Longleftrightarrow T^{1 - \left( \frac{q}{2+q} + \frac{q\gamma}{2}\right)} \geq 8(32p+1) C\opnorm{\Gammadep(\Px)}^2 \\
    &\Longleftrightarrow T \geq \left[8 (32p+1)C \opnorm{\Gammadep(\Px)}^2 \right]^{\frac{1}{1-\frac{q}{2}\left(\frac{2}{2+q} + \gamma\right)}}.
\end{align*}
Now, with this requirement, we are left with the sufficient condition:
\begin{align*}
    B^2 \exp\left( - T^{\frac{q}{2+q} + \frac{q\gamma}{2}} \right) \leq \frac{1}{T^{\frac{2}{2+q}+\gamma}} &\Longleftrightarrow T^{\frac{q}{2+q} + \frac{q\gamma}{2}} \geq \log(B^2 T^{\frac{2}{2+q}+\gamma}).
\end{align*}
It suffices to require that:
\begin{align*}
    &T^{\frac{q}{2+q} + \frac{q\gamma}{2}} \geq \log{B} \Longleftrightarrow T \geq (\log{B})^{\frac{1}{\frac{q}{2}\left(\frac{2}{2+q} + \gamma\right)}}, \\
    &T^{\frac{q}{2+q} + \frac{q\gamma}{2}} \geq \log(T^{\frac{2}{2+q}+\gamma}) = \frac{2}{q} \log( T^{\frac{q}{2+q} + \frac{q\gamma}{2}} ).
\end{align*}
By \citet[Lemma A.4]{simchowitz2018learning}, the bottom inequality
holds when:
\begin{align*}
    T^{\frac{q}{2+q} + \frac{q\gamma}{2}} \geq \frac{4}{q} \log\left(\frac{8}{q}\right).
\end{align*}
The claim now follows from \Cref{thm:themainthm}.
\hfill$\blacksquare$

\subsection{Proof of \Cref{corr:gammagrowandeps}}
We set $r^2=1/T^{1+\gamma}$.
By the given assumptions,
we can construct a $r/\sqrt{8}$-net $\scrF_r$ of $\partial B(r)$
in the $\norm{\cdot}_\infty$-norm
that (a) satisfies
\begin{align*}
    |\scrF_r| \leq p \log^q\left(\frac{\sqrt{8}}{r}\right),
\end{align*}
and (b) satisfies the trajectory
$(C(r/\sqrt{8}), \alpha)$-hypercontractivity condition. 
Recalling the bounds
$\opnorm{\Gammadep(\Px)}^2 \leq T^{b_1}$
and $C(r) \leq (1/r)^{b_2}$, we have:
\begin{align*}
    B^2 |\scrF_r| \exp\left( \frac{-Tr^{4-2\alpha}}{8C(r/\sqrt{8})\opnorm{\Gammadep(\Px)}^2} \right) &\leq B^2 \exp\left\{ p \log^q\left(\frac{\sqrt{8}}{r}\right) - \frac{T^{1-b_1} r^{4-2\alpha+b_2}}{8^{1+b_2/2} }\right\} \\
    &= B^2 \exp\left\{ p \log^q\left(\sqrt{8} T^{\frac{1+\gamma}{2}}\right) - \frac{T^{1-b_1 - \frac{(1+\gamma)(4-2\alpha+b_2)}{2}}}{8^{1+b_2/2} }\right\} \\
    &= B^2 \exp\left\{ p \log^q\left(\sqrt{8} T^{\frac{1+\gamma}{2}}\right) - \frac{T^{\psi}}{8^{1+b_2/2} }\right\} \\
    &\leq B^2 \exp\left\{ p \log^q\left(\sqrt{8} T^{\frac{1+\gamma}{2}}\right) - \frac{T^{\psi}}{64 }\right\}. 
\end{align*}
Above, the last inequality holds since
$b_2 < 2$.
Now, we choose $T$ large enough so that:
\begin{align*}
    p \log^q\left(\sqrt{8} T^{\frac{1+\gamma}{2}}\right) - \frac{T^{\psi}}{64} \leq - \frac{T^{\psi}}{128} &\Longleftrightarrow T^\psi \geq 128 p \log^q\left(\sqrt{8} T^{\frac{1+\gamma}{2}}\right) \\
    &\Longleftrightarrow T^{\psi/q} \geq (128 p)^{1/q} \left(\log(\sqrt{8}) + \frac{1+\gamma}{2\psi/q} \log(T^{\psi/q}) \right).
\end{align*}
Thus, it suffices to require that:
\begin{align*}
    T^{\psi/q} \geq (128p)^{1/q} \log{8}, \:\: T^{\psi/q} \geq (128p)^{1/q} \frac{1+\gamma}{\psi/q} \log(T^{\psi/q}).
\end{align*}
By \citet[Lemma A.4]{simchowitz2018learning}, the right hand
side inequality holds when:
\begin{align*}
    T^{\psi/q} \geq 2(128p)^{1/q} \frac{1+\gamma}{\psi/q}
\log\left( 4 (128p)^{1/q} \frac{1+\gamma}{\psi/q} \right).
\end{align*}
We finish the proof by finding $T$ such that:
\begin{align*}
    B^2 \exp\left( \frac{-T^{\psi}}{128} \right) \leq \frac{1}{T^{1+\gamma}} &\Longleftrightarrow T^{\psi} \geq 128 \log(B^2 T^{1+\gamma}) \\
    &\Longleftrightarrow T^\psi \geq 256 \log{B} + 128 \frac{(1+\gamma)}{\psi} \log(T^{\psi}).
\end{align*}
Thus, it suffices to require that:
\begin{align*}
    T^\psi \geq 512 \log{B}, \:\: T^\psi \geq 256\frac{(1+\gamma)}{\psi} \log(T^{\psi}).
\end{align*}
Another application of
\citet[Lemma A.4]{simchowitz2018learning} yields that the latter inequality holds if:
\begin{align*}
    T^\psi \geq 512 \frac{1+\gamma}{\psi} \log\left( 1024 \frac{1+\gamma}{\psi}\right).
\end{align*}
Combining all our requirements on $T$, we require that
$T \geq \max\{T_1, T_2\}$, with:
\begin{align*}
    T_1 &\triangleq \max\left\{(128 p)^{1/\psi} (\log{8})^{q/\psi}, (128 p)^{1/\psi} \left[  \frac{4q}{\psi} \log\left( (128p)^{1/q} \frac{8q}{\psi}\right) \right]^{q/\psi} \right\}, \\
    T_2 &\triangleq \max\left\{(512 \log{B})^{1/\psi}, \left[ \frac{1024}{\psi} \log\left(\frac{2056}{\psi}\right)\right]^{1/\psi}\right\}.
\end{align*}
\hfill $\blacksquare$

\section{Proofs for Section~\ref{sec:examples}}

\subsection{Proof of Proposition~\ref{prop:mcmoments}}
For notational brevity, we make the identification of the atoms $\{\psi_1, \dots, \psi_K\}$ with the integers $\{1, \dots, K\}$.
Fix a function $f : \{1, \dots, K\} \to \R^{\dy}$. 
For any time indices $t_1, t_2 \in \{0, \dots, T-1\}$:
\begin{align*}
    \E\norm{f(X_{t_1})}_2^2 \E\norm{f(X_{t_2})}_2^2 &= \left( \sum_{k=1}^{K} \norm{f(k)}_2^2 \mu_{t_1}(k) \right)\left( \sum_{k=1}^{K} \norm{f(k)}_2^2 \mu_{t_2}(k) \right) \\
    &= \sum_{k_1=1}^{K} \sum_{k_2=1}^{K} \norm{f(k_1)}_2^2 \norm{f(k_2)}_2^2 \mu_{t_1}(k_1) \mu_{t_2}(k_2) 
    \geq \sum_{k_1=1}^{K} \norm{f(k_1)}_2^4 \mu_{t_1}(k_1) \mu_{t_2}(k_1) \\
    &\geq \mumin \sum_{k_1=1}^{K} \norm{f(k_1)}_2^4 \mu_{t_1}(k_1) 
    = \mumin \E\norm{f(X_{t_1})}_2^4.
\end{align*}
Therefore:
\begin{align*}
    \left( \frac{1}{T} \sum_{t=0}^{T-1} \E \norm{f(X_t)}_2^2 \right)^2 &= \frac{1}{T^2} \sum_{t_1=0}^{T-1} \sum_{t_2=0}^{T-1} \E \norm{f(X_{t_1})}^2 \E\norm{f(X_{t_2})}_2^2 \\
    &\geq \frac{\mumin}{T^2} \sum_{t_1=0}^{T-1} \sum_{t_2=0}^{T-1} \E \norm{f(X_{t_1})}_2^4 
    = \frac{\mumin}{T} \sum_{t_1=0}^{T-1} \E\norm{f(X_{t_1})}_2^4.
\end{align*}
The claim now follows since we assume $\mumin > 0$.\hfill $\blacksquare$

\subsection{Proof of Proposition~\ref{prop:momeqv}}
Recall that $\E\left[ \frac{1}{T}\sum_{t=0}^{T-1} \| f(X_t)\|^p_2\right] = \|f\|_{L^p}^p$
for $p \geq 1$. We estimate the left hand side of inequality (\ref{cond:hyperconttrajrelax}) as follows:
\begin{align*}
     \E\left[ \frac{1}{T}\sum_{t=0}^{T-1} \| f(X_t)\|^4_2\right]&= \E\left[ \frac{1}{T}\sum_{t=0}^{T-1} \| f(X_t)\|^{2-\e}_2\| f(X_t)\|^{2+\e}_2\right]\\
     &\leq B^{2-\e} \E\left[ \frac{1}{T}\sum_{t=0}^{T-1}\| f(X_t)\|^{2+\e}_2\right] &&(\textnormal{$B$-bounded})\\
     &=B^{2-\e}\|f\|_{L^{2+\e}}^{2+\e}\\
     &\leq B^{2-\e}(c\|f\|_{L^{2}})^{2+\e}&&\textnormal{($L^2-L^{2+\e}$-equivalence)}\\
     &=B^{2-\e}c^{2+\e}\|f\|_{L^{2}}^{2+\e}\\
     &=B^{2-\e}c^{2+\e} \left( \E\left[ \frac{1}{T}\sum_{t=0}^{T-1} \| f(X_t)\|^{2}_2\right]\right)^{1+\e/2}.
\end{align*}
The result now follows.
\hfill$\blacksquare$

\subsection{Proof of \Cref{prop:eqvfromstat_v2}}
We first state an auxiliary proposition.
\begin{proposition}
\label{prop:chisq_bound}
Let $\mu$ and $\nu$ be distributions satisfying
$\mu \ll \nu$.
Let $g$ be any measurable function such that
$\E_{\nu} g^2 < \infty$. We have:
\begin{align*}
    \E_{\mu} g - \E_{\nu} g \leq \sqrt{\E_{\nu} g^2} \sqrt{ \chi^2(\mu, \nu) }.
\end{align*}
\end{proposition}
\begin{proof}
By Cauchy-Schwarz:
\begin{align*}
    \E_\mu g - \E_\nu g = \int g \left(\frac{\rmd\mu}{\rmd\nu} - 1\right)\rmd\nu \leq \sqrt{\int g^2 \,\rmd\nu} \sqrt{ \int \left(\frac{\rmd\mu}{\rmd\nu} - 1\right)^2\rmd\nu} = \sqrt{\E_\nu g^2} \sqrt{ \chi^2(\mu, \nu) }.
\end{align*}
\end{proof}
We can now complete the proof of \Cref{prop:eqvfromstat_v2}.
Fix any $f \in \scrF_\star$.
First, we note that the the condition \eqref{eq:L8_L2_v2} implies:
\begin{align}
    \E_\pi \norm{f}_2^4 \leq (\E_\pi\norm{f}_2^8)^{1/2} \leq (C_{8 \to 2} (\E_\pi\norm{f}_2^2)^4)^{1/2} = \sqrt{C_{8 \to 2}} (\E_\pi \norm{f}_2^2)^2. \label{eq:L4_L2_v2}
\end{align}
Therefore, for any $f \in \scrF_\star$
and any $t \in \N$:
\begin{align}
    \E\norm{f(X_t)}_2^4 &\leq \E_\pi\norm{f}_2^4 + \sqrt{\E_\pi\norm{f}_2^8} \sqrt{\chi^2(\mu_t, \pi)} && \text{using \Cref{prop:chisq_bound}} \nonumber \\
    &\leq (1 + \sqrt{C_{\chi^2}})\sqrt{C_{8\to 2}}(\E_\pi\norm{f}_2^2)^2 &&\text{using \eqref{eq:chisq_ergodicity_v2}, \eqref{eq:L8_L2_v2}, and \eqref{eq:L4_L2_v2}}. \label{eq:L4_L2sq_for_pi}
\end{align}
Now let $f \in \partial B(r)$.
By \citet[Lemma 1]{kuznetsov2017generalization},
since $\norm{f(x)}_2^2 \in [0, B^2]$, we have:
\begin{align}
    \E_\pi\norm{f}_2^2 - \E\norm{f(X_t)}_2^2 \leq B^2 \tvnorm{\mu_t - \pi}. \label{eq:tv_norm_ineq}
\end{align}
Therefore:
\begin{align}
    \E_\pi\norm{f}_2^2 &= \frac{1}{T}\sum_{t=0}^{T-1} (\E_\pi\norm{f}_2^2 - \E\norm{f(X_t)}_2^2) + r^2 &&\text{since } f \in \partial B(r) \nonumber \\
    &\leq \frac{B^2}{T} \sum_{t=0}^{T-1} \tvnorm{\mu_t - \pi} + r^2 &&\text{using \eqref{eq:tv_norm_ineq}} \nonumber \\
    &\leq (1 + C_{\mathsf{TV}} B^2) r^2 &&\text{using \eqref{eq:chisq_ergodicity_v2}}. \label{eq:L2_for_pi_to_rsq}
\end{align}
Combining these inequalities:
\begin{align*}
    \frac{1}{T}\sum_{t=0}^{T-1} \E\norm{f(X_t)}_2^4 &\leq (1 + \sqrt{C_{\chi^2}}) \sqrt{C_{8\to 2}} (\E_\pi \norm{f}_2^2)^2 &&\text{using \eqref{eq:L4_L2sq_for_pi}} \\
    &\leq (1 + \sqrt{C_{\chi^2}}) \sqrt{C_{8\to 2}} (1 + C_{\mathsf{TV}} B^2)^2 r^4 &&\text{using \eqref{eq:L2_for_pi_to_rsq}} \\
    &= (1 + \sqrt{C_{\chi^2}}) \sqrt{C_{8\to 2}} (1 + C_{\mathsf{TV}} B^2)^2 \left( \frac{1}{T}\sum_{t=0}^{T-1} \E\norm{f(X_t)}_2^2\right)^2 &&\text{since } f \in \partial B(r).
\end{align*}
The claim now follows.\hfill$\blacksquare$

\subsection{Proof of Proposition~\ref{prop:infdimex}}






\paragraph{Covering:} We first approximate $\mathscr{P}$ by a finite-dimensional ellipsoid at resolution $\e/4$.
To this end, fix an integer $m \in \N_+$ and define:
\begin{align*}
\mathscr{P}_m =\left\{ f= \sum_{j=1}^m\theta_j \phi_j \,\Bigg|\, \sum_{j=1}^\infty \frac{\theta_j^2}{\mu_j}\leq 1 \right\}.
\end{align*}

Fix now an element $f \in \mathscr{P}$ with coordinates $\theta$. Let $f'$ be the orthogonal projection onto the subspace of the first $m$-many coordinates ($f' \in \mathscr{P}_m$). Then:
\begin{equation}
  \label{eq:truncbound}
  \begin{aligned}
   \|f-f'\|_{\infty}&= \left\|\sum_{j=m+1}^\infty \theta_j \phi_j  \right\|_{\infty} 
   \leq \left\|\underbrace{\sqrt{\sum_{j=m+1}^\infty\frac{\theta_j^2}{\mu_j} }}_{\leq 1}\sqrt{\sum_{j=m+1}^\infty\mu_j \|\phi_j\|_2^2} \right\|_{\infty} && (\textnormal{Cauchy-Schwarz})\\
   &\leq B\sqrt{\sum_{j=m+1}^\infty j^{2q} e^{-2\beta j}}&& (\|\phi_j\|_\infty \leq B j^q ,\mu_j\leq e^{-2\beta j})\\
   &\leq B\sqrt{\sum_{j=m+1}^\infty e^{-\beta j}}&& \left(\textnormal{if } \frac{m}{\log m} \geq \frac{q}{\beta}\right)\\
   & =   B \frac{e^{-\beta m/2}}{\sqrt{e^{\beta}-1}} \leq 2B \frac{e^{-\beta m/2}}{\beta}. && (\sqrt{e^{2x}-1} \geq e^x-1 \geq x, x\geq 0)
  \end{aligned}
\end{equation}
Hence, we can take $m_\e$ to be the smallest integer solution to $m \geq \frac{2}{\beta} \left| \log \left(\frac{8B}{\beta \e}\right) \right|$ to guarantee that for every $f\in \mathscr{P}$ there exists $f' \in \mathscr{P}_m$ at most $\e/4$ removed from $f$, i.e.,\ $\|f-f'\|_{\infty} \leq \e/4$. 

Next, we construct an $\e/4$-covering of the set $\scrP_m$.
Observe now that the set of parameters of $\Theta_m$ defining $\mathscr{P}_m$ satisfies:
\begin{align*}
\Theta_m \triangleq \left\{ \theta \in \mathbb{R}^m \,\Bigg|\,  \sum_{j=1}^m \frac{\theta_j^2}{\mu_j} \leq  1\right\}.
\end{align*}
Using this, we obtain a covering of $(\mathscr{P}_m, \|\cdot\|_{\infty})$ by regarding it as a subset of $\mathbb{R}^m$. More precisely, $\Theta_m$ is the unit ball in the norm  $\norm{\theta}_{\mu} \triangleq \sqrt{\sum_{i=1}^{m} \theta_i^2/\mu_i}, \theta \in \mathbb{R}^{m}$. Hence, by a standard volumetric argument, we need no more than $(1+2/\delta)^m$ points to cover $\Theta_m$ at resolution $\delta$ in $\norm{\cdot}_\mu$. Let now $\delta>0$ and choose $N \in \mathbb{N}_+$ so that $\{\theta^1,\dots,\theta^N\}$ is an optimal $\delta$-covering of $\Theta_m$. We thus obtain the cover $\mathscr{P}_m^N\triangleq \{ (\theta^1)^\top \phi(\cdot), \dots, (\theta^N)^\top \phi(\cdot) \}\subset \mathscr{P}_m$ where $\phi(\cdot) = (\phi_1 (\cdot),\dots,\phi_m(\cdot))$. Let $f'=(\theta')^\top \phi \in \mathscr{P}_m$ be arbitrary. It remains to verify the resolution of $\mathscr{P}_m^N$:
\begin{align*}
    \min_{n\in[N]} \| f'-(\theta^n)^\top \phi\|_\infty &=\min_{n\in [N]}\left\|\sum_{j=1}^m (\theta'_j-\theta^n_j) \phi_j  \right\|_{\infty}\\
    &\leq\min_{n\in [N]} \left\|\sqrt{\sum_{j=1}^m\frac{(\theta'_j-\theta^n_j)^2}{\mu_j} } \sqrt{\sum_{j=1}^m \mu_j \|\phi_j\|_2^2}\right\|_{\infty} && (\textnormal{Cauchy-Schwarz})\\
     &\leq \delta Bm^q&& (\|\phi_j\|_\infty \leq Bm^q \textnormal{ if } j \leq m).
\end{align*}
Hence, if we take $N$ large enough so that $\delta \leq \frac{\e}{4Bm^q}$, $\mathscr{P}_m^N$ is a cover of $\mathscr{P}_m$ at resolution $\e/4$.  Hence, since we may take $m \leq m_\e$:
\begin{align*}
    N \leq \left(1+\frac{8Bm_\e^{q}}{\e}\right)^{m_\e}.
\end{align*}

Now, we can immediately convert the covering $\scrP_m^N$ into an \emph{exterior cover}\footnote{An exterior cover of a set $T$ is a cover where the elements are not restricted to $T$.}
of the set $P$. 
For every $f \in P$, by the approximation property of
$\scrP_m$, there exists an $f' \in \scrP_m$ such that
$\norm{f - f'}_\infty \leq \e/4$.
But since $f' \in \scrP_m$, there exists an $f'' \in \scrP_m^N$
such that $\norm{f' - f''}_\infty \leq \e/4$.
By triangle inequality, $\norm{f - f''}_\infty \leq \e/2$.
Thus, $\scrP_m^N$ forms an exterior cover of
$P$ at resolution $\e/2$. By \citet[Exercise 4.2.9]{vershynin2018},
this means that there exists a (proper) cover of 
$P$ at resolution $\e$ with cardinality bounded by
$N$.

\paragraph{Hypercontractivity:}
We first show that every $f \in \scrP_m$ is hypercontractive.
First, observe by orthogonality that the second moment takes the form:
\begin{align*}
\int \left\|\sum_{i=1}^{m_\e} \theta_i \phi_i \right\|^2_2 \rmd\lambda &= \sum_{i=1}^{m_\e} \theta_i^2 \|\phi_i\|_{L_2(\lambda)}^2 = \norm{\theta}_2^2.
\end{align*}
On the other hand by the eigenfunction growth condition:
\begin{align*}
    \int \left\|\sum_{i=1}^{m_\e} \theta_i \phi_i \right\|^4_2 \rmd\lambda &\leq \int \left( \sum_{i=1}^{m_\e} |\theta_i| \norm{\phi_i}_2 \right)^4 \rmd\lambda \\
    &\leq B^4 m_\e^{4q}\left(\sum_{i=1}^{m_e} |\theta_i| \right)^4 \\
    &\leq B^4 m_\e^{4q+2} \norm{\theta}_2^4 \\
    &= B^4 m_\e^{4q+2} \left(\int \left\| \sum_{i=1}^{m_\e} \theta_i\phi_i \right\|_2^2 \rmd\lambda\right)^2.
\end{align*}
Now, for any $t \in \N$, by a change of measure,
with $f = \sum_{i=1}^{m_\e} \theta_i \phi_i$,
\begin{align*}
    \E_{\mu_t} \left\|f \right\|^4_2 &= \int \left\|f \right\|^4_2 \frac{\rmd\mu_t}{\rmd\lambda} \,\rmd\lambda \leq K \int \left\|f \right\|^4_2 \rmd\lambda 
    \leq K B^4 m_\e^{4q+2}\left(\int \left\| f \right\|_2^2 \rmd\lambda\right)^2.
\end{align*}
Hence, applying the previous inequality and another change of measure:
\begin{align}
    \frac{1}{T}\sum_{t=0}^{T-1} \E_{\mu_t} \norm{f}_2^4 &\leq KB^4 m_\varepsilon^{4q+2}\left(\int \left\| f \right\|_2^2 \rmd\lambda\right)^2 \nonumber \\
    &= KB^4 m_\varepsilon^{4q+2}\left( \frac{1}{T} \int \sum_{t=0}^{T-1} \norm{f}_2^2 \frac{\rmd\lambda}{\rmd\mu_t} \,\rmd\mu_t \right)^2 \nonumber \\
    &\leq K^3 B^4 m_\varepsilon^{4q+2}\left( \frac{1}{T} \sum_{t=0}^{T-1} \E_{\mu_t} \norm{f}_2^2 \right)^2. \label{eq:rkhs_lemma_traj_hyp}
\end{align}

Next, fix a $f \in P$. We will show that $f$ is hypercontractive.
First, recall that $f'$ is the element in $\scrP_m$ satisfying
$\norm{f - f'}_\infty \leq \e/4$.
Hence, we have for every $x$:
\begin{align}
    \norm{f(x)}_2^4 &\leq 8(\norm{f(x)-f'(x)}_2^4 + \norm{f'(x)}_2^4) \leq \frac{\e^4}{32} + 8\norm{f'(x)}_2^4, \label{eq:rkhs_lemma_r4} \\
    \norm{f'(x)}_2^2 &\leq 2(\norm{f(x)-f'(x)}_2^2 + \norm{f(x)}_2^2) \leq \frac{\e^2}{2} + 2\norm{f(x)}_2^2. \label{eq:rkhs_lemma_r2}
\end{align}
We now bound:
\begin{align*}
    \frac{1}{T} \sum_{t=0}^{T-1} \E_{\mu_t} \norm{f}_2^4 &\stackrel{(a)}{\leq} \frac{\e^4}{32} + \frac{1}{T}\sum_{t=0}^{T-1} \E_{\mu_t}\norm{f'}_2^4 
    \stackrel{(b)}{\leq} \frac{\e^4}{32} + K^3 B^4 m_\e^{4q+2}\left(\frac{1}{T} \sum_{t=0}^{T-1} \E_{\mu_t}\norm{f'}_2^2\right)^2 \\
    &\stackrel{(c)}{\leq} \frac{\e^4}{32} + K^3 B^4 m_\e^{4q+2}\left( \frac{\e^2}{2} + \frac{2}{T} \sum_{t=0}^{T-1} \E_{\mu_t}\norm{f}_2^2 \right)^2 \\
    &\stackrel{(d)}{\leq} \left(\frac{1}{32} + \frac{25}{4} K^3 B^4 m_\e^{4q+2} \right) \left( \frac{1}{T} \sum_{t=0}^{T-1} \E_{\mu_t} \norm{f}_2^2\right)^2.
\end{align*}
Above, (a) uses the inequality \eqref{eq:rkhs_lemma_r4},
(b) uses the fact that $f' \in \scrP_m$ and \eqref{eq:rkhs_lemma_traj_hyp},
(c) uses \eqref{eq:rkhs_lemma_r2},
and (d) uses the assumption that $\e \leq \inf_{f \in P} \norm{f}_{L^2(\Px)}$,
which implies that
$\e^2 \leq \frac{1}{T} \sum_{t=0}^{T-1} \E_{\mu_t}\norm{f}_2^2$
and $\e^4 \leq \left( \frac{1}{T} \sum_{t=0}^{T-1} \E_{\mu_t}\norm{f}_2^2\right)^2$.
Since $f \in P$ is arbitrary, the claim follows.
\hfill $\blacksquare$

\section{Basic tools for analyzing the dependency matrix}
\label{sec:basictoolsdepmat}

In this section, we outline some basic tools used
to analyze the dependency matrix $\Gammadep(\Px)$.
We will introduce the following shorthand.
Given a process $\{Z_t\}_{t \geq 0}$ and indices $0 \leq i \leq j \leq k$,
we will write $\sfP_{Z_{j:k}}(\cdot \mid Z_{0:i}=z_{0:i})$
as shorthand for $\sfP_{Z_{j:k}}(\cdot \mid A)$ for $A \in \calZ_{0:i}$,
where we recall that $\calZ_{0:i}$ denotes
the $\sigma$-algebra generated by $Z_{0:i}$.
We will also write $\esssup_{z_{0:i} \in \sfZ_{0:i}}$
as shorthand for $\sup_{A \in \calZ_{0:i}}$.

Before we proceed, we recall the coupling representation of
the total-variation norm:
\begin{align}
    \tvnorm{\mu - \nu} = \inf\{ \Pr(X \neq Y) \mid (X, Y) \text{ is a coupling of } (\mu, \nu) \}. \label{eq:tv_norm_coupling_def}
\end{align}

\begin{proposition}
\label{prop:tv_reduction_markov}
Suppose that $\{Z_t\}_{t \geq 0}$ is a Markov chain. For any integers $0 \leq i \leq j \leq k$:
\begin{align*}
    \tvnorm{\sfP_{Z_{j:k}}(\cdot\mid Z_i=z) - \sfP_{Z_{j:k}}} = \tvnorm{\sfP_{Z_j}(\cdot\mid Z_i=z) - \sfP_{Z_j}}.
\end{align*}
\end{proposition}
\begin{proof}
Let us first prove the upper bound.
Let $(Z_j, Z'_j)$ be a coupling of
$(\sfP_{Z_j}(\cdot \mid Z_i=z), \sfP_{Z_j})$.
We can construct a coupling
$(\bar{Z}_{j:k}, \bar{Z}'_{j:k})$
of $(\sfP_{Z_{j:k}}(\cdot \mid Z_i=z), \sfP_{Z_{j:k}})$
by first setting $\bar{Z}_{j} = Z_j$,
$\bar{Z}'_j = Z'_j$, and then evolving the chains
onward via the following process.
If $\bar{Z}_{j} = \bar{Z}'_j$, we evolve
$\bar{Z}_{j+1:k}$ onwards according to the dynamics, and copy 
$\bar{Z}'_{j+1:k} = \bar{Z}_{j+1:k}$. Otherwise if $\bar{Z}_{j} \neq \bar{Z}'_j$, then
we evolve both chains separately.
Observe that $\bar{Z}_{j:k} \neq \bar{Z}'_{j:k}$
iff $Z_j \neq Z'_j$.
Hence
$\tvnorm{\sfP_{Z_{j:k}}(\cdot\mid Z_i=x) - \sfP_{Z_{j:k}}} \leq \Pr(Z_j \neq Z'_j)$.
Since the coupling $(Z_j,Z'_j)$ is arbitrary, taking
the infimum over all couplings of 
$(\sfP_{Z_j}(\cdot \mid Z_i=z), \sfP_{Z_j})$ yields the upper bound via \eqref{eq:tv_norm_coupling_def}.

We now turn to the lower bound. Let 
$(Z_{j:k}, Z'_{j:k})$ be a coupling
of $(\sfP_{Z_{j:k}}(\cdot \mid Z_i=z), \sfP_{Z_{j:k}})$.
Since projection $(Z_j, Z'_j)$ is a coupling
for $(\sfP_{Z_j}(\cdot \mid Z_i=z), \sfP_{Z_j})$,
and $Z_j \neq Z'_j$ implies $Z_{j:k} \neq Z'_{j:k}$,
we have again by \eqref{eq:tv_norm_coupling_def}:
\begin{align*}
    \tvnorm{\sfP_{Z_j}(\cdot\mid Z_i=x) - \sfP_{Z_j}} \leq  \Pr(Z_j \neq Z'_j) \leq \Pr(Z_{j:k} \neq Z'_{j:k}).
\end{align*}
Taking the infimum over all couplings
of $(\sfP_{Z_{j:k}}(\cdot \mid Z_i=z), \sfP_{Z_{j:k}})$
yields the lower bound.
\end{proof}

\begin{proposition}
\label{prop:entrywise_opnorm_bound}
Let $M, N$ be two size conforming matrices
with all non-negative entries. Suppose that
$M \leq N$, where the inequality holds elementwise.
Then, $\opnorm{M} \leq \opnorm{N}$.
\end{proposition}
\begin{proof}
Let $Q$ be a matrix with non-negative entries,
and let $q_i$ denote the rows of $Q$.
The variational form of the operator norm states that $\opnorm{Q} = \sup_{\norm{v}_2 \leq 1} \norm{Q v}_2 = \sup_{\norm{v}_2\leq 1} \sqrt{\sum_{i} \ip{q_i}{v}^2}$.
Since each $q_i$ only has non-negative entries,
the supremum must be attained by a vector $v$ with non-negative entries, otherwise flipping the sign of the negative entries in $v$ would only possibly increase the value of $\norm{Qv}_2$,
and never decrease the value.

Now let $m_i, n_i$ denote the rows of $M,N$, and let $v$ be a vector with non-negative entries.
Since $0 \leq m_i \leq n_i$ (elementwise), it is clear that $\ip{m_i}{v}^2 \leq \ip{n_i}{v}^2$.
Hence the claim follows.
\end{proof}

\begin{proposition}
\label{prop:upper_triangle_nilpotent}
Let $a_1, \dots, a_n \in \R$, and let $M \in \R^{n \times n}$ be the upper
triangular Toeplitz matrix:
\begin{align*}
    M = \begin{bmatrix}
    a_1 & a_2 & a_3 & a_4 & \cdots & a_n\\
    0 & a_1 & a_2 & a_3 & \cdots & a_{n-1} \\
    0 & 0 & a_1 & a_2 & \cdots & a_{n-2} \\
    \vdots & \vdots & \vdots & \vdots & \ddots & \vdots \\
    0 & 0 & 0 & 0 & \vdots & a_1
    \end{bmatrix} .
\end{align*}
We have that:
\begin{align*}
    \opnorm{M} \leq \sum_{i=1}^{n} |a_i|.
\end{align*}
\end{proposition}
\begin{proof}
Let $E_i$, for $i=1, \dots, n$,
denote the shift matrix
where $E_i$ has ones along the $(i-1)$-th super diagonal and is zero everywhere else (the zero-th diagonal refers to the main diagonal).
It is not hard to see that
$\opnorm{E_i} \leq 1$ for all $i$, since it simply selects (and shifts) a subset of the coordinates of the input.
With this notation, $M = \sum_{i=1}^{n} a_i E_i$.
The claim now follows by the triangle inequality.
\end{proof}

\begin{proposition}
\label{prop:gamma_dep_bound_markov}
Let $\{Z_t\}_{t\geq 0}$ be a Markov process,
and let $\sfP_Z$ denote the joint distribution of
$\{Z_t\}_{t=0}^{T-1}$.
We have that:
\begin{align*}
    \opnorm{\Gammadep(\sfP_Z)} \leq 1 + \sqrt{2} \sum_{k=1}^{T-1} \max_{t=0, \dots, T-1-k}
    \esssup_{z \in \sfZ_{t}} \sqrt{ \tvnorm{ \sfP_{Z_{t+k}}(\cdot \mid Z_t=z) - \sfP_{Z_{t+k}}}}.
\end{align*}
\end{proposition}
\begin{proof}
For any indices $0 \leq i < j$, by the Markov property
and \Cref{prop:tv_reduction_markov}:
\begin{align*}
    \esssup_{z_{0:i} \in \sfZ_{0:i}} \tvnorm{ \sfP_{Z_{j:T-1}}(\cdot \mid Z_{0:i}=z_{0:i}) - \sfP_{Z_{j:T-1}}} &= \esssup_{z \in \sfZ_i} \tvnorm{\sfP_{Z_{j:T-1}}(\cdot\mid Z_i=z) - \sfP_{Z_{j:T-1}}} \\
    &= \esssup_{z \in \sfZ_i} \tvnorm{\sfP_{Z_{j}}(\cdot\mid Z_i=z) - \sfP_{Z_{j}}}.
\end{align*}
Therefore:
\begin{align*}
    \Gammadep(\sfP_Z)_{ij} &= \sqrt{2} \esssup_{z \in \sfZ_i} \sqrt{ \tvnorm{\sfP_{Z_{j}}(\cdot\mid Z_i=z) - \sfP_{Z_{j}}}} \\
    &\leq \sqrt{2} \max_{t = 0, \dots, T-1-(j-i)} \esssup_{z \in \sfZ_t} \sqrt{ \tvnorm{\sfP_{Z_{t+j-i}}(\cdot \mid Z_t=z) - \sfP_{Z_{t+j-i}}}} \triangleq a_{j-i}.
\end{align*}
Thus, we can construct a matrix $\Gamma'$
such that for all indices $0 \leq i < j$,
we have $\Gamma'_{ij} = a_{j-i}$, 
and the other entries are identical to $\Gammadep(\sfP_Z)$.
This gives us the entry-wise bound $\Gammadep(\sfP_Z) \leq \Gamma'$.
Applying \Cref{prop:entrywise_opnorm_bound}
and \Cref{prop:upper_triangle_nilpotent},
we conclude $\opnorm{\Gammadep(\sfP_Z)} \leq \opnorm{\Gamma'} \leq 1 + \sum_{k=1}^{T-1} a_k$.
\end{proof}

\section{Mixing properties of truncated Gaussian processes}
\label{sec:appendix:truncated_gaussian}

We first recall the notation from \Cref{sec:proofs:unbounded}.
Let $\{W_t\}_{t \geq 0}, \{W'_t\}_{t \geq 0}$ be sequences of \iid\ $N(0, I)$ vectors in $\R^{\dx}$.
Fix a dynamics function $f : \R^{\dx} \to \R^{\dx}$
and radius $R > 0$. Define the truncated Gaussian noise process $\{\bar{W}_t\}_{t \geq 0}$ as
$\bar{W}_t \triangleq W'_t \ind\{\norm{W'_t}_2 \leq R \}$.
Now, consider the two processes:
\begin{subequations}
\begin{align}
    X_{t+1} &= f(X_t) + HW_t, \:\: X_0 = HW_0, \label{eq:dynamic_process_gaussian} \\
    \bar{X}_{t+1} &= f(\bar{X}_t) + H\bar{W}_t, \:\: \bar{X}_0 = H\bar{W}_0. \label{eq:dynamic_process_trunc}
\end{align}
\end{subequations}
We develop the necessary arguments in this section to
transfer mixing properties of the original process \eqref{eq:dynamic_process_gaussian}
to the truncated process \eqref{eq:dynamic_process_trunc}.
This will let us apply our results in \Cref{sec:results}
to unbounded processes of the form \eqref{eq:dynamic_process_gaussian}, by studying their truncated
counterparts \eqref{eq:dynamic_process_trunc}.

The main tool to do this is the following coupling argument.
\begin{proposition}
\label{prop:coupling_bound}
Fix a $\delta \in (0, 1)$.
Let $k \in \{1, \dots, T-1\}$ and
$t \in \{0, \dots, T-1-k\}$.
Consider the processes $\{X_t\}_{t \geq 0}$ and $\{\bar{X}_t\}_{t \geq 0}$
described in \eqref{eq:dynamic_process_gaussian}
and \eqref{eq:dynamic_process_trunc} with
$R$ satisfying the inequality $R \geq \sqrt{\dx} + \sqrt{2\log(T/\delta)}$.
The following bound hold for any $x \in \R^{\dx}$:
\begin{align*}
    \tvnorm{\sfP_{X_{t+k}}(\cdot \mid X_t = x) - \sfP_{\bar{X}_{t+k}}(\cdot \mid \bar{X}_t=x)} \leq \delta.
\end{align*}
The following bound also holds for any
$t \in \{0,\dots, T-1\}$:
\begin{align*}
    \tvnorm{\sfP_{X_t} - \sfP_{\bar{X}_t}} \leq \delta.
\end{align*}
\end{proposition}
\begin{proof}
Let $(Z_{t+k}, Z'_{t+k})$ be a coupling of
$(\sfP_{X_{t+k}}(\cdot \mid X_t = x), \sfP_{\bar{X}_{t+k}}(\cdot \mid \bar{X}_t=x))$ defined as follows.
We initialize both $X_t=\bar{X}_t=x$.
We let $\{W_s\}_{s=t}^{t+k-1}$ be iid draws from $N(0, I)$, we set $\bar{W_s} = W_s \ind\{\norm{W_s}_2 \leq R\}$,
and we evolve $X_t,\bar{X}'_t$ forward to $X_{t+k},\bar{X}'_{t+k}$ according to their laws
\eqref{eq:dynamic_process_gaussian} and
\eqref{eq:dynamic_process_trunc}, respectively.
Let $\calE$ denote the event $\calE = \{ \max_{s=t, \dots, t+k-1} \norm{W_s}_2 \leq R\}$.
A standard Gaussian concentration plus union bound yields $\Pr(\calE^c) \leq \delta$, since $t + k - 1\leq T-2$.
By the coupling representation 
\eqref{eq:tv_norm_coupling_def}
of the total-variation norm:
\begin{align*}
        \tvnorm{\sfP_{X_{t+k}}(\cdot \mid X_t = x) - \sfP_{\bar{X}_{t+k}}(\cdot \mid \bar{X}_t=x)} &\leq \Pr\{ Z_{t+k} \neq Z_{t+k}' \} \\
        &= \Pr( \{Z_{t+k} \neq Z_{t+k}' \} \cap \calE ) + \Pr( \{Z_{t+k} \neq Z_{t+k}' \} \cap \calE^c ) \\
        &\leq \Pr( \calE^c ) \leq \delta.
\end{align*}
The second inequality holds since on $\calE$,
$Z_{t+k} = Z_{t+k}'$ because the truncation is inactive the entire duration of the process. This establishes the first inequality.

The second inequality holds by a nearly identical coupling argument, where we set $\{W_s\}_{s=0}^{t-1}$ to be \iid\ draws from $N(0, I)$, we set $\bar{W}_s = W_s \ind\{\norm{W_s}_2 \leq R\}$,
and we initialize the processes
at $X_0 = HW_0$ and $\bar{X}_0 = H\bar{W}_0$.
\end{proof}

The next result states that as long as we set the
failure probability $\delta$ in $R$ as $1/T^2$,
then we can bound the dependency matrix appropriately.
\begin{proposition}
\label{prop:final_gamma_bound}
Let $\sfP_{X}$ denote the joint distribution of 
$\{X_t\}_{t=0}^{T-1}$ from \eqref{eq:dynamic_process_gaussian}, and
let $\sfP_{\bar{X}}$ denote the joint distribution of $\{\bar{X}_t\}_{t=0}^{T-1}$ from \eqref{eq:dynamic_process_trunc},
with $R \geq \sqrt{\dx} + \sqrt{6 \log{T}}$.
We have that:
\begin{align}
    \opnorm{\Gammadep(\sfP_{\bar{X}})} \leq 3 + \sqrt{2} \sum_{k=1}^{T-1} \max_{t=0,\dots,T-1-k} \esssup_{x \in \bar{\sfX}_t} \sqrt{\tvnorm{\sfP_{X_{t+k}}(\cdot \mid X_t=x) - \sfP_{X_{t+k}}}}. \label{eq:truncated_proc_dep_bound}
\end{align}
\end{proposition}
\begin{proof}
First, we invoke \Cref{prop:gamma_dep_bound_markov} to obtain:
\begin{align*}
    \opnorm{\Gammadep(\sfP_{\bar{X}})} \leq 1 + \sqrt{2} \sum_{k=1}^{T-1} \max_{t=0, \dots, T-1-k}
    \esssup_{x \in \bar{\sfX}_{t}} \sqrt{ \tvnorm{ \sfP_{\bar{X}_{t+k}}(\cdot \mid \bar{X}_t=x) - \sfP_{\bar{X}_{t+k}}}}.
\end{align*}
Now fix 
$k \in \{1, \dots, T-1\}$, $t \in \{0, \dots, T-1-k\}$, and $x \in \bar{\sf{X}}_t$.
By triangle inequality:
\begin{align*}
    \tvnorm{ \sfP_{\bar{X}_{t+k}}(\cdot \mid \bar{X}_t=x) - \sfP_{\bar{X}_{t+k}} } 
    &\leq \tvnorm{\sfP_{X_{t+k}}(\cdot \mid X_t=x) - \sfP_{X_{t+k}}} \\
    &\qquad + \tvnorm{\sfP_{\bar{X}_{t+k}}(\cdot \mid \bar{X}_t=x) - \sfP_{X_{t+k}}(\cdot \mid X_t=x) } \\
    &\qquad + \tvnorm{ \sfP_{\bar{X}_{t+k}} - \sfP_{X_{t+k}} }.
\end{align*}
By setting $\delta = 1/T^2$ in 
\Cref{prop:coupling_bound}, the last two terms are bounded by $1/T^2$.
Hence:
\begin{align*}
     \tvnorm{ \sfP_{\bar{X}_{t+k}}(\cdot \mid \bar{X}_t=x) - \sfP_{\bar{X}_{t+k}} } 
    \leq \tvnorm{\sfP_{X_{t+k}}(\cdot \mid X_t=x) - \sfP_{X_{t+k}}} + \frac{2}{T^2}.
\end{align*}
The claim now follows.
\end{proof}
Crucially, the essential supremum in
\eqref{eq:truncated_proc_dep_bound} is over
$\bar{\sfX}_t$ and \emph{not} $\sfX_t$, of which the latter is
unbounded.

The next condition that we need to check
for the truncated process \eqref{eq:dynamic_process_trunc}
is that the noise process $\{H \bar{W}_t\}_{t \geq 0}$
is still a zero-mean sub-Gaussian martingale difference
sequence. By symmetry of the truncation, it is clear that the noise process remains zero-mean.
To check sub-Gaussianity, we use the following result.
\begin{proposition}
\label{prop:truncated_sub_gaussian_bound}
Let $A \subseteq \R^{\dx}$ be any set that is symmetric about the origin. 
Let $W \sim N(0,  I)$, and let
$\bar{W} := W \ind\{W \in A\}$.
We have that $\bar{W}$ is $4$-sub-Gaussian.
Hence for any $H$,
$H \bar{W}$ is $4\opnorm{H}^2$-sub-Gaussian.
\end{proposition}
\begin{proof}
Since $A$ is symmetric about the origin, $\bar{W}$ inherits the symmetry of $W$, i.e.,\ $\E[\bar{W}] = 0$.
Now fix a unit vector $u \in \R^{\dx}$, and $\lambda \in \R$.
First, let us assume that $\lambda^2 \leq 1/2$.
Let $\varepsilon$ denote a Rademacher random variable\footnote{That is, $\Pr(\varepsilon = 1) = \Pr(\varepsilon = -1) = 1/2$.} that is independent of $\bar{W}$.
Since $\bar{W}$ is a symmetric zero-mean distribution, we have that $\ip{u}{\bar{W}}$ has the same distribution as
$\varepsilon \ip{u}{\bar{W}}$.
Therefore:
\begin{align*}
    \E \exp(\lambda \ip{u}{\bar{W}}) &= \E_{\bar{W}} \E_{\varepsilon} \exp(\lambda \varepsilon \ip{u}{\bar{W}}) \\
    &\leq \E_{\bar{W}} \exp(\lambda^2 \ip{u}{\bar{W}}^2/2) &&\cosh(x) \leq \exp(x^2/2)\, \forall x \in \R \\
    &\leq \E_{\bar{W}} \exp(\lambda^2 \ip{u}{W}^2/2) \\
    &= \frac{1}{(1-\lambda^2)^{1/2}} &&\text{since } \ip{u}{W} \sim N(0, 1) \text{ and } \lambda^2 < 1 \\
    &\leq \exp(\lambda^2) &&\frac{1}{1-x} \leq \exp(2x)\, \forall x \in [0, 1/2].
\end{align*}
Now, let us assume $\lambda^2 > 1/2$. We have:
\begin{align*}
    \E \exp(\lambda \ip{u}{\bar{W}}) &= \E \exp(\lambda \ip{u}{W}) \ind\{W \in A\} + \Pr(W \not\in A) \\
    &\leq \E \exp(\lambda \ip{u}{W}) + 1 \\
    &= \exp(\lambda^2/2) + 1 &&\text{since } \ip{u}{W} \sim N(0, 1) \\
    &\leq \exp(\log{2} + \lambda^2/2) &&\text{since } 1 \leq \exp(\lambda^2/2) \\
    &\leq \exp((2\log{2} + 1/2)\lambda^2) &&\text{since } \lambda^2 > 1/2 \\
    &\leq \exp(2 \lambda^2).
\end{align*}
The claim now follows.
\end{proof}

The following result will be useful later on.
It states that the truncation does not affect the isotropic
nature of the noise, as long as the truncation probability is
a sufficiently small constant.
\begin{proposition}
\label{prop:trunc_gauss_approximate_isotropic}
Let $A \subseteq \R^{\dx}$ be any set.
Let $W \sim N(0, I)$ and $\bar{W} = W \ind\{ W \in A \}$,
and suppose that $\Pr( W \not\in A ) \leq 1/12$. We have that:
\begin{align*}
    \frac{1}{2} I \preccurlyeq \E[ \bar{W} \bar{W}^\T] \preccurlyeq I.
\end{align*}
\end{proposition}
\begin{proof}
The upper bound is immediate.
For the lower bound, fix a $v \in \mathbb{S}^{\dx-1}$.
We have:
\begin{align*}
    \E[ \ip{v}{\bar{W}}^2 ] &= \E[ \ip{v}{\bar{W}}^2 \ind\{W \in A \} ] + \E[ \ip{v}{\bar{W}}^2 \ind\{W \not\in A \} ] \\
    &= \E[ \ip{v}{\bar{W}}^2 \ind\{W \in A \} ]  &&\text{since } \bar{W} = W \ind\{W \in A\} \\
    &= \E[\ip{v}{W}^2] - \E[ \ip{v}{W}^2 \ind\{ W \not \in A \}] \\
    &\geq 1 - \sqrt{ \E[\ip{v}{W}^4] \Pr(W \not\in A) } &&\text{since } \ip{v}{W} \sim N(0, 1) \text{ and Cauchy-Schwarz} \\
    &\geq 1 - \sqrt{3\delta} \\
    &\geq 1/2 &&\text{since } \Pr(W \not\in A) \leq 1/12.
\end{align*}
Since $v \in \mathbb{S}^{\dx-1}$ is arbitrary, the claim follows.
\end{proof}

\begin{proposition}
\label{prop:gaussian_fourth_moment_bound}
Let $w \sim N(0, I)$ and let $M$ be positive semidefinite. We have:
\begin{align*}
    \E[ (w^\T M w)^2 ] \leq 3 (\E[ w^\T M w ])^2.
\end{align*}
\end{proposition}
\begin{proof}
This is a standard calculation~\citep[see e.g.][Lemma 6.2]{magnus1978gaussianforms}.
\end{proof}

We will also need the following
result which states that the square of quadratic forms
under $\bar{W}$ can be upper bounded by the square of the same
quadratic form under the original noise $W$.
\begin{proposition}
\label{prop:gaussian_truncate_fourth_moment}
Let $A \subseteq \R^{\dx}$ be any set.
Let $W \sim N(0, I)$ and $\bar{W} = W \ind\{ W \in A \}$.
Fix a $k \geq 1$. Let $M \in \R^{\dx k \times \dx k}$
be a positive semidefinite matrix,
and let
$\{W_i\}_{i=1}^{k}$
and $\{\bar{W}_i\}_{i=1}^{k}$ be
\iid\ copies of $W$ and $\bar{W}$, respectively.
Let $W_{1:k} \in \R^{\dx k}$
denote the stacked
column vector of $\{W_i\}_{i=1}^{k}$
and similarly for $\bar{W}_{1:k} \in \R^{\dx k}$.
We have that:
\begin{align*}
    \E[(\bar{W}_{1:k}^\T M \bar{W}_{1:k})^2] \leq \E[ (W_{1:k}^\T M W_{1:k})^2 ].
\end{align*}
\end{proposition}
\begin{proof}
Let $\{M_{ij}\}_{i,j=1}^{k} \subset \R^{\dx \times \dx}$ denote the blocks of $M$.
We have:
\begin{align*}
    \E[(\bar{W}_{1:k}^\T M \bar{W}_{1:k})^2] &= \sum_{a,b,c,d} \E[ (\bar{W}_a^\T M_{ab} \bar{W}_b)(\bar{W}_c^\T M_{cd} \bar{W}_d) ].
\end{align*}
Since $\bar{W}$ is zero-mean, the only terms that are non-zero in the summation have the following form
$\E[ (\bar{W}_a^\T M_{aa} \bar{W}_a) (\bar{W}_b^\T M_{bb} \bar{W}_b) ]$. Hence:
\begin{align*}
    \E[(\bar{W}_{1:k}^\T M \bar{W}_{1:k})^2] &= \sum_{a} \E[ (\bar{W}_a^\T M_{aa} \bar{W}_a)^2 ] + \sum_{a \neq b}\E[ (\bar{W}_a^\T M_{aa} \bar{W}_a) (\bar{W}_b^\T M_{bb} \bar{W}_b) ] \\
    &\stackrel{(a)}{\leq} \sum_{a} \E[ (W_a^\T M_{aa} W_a)^2 ] + \sum_{a \neq b}\E[ (W_a^\T M_{aa} W_a) (W_b^\T M_{bb} W_b) ] \\
    &= \E[ (W_{1:k}^\T M W_{1:k})^2 ].
\end{align*}
Above, (a) holds since the matrix $M$ is positive semidefinite and therefore so are its diagonal sub-blocks $M_{aa}$. This ensures that each of the quadratic forms are
non-negative, and hence we can upper bound the first expression by removing the indicators.
%
\end{proof}

We conclude this section with a result that will be useful for analyzing
the mixing properties of the Gaussian process \eqref{eq:dynamic_process_gaussian}, when the dynamics function $f$ is nonlinear.
First, recall the definition of the $1$-Wasserstein distance:
\begin{align}
    W_1(\mu, \nu) \triangleq \inf\{ \E\norm{X - Y}_2 \mid (X, Y) \text{ is a coupling of } (\mu, \nu) \}. \label{eq:wasserstein_distance}
\end{align}
The following result uses the smoothness of the Gaussian transition kernel to upper bound the TV norm via the
$1$-Wasserstein distance. 
This result is inspired by the work of \citet{chae2020wasserstein}.
\begin{lemma}
\label{prop:wasserstein_bounds_tv_when_gaussian}
Let $X_0, Y_0$ be random vectors in $\R^p$,
and let $f : \R^p \to \R^n$ be an $L$-Lipschitz function.
Suppose that $X_0, Y_0$ are both absolutely continuous w.r.t.\ the
Lebesgue measure on $\R^p$.
Let $\Sigma \in \R^{n \times n}$ be positive definite,
and let $X_1, Y_1$ be random vectors in $\R^n$ defined 
conditionally: $X_1 \mid X_0 = N(f(X_0), \Sigma)$
and $Y_1 \mid Y_0 = N(f(Y_0), \Sigma)$.
Then:
\begin{align*}
    \tvnorm{\sfP_{X_1} - \sfP_{Y_1}} \leq \frac{L \sqrt{\tr(\Sigma^{-1})}}{2} W_1(\sfP_{X_0}, \sfP_{Y_0}).
\end{align*}
\end{lemma}
\begin{proof}
Since $X_0, Y_0$ are absolutely continuous, the Radon-Nikodym theorem ensures that there exists
densities $p_0, q_0$ for $X_0, Y_0$, respectively.
Let $\phi$ denote the density of 
the $N(0, \Sigma)$ distribution.
Let $p_1, q_1$ denote the densities of $X_1, Y_1$, respectively.
We have the following convolution expressions:
\begin{align*}
    p_1(x) &= \int \phi(x - f(x_0)) p_0(x_0) \,\rmd x_0 = \int \phi(x - f(X_0)) \,\rmd X_0, \\
    q_1(x) &= \int \phi(x - f(x_0)) q_0(x_0) \,\rmd x_0 = \int \phi(x - f(Y_0)) \,\rmd Y_0.
\end{align*}
Now, let $\pi$ be a coupling of $(X_0, Y_0)$.
We can equivalently write $p_1, q_1$ as a convolution over $\pi$:
\begin{align*}
    p_1(x) &= \int \phi(x - f(X_0)) \,\rmd\pi(X_0, Y_0), \\
    q_1(x) &= \int \phi(x - f(Y_0)) \,\rmd\pi(X_0, Y_0).
\end{align*}
Hence:
\begin{align*}
    (p_1-q_1)(x) &= \int [ \phi(x - f(X_0)) - \phi(x - f(Y_0))] \,\rmd\pi(X_0, Y_0) \\
    &= \E_{\pi(X_0, Y_0)}[ \phi(x - f(X_0)) - \phi(x - f(Y_0))].
\end{align*}
Now by the $L^1$ representation of total-variation norm~\citep[see e.g.][Lemma 2.1]{tsybakov2009introduction}:
\begin{align*}
    \tvnorm{\sfP_{X_1} - \sfP_{Y_1}} &= \frac{1}{2} \int \abs{p_1(x) - q_1(x)} \,\rmd x \\
    &= \frac{1}{2} \int \abs{ \E_{\pi(X_0, Y_0)}[ \phi(x - f(X_0)) - \phi(x - f(Y_0))]} \,\rmd x \\
    &\stackrel{(a)}{\leq} \frac{1}{2} \int \E_{\pi(X_0, Y_0)} [\abs{\phi(x - f(X_0)) - \phi(x - f(Y_0))}] \,\rmd x \\
    &\stackrel{(b)}{=} \frac{1}{2} \E_{\pi(X_0, Y_0)} \left[\int \abs{\phi(x - f(X_0)) - \phi(x - f(Y_0))} \,\rmd x\right].
\end{align*}
The inequality (a) is Jensen's inequality,
and the equality (b) is Tonelli's theorem
since the integrand is non-negative.
We now focus on the inner integral
inside the expectation over $\pi$.
By the mean-value theorem, since $\phi$ is continuously differentiable, fixing
$x, X_0, Y_0$:
\begin{align*}
    \abs{\phi(x - f(X_0)) - \phi(x - f(Y_0))} &= \bigabs{\int_{0}^{1} \nabla \phi( (1-s)(x - f(Y_0)) + s (x - f(X_0)))^\T (f(Y_0) - f(X_0)) \,\rmd s} \\
    &\leq \norm{f(Y_0) - f(X_0)}_2 \int_0^1 \norm{ \nabla \phi(x - (s f(X_0) - (1-s) f(Y_0)) }_2 \,\rmd s.
\end{align*}
Hence, by another application of Tonelli's theorem:
\begin{align*}
    &\int \abs{\phi(x - f(X_0)) - \phi(x - f(Y_0))} \,\rmd x \\
    &\leq \norm{f(X_0) - f(Y_0)}_2 \int \int_0^1 \norm{ \nabla \phi(x - (s f(X_0) - (1-s) f(Y_0)) }_2 \,\rmd s \,\rmd x \\
    &= \norm{f(X_0) - f(Y_0)}_2 \int_0^1 \int \norm{ \nabla \phi(x - (s f(X_0) - (1-s) f(Y_0)) }_2 \,\rmd x \,\rmd s \\
    &\leq \norm{f(X_0) - f(Y_0)}_2 \sqrt{\tr(\Sigma^{-1})}.
\end{align*}
The last inequality follows from the following computation.
Observe that $\nabla \phi(x) = - \Sigma^{-1} x \phi(x)$
and define $\mu = s f(X_0) - (1-s) f(Y_0)$.
Since $\mu$ does not depend on $x$, by the translation invariance
of the Lebesgue integral:
\begin{align*}
    \int \norm{ \nabla \phi(x - (s f(X_0) - (1-s) f(Y_0)) }_2 \,\rmd x 
    &= \int \norm{\Sigma^{-1}(x - \mu) }_2 \phi(x - \mu) \,\rmd x \\
    &= \int \norm{\Sigma^{-1} x}_2 \phi(x) \,\rmd x \\
    &= \E_{x \sim N(0, \Sigma^{-1})}[ \norm{x}_2 ] \\
    &\leq \sqrt{\tr(\Sigma^{-1})}.
\end{align*}
The last inequality above is another application of Jensen's inequality.
Therefore, combining the inequalities thus far,
and using the $L$-Lipschitz property of $f$:
\begin{align*}
    \tvnorm{\sfP_{X_1} - \sfP_{Y_1}} &\leq \frac{\sqrt{\tr(\Sigma^{-1})}}{2} \E_{\pi(X_0, Y_0)}\left[ \norm{f(X_0) - f(Y_0)}_2  \right] \\
    &\leq \frac{L \sqrt{\tr(\Sigma^{-1})}}{2} \E_{\pi(X_0, Y_0)}[ \norm{X_0 - Y_0}_2 ].
\end{align*}
Since the coupling $\pi$ of $(X_0, Y_0)$ was arbitrary, the result now follows by taking the infimum of the right hand side
over all valid couplings.
\end{proof}




\section{Recovering \citet{ziemann2022single} via boundedness}
\label{sec:appendix:bdd}
Here  we show how to recover the results for mixing systems from \citet[Theorem 1 combined with Proposition 2]{ziemann2022single}, corresponding to $\alpha=1$. This rests on the observation that $(B^2,1)$-hypercontractivity is automatic by $B$-boundedness.
\begin{corollary}
\label{corr:alpha1corr}
Suppose that $\mathscr{F}_\star$ is star-shaped and $B$-bounded. Fix also $p \in \mathbb{R}_+$ and $q\in (0,2)$ and suppose further that $\mathscr{F}_\star$ satisfies condition (\ref{eq:nonparametric_simple}). Then we have that:
\begin{multline}
\label{eq:alpha1corrstmt}
    \E \|\widehat f - f_\star\|_{L^2}^2 \leq 8 \E \mathsf{M}_T(\mathscr{F}_\star)   +\frac{1}{\sqrt{8}} \left( 16B^2 p \opnorm{ \Gamma_{\mathsf{dep}}(\mathsf{P}_X)}^2 \right)^{2/q}T^{-2/(2+q)}\\+\exp \left(\frac{ -T^{q/(2+q)} }{16B^2 \opnorm{ \Gamma_{\mathsf{dep}}(\mathsf{P}_X)}^2 } \right). 
\end{multline}
\end{corollary}
The first two terms in inequality (\ref{eq:alpha1corrstmt}) are both of order $T^{-2/(2+q)}$ if $\opnorm{\Gamma_{\mathsf{dep}}}^2=O(1)$. Note that, without further control of the moments of $f\in \mathscr{F}_\star$, the bound in Theorem~\ref{thm:themainthm} thus degrades by a factor of the dependency matrix through the second term. 

\paragraph{Proof of Corollary~\ref{corr:alpha1corr}}
Fix $c>0$ to be determined later and choose $r=c T^{-1/(2+q)}$. We find:
\begin{equation*}
    \begin{aligned}
    & B^2 \mathcal{N}_\infty(\mathscr{F}_\star,r/\sqrt{8}) \exp \left( \frac{-T r^{2} }{8B^2 \opnorm{ \Gamma_{\mathsf{dep}}(\mathsf{P}_X)}^2 } \right)\\
     & \leq \exp\left(p \left(\frac{\sqrt{8}}{r}\right)^q - \frac{T r^{2} }{8B^2 \opnorm{ \Gamma_{\mathsf{dep}}(\mathsf{P}_X)}^2 }  \right) &&(\textnormal{Condition~(\ref{eq:nonparametric_simple})})\\
     &= \exp\left( \left[p \left(\frac{\sqrt{8}}{c}\right)^q  - \frac{1}{8B^2 \opnorm{ \Gamma_{\mathsf{dep}}(\mathsf{P}_X)}^2 }\right] T^{q/(2+q)}  \right). && (r=c T^{-1/(2+q)})
    \end{aligned}
\end{equation*}
Hence we may solve for $c = \frac{1}{\sqrt{8}} \left( 16B^2 p \opnorm{ \Gamma_{\mathsf{dep}}(\mathsf{P}_X)}^2 \right)^{1/q}$ in
\begin{align*}
    p \left(\frac{\sqrt{8}}{c}\right)^q  = \frac{1}{16B^2 \opnorm{ \Gamma_{\mathsf{dep}}(\mathsf{P}_X)}^2 } 
\end{align*}
to arrive at the desired conclusion.\hfill $\blacksquare$

\section{Linear dynamical systems}
\label{sec:appendix:lds}

We define the truncated linear dynamics:
\begin{align}
    \bar{X}_{t+1} = A_\star \bar{X}_t + H \bar{V}_t, \:\: \bar{X}_0 = H \bar{V}_0, \:\: \bar{V}_t = V_t \ind\{\norm{V_t}_2 \leq R\}. \label{eq:LDS_dynamics_truncated}
\end{align}
We set $R = \sqrt{\dx} + \sqrt{2(1+\beta)\log{T}}$ where $\beta > 4$ is a free parameter.
Define the event $\calE$ as:
\begin{align}
    \calE := \left\{ \max_{0 \leq t \leq T-1} \norm{V_t}_2 \leq R \right\}. \label{eq:LDS_good_truncation_event}
\end{align}
Note that by the setting of $R$, we have $\Pr(\calE^c) \leq 1/T^{\beta}$ using standard Gaussian concentration results plus a union bound. Furthermore on $\calE$, the original GLM process
driven by Gaussian noise \eqref{eq:LDS_dynamics}
coincides with the truncated process \eqref{eq:LDS_dynamics_truncated}.
Let $\widehat{f}$ denote the LSE on the original process
\eqref{eq:LDS_dynamics_truncated}, and
let $\bar{f}$ denote the LSE on the truncated process \eqref{eq:LDS_dynamics_truncated}.
Hence:
\begin{align*}
    \E\norm{\widehat{f} - f_\star}_{L^2}^2 &= \E \norm{\widehat{f} - f_\star}_{L^2}^2 \ind\{\calE\} + \E \norm{\widehat{f} - f_\star}_{L^2}^2 \ind\{\calE^c\} \\
    &\leq \E\norm{\bar{f} - f_\star}_{L^2}^2 + \E \norm{\widehat{f} - f_\star}_{L^2}^2 \ind\{\calE^c\}.
\end{align*}
Let us now control the error term 
$\E \norm{\widehat{f} - f_\star}_{L^2}^2 \ind\{\calE^c\}$.
Since $X_t$ is a linear function of the Gaussian noise $\{W_t\}$ process, by \Cref{prop:gaussian_fourth_moment_bound} we have
$\E\norm{X_t}_2^4 \leq 3 (\E\norm{X_t}_2^2)^2$.
Write $\widehat{f}(x) = \widehat{A} x$, and put $\widehat{\Delta} = \widehat{A}-A_\star$.
We have:
\begin{align}
    \E \norm{\widehat{f} - f_\star}_{L^2}^2 \ind\{\calE^c\} &= \frac{1}{T} \sum_{t=0}^{T-1} \E\norm{\widehat{\Delta} X_t}_2^2 \ind\{\calE^c\} \stackrel{(a)}{\leq} \frac{4B^2}{T} \sum_{t=0}^{T-1} \E\norm{X_t}_2^2\ind\{\calE^c\} \nonumber \\
    &\stackrel{(b)}{\leq}  \frac{4B^2}{T^{1+\beta/2}} \sum_{t=0}^{T-1} \sqrt{ \E\norm{X_t}_2^4} 
    \stackrel{(c)}{\leq} \frac{4\sqrt{3} B^2}{T^{1+\beta/2}} \sum_{t=0}^{T-1} \E\norm{X_t}_2^2 \nonumber \\
    &= \frac{4\sqrt{3} B^2}{T^{1+\beta/2}} \sum_{t=0}^{T-1} \tr(\Gamma_t) 
    \stackrel{(d)}{\leq} \frac{4\sqrt{3}B^2 \tr(\Gamma_{T-1})}{T^{\beta/2}} 
    \stackrel{(e)}{\leq} \frac{4\sqrt{3}B^2 \opnorm{H}^2 \tau^2 \dx}{(1-\rho) T^{\beta/2}}. \label{eq:lds_fhat_error_term}
\end{align}
Above, (a) follows from the definition 
of $\scrF$,
(b) follows from Cauchy-Schwarz,
(c) uses the hypercontractivity bound 
$\E\norm{X_t}_2^4 \leq 3 (\E\norm{X_t}_2^2)^2$,
(d) uses the fact that $\Gamma_t$ is monotonically
increasing in the Loewner order,
and (e) uses the following bound on $\tr(\Gamma_{T-1})$ using the $(\tau,\rho)$-stability of $A_\star$:
\begin{align*}
    \tr(\Gamma_{T-1}) \leq \frac{\opnorm{H}^2 \tau^2 \dx}{1-\rho^2} \leq \frac{\opnorm{H}^2 \tau^2 \dx}{1-\rho}.
\end{align*}

The remainder of the proof is to bound
the error of the LSE $\bar{f}$ using \Cref{thm:themainthm}.
This involves two main steps:
showing the trajectory hypercontractivity condition
\Cref{def:hypconrelax} holds, and
bounding the dependency matrix
$\opnorm{\Gammadep(\Pxbar)}$ (cf.~\Cref{def:depmat}), where $\Pxbar$
denotes the joint distribution of the process $\{\bar{X}_t\}_{t=0}^{T-1}$.
Before we proceed, we define some reoccurring constants:
\begin{align}
    \mu \triangleq \lambda_{\min}(\Gamma_{\kappa - 1}), \:\: B_{\bar{X}} \triangleq \frac{\opnorm{H} \tau (\sqrt{\dx} + \sqrt{2(1+\beta)\log{T}})}{1-\rho}. \label{eq:LDS_constants}
\end{align}

\subsection{Trajectory hypercontractivity for truncated LDS}
\begin{proposition}
\label{prop:LDS_traj_hypercontractivity}
Suppose that $T \geq \max\{6, 2\kappa\}$.
The pair $(\scrF_\star, \Pxbar)$ with $\scrF$ given in \eqref{eq:LDS_function_class}
and $\Pxbar$ as the joint distribution of $\{\bar{X}_t\}_{t=0}^{T-1}$ from \eqref{eq:LDS_dynamics_truncated}
satisfies the $(C_{\mathsf{LDS}}, 2)$-trajectory hypercontractivity condition with $C_{\mathsf{LDS}}= \frac{108\tau^4 \opnorm{H}^4}{(1-\rho)^2 \mu^2}$.
\end{proposition}
\begin{proof}
Fix any size-conforming matrix $M$.
Let the noise process $\{\bar{V}_t\}_{t=0}^{T-1}$
be stacked into a noise vector $\bar{V}_{0:T-1} \in \R^{\dx T}$.
Observe that we can write $M X_t = M T_t \bar{V}_{0:T-1}$
for some matrix $T_t$.
We invoke the comparison inequality in \Cref{prop:gaussian_truncate_fourth_moment} 
followed by the Gaussian fourth moment identity in \Cref{prop:gaussian_fourth_moment_bound} to conclude that:
\begin{align*}
    \E\norm{M \bar{X}_t}_2^4 = \E\norm{M T_t \bar{V}_{0:T-1}}_2^4 \leq \E\norm{M X_t}_2^4 \leq 3 (\E\norm{M X_t}_2^2)^2 = 3 \tr(M^\T M \Gamma_t)^2.
\end{align*}
By monotonicity of $\Gamma_t$ and the assumption $T \geq 6$:
\begin{align}
    \frac{1}{T} \sum_{t=0}^{T-1} \Gamma_t \succcurlyeq \frac{1}{T}\sum_{t=\floor{T/2}}^{T-1} \Gamma_t \succcurlyeq  \frac{T-\floor{T/2}}{T} \Gamma_{\floor{T/2}} \succcurlyeq \frac{1}{3} \Gamma_{\floor{T/2}}. \label{eq:LDS_traj_hyp_lower_avg_gram}
\end{align}
Since $T \geq 2\kappa$, the inequality $\Gamma_{\floor{T/2}} \succcurlyeq \Gamma_{\kappa-1}$ holds, and therefore
$\Gamma_{\floor{T/2}}$ is invertible since $(A_\star, H)$ is
$\kappa$-step controllable.
Therefore:
\begin{align*}
    \frac{1}{T}\sum_{t=0}^{T-1} \E\norm{M \bar{X}_t}_2^4 &\leq 3 \tr(M^\T M \Gamma_{T-1})^2 \\
    &= 3 \tr( M \Gamma_{\floor{T/2}}^{1/2} \Gamma_{\floor{T/2}}^{-1/2} \Gamma_{T-1} \Gamma_{\floor{T/2}}^{-1/2} \Gamma_{\floor{T/2}}^{1/2} M^\T  )^2 \\
    &\leq 3 \opnorm{ \Gamma_{\floor{T/2}}^{-1}  \Gamma_{T-1} }^2\tr(M^\T M \Gamma_{\floor{T/2}})^2 \\
    &\leq 27\opnorm{ \Gamma_{\floor{T/2}}^{-1}  \Gamma_{T-1} }^2 \tr\left( M^\T M \cdot \frac{1}{T} \sum_{t=0}^{T-1} \Gamma_t \right)^2 &&\text{using \eqref{eq:LDS_traj_hyp_lower_avg_gram}} \\
    &= 27\opnorm{ \Gamma_{\floor{T/2}}^{-1}  \Gamma_{T-1} }^2 \left( \frac{1}{T} \sum_{t=0}^{T-1} \E\norm{M X_t}_2^2 \right)^2 \\
    &\leq 108 \opnorm{ \Gamma_{\floor{T/2}}^{-1}  \Gamma_{T-1} }^2 \left( \frac{1}{T} \sum_{t=0}^{T-1} \E\norm{M \bar{X}_t}_2^2 \right)^2 &&\text{using \Cref{prop:trunc_gauss_approximate_isotropic}}.
\end{align*}
Since the matrix $M$ is arbitrary, the claim follows
using the following bound for $\opnorm{\Gamma_{\floor{T/2}}^{-1} \Gamma_{T-1}}^2$:
\begin{align*}
    \opnorm{\Gamma_{\floor{T/2}}^{-1} \Gamma_{T-1}}^2 \leq \frac{\tau^4 \opnorm{H}^4}{(1-\rho^2)^2 \mu^2} \leq \frac{\tau^4 \opnorm{H}^4}{(1-\rho)^2 \mu^2}.
\end{align*}
\end{proof}

\subsection{Bounding the dependency matrix for truncated LDS}

We control $\opnorm{\Gammadep(\Pxbar)}$ by a direct computation of the mixing
properties of the original Gaussian process \eqref{eq:LDS_dynamics}.
\begin{proposition}
\label{prop:LDS_ergodicity}
Consider the process $\{\bar{X}_t\}_{t \geq 0}$ from \eqref{eq:LDS_dynamics_truncated}, and let $\Pxbar$ denote the
joint distribution of $\{\bar{X}_t\}_{t=0}^{T-1}$.
We have that:
\begin{align*}
     \opnorm{\Gammadep(\Pxbar)} \leq 5\kappa + \frac{22}{1-\rho} \log\left(
     \frac{\tau^2}{4\mu}\left[ B_{\bar{X}}^2 + \frac{\dx \opnorm{H}^2}{1-\rho}\right]\right).
\end{align*}
\end{proposition}
\begin{proof}
We first construct an almost sure bound on the process
$\{\bar{X}_t\}_{t \geq 0}$. Indeed, for any $t \geq 0$,
using the $(\tau,\rho)$-stability of $A_\star$:
\begin{align*}
    \norm{\bar{X}_t}_2 \leq \frac{\opnorm{H} \tau R}{1-\rho} = \frac{\opnorm{H} \tau (\sqrt{\dx} + \sqrt{2(1+\beta)\log{T}})}{1-\rho} = B_{\bar{X}}.
\end{align*}
Also, by $(\tau,\rho)$-stability, we have for any indices $s \leq t$:
\begin{align}
    \opnorm{ \Gamma_s - \Gamma_t } = \bigopnorm{\sum_{k=s+1}^{t} A^k HH^\T (A^k)^\T } \leq \opnorm{H}^2 \tau^2 \sum_{k=s+1}^{t} \rho^{2k} \leq \frac{\opnorm{H}^2 \tau^2}{1-\rho^2} \rho^{2(s+1)}. \label{eq:LDS_covariance_ineq}
\end{align}
The marginal and conditional distributions of $\{X_t\}_{t \geq 0}$ are easily characterized.
We have that $X_t \sim N(0, \Gamma_t)$.
Furthermore, $X_{t}\mid X_0=x$ for $t \geq 1$ is distributed as $N(A^t x, \Gamma_{t-1})$.
So now for any $t \geq 0$ and $k \geq 1$:
\begin{align*}
    \sfP_{X_{t+k}}(\cdot \mid X_t=x) = N(A^k x, \Gamma_{k-1}), \:\: \sfP_{X_{t+k}} = N(0, \Gamma_{t+k}).
\end{align*}
Now suppose $k \geq \kappa$.
The matrices $\Gamma_{k-1}$ and $\Gamma_{t+k}$ will
both be invertible, so the two distributions are mutually absolutely continuous.
We can then use the closed-form expression for the KL-divergence between two multivariate Gaussians:
\begin{align*}
    &\mathsf{KL}(N(A^k x, \Gamma_{k-1}), N(0, \Gamma_{t+k})) \\
    &= \frac{1}{2}\left[ \tr(\Gamma_{t+k}^{-1} \Gamma_{k-1}) + x^\T (A^k)^\T \Gamma_{t+k}^{-1} A^k x - \dx + \log\det( \Gamma_{t+k} \Gamma_{k-1}^{-1} )  \right] \\
    &\leq \frac{1}{2} x^\T (A^k)^\T \Gamma_{t+k}^{-1} A^k x + \frac{\dx}{2}\log\opnorm{\Gamma_{t+k} \Gamma_{k-1}^{-1}} && \text{since } \Gamma_{k-1} \preccurlyeq \Gamma_{t+k} \\
    &\leq \frac{\tau^2 \rho^{2k} \norm{x}_2^2}{2\mu} + \frac{\dx}{2}\log\left(1 + \frac{\opnorm{\Gamma_{t+k} - \Gamma_{k-1}}}{\mu}\right) &&\text{using $(\tau,\rho)$-stability} \\
    &\leq \frac{\tau^2 \rho^{2k} \norm{x}_2^2}{2\mu} + \frac{\dx}{2}\log\left(1 + \frac{\opnorm{H}^2 \tau^2}{(1-\rho^2)\mu} \rho^{2k} \right) &&\text{using \eqref{eq:LDS_covariance_ineq}} \\
    &\leq \left[\frac{\tau^2 \norm{x}_2^2}{2\mu} + \frac{\dx \opnorm{H}^2 \tau^2}{2 (1-\rho^2) \mu} \right]\rho^{2k} &&\log(1+x) \leq x \,\forall x \geq 0.
\end{align*}
Hence by Pinsker's inequality~\citep[see e.g.][Lemma 2.5]{tsybakov2009introduction}, whenever $k \geq \kappa$:
\begin{align*}
    \tvnorm{ \sfP_{X_{t+k}}(\cdot \mid X_t=x) - \sfP_{X_{t+k}} } &\leq \sqrt{\mathsf{KL}(N(A^k x, \Gamma_{k-1}), N(0, \Gamma_{t+k}))/2} \\
    &\leq \sqrt{\frac{\tau^2 \norm{x}_2^2}{4\mu} + \frac{\dx \opnorm{H}^2 \tau^2}{4 (1-\rho^2) \mu} }\rho^{k}.
\end{align*}
By \Cref{prop:final_gamma_bound} (which we can invoke since we constrained $\beta \geq 2$), for any $\ell \in \N$:
\begin{align*}
    \opnorm{\Gammadep(\sfP_{\bar{X}})} &\leq 3 + \sqrt{2} \sum_{k=1}^{T-1} \max_{t=0,\dots,T-1-k} \esssup_{x \in \bar{\sfX}_t} \sqrt{\tvnorm{\sfP_{X_{t+k}}(\cdot \mid X_t=x) - \sfP_{X_{t+k}}}} \\
    &\leq 3 + \sqrt{2} (\kappa-1+\ell) + \sum_{k=\kappa+\ell}^{T-1}\left[\frac{\tau^2 B_{\bar{X}}^2}{4\mu} + \frac{\dx \opnorm{H}^2 \tau^2}{4 (1-\rho^2) \mu} \right]^{1/4} \rho^{k/2} \\
    &\leq 5(\kappa+\ell) + \left[\frac{\tau^2 B_{\bar{X}}^2}{4\mu} + \frac{\dx \opnorm{H}^2 \tau^2}{4 (1-\rho^2) \mu} \right]^{1/4} \frac{\rho^{(\kappa+\ell)/2}}{1-\rho^{1/2}}.
\end{align*}
Now, define $\psi \triangleq \frac{\tau^2 B_{\bar{X}}^2}{4\mu} + \frac{\dx \opnorm{H}^2 \tau^2}{4 (1-\rho^2) \mu}$.
We choose $\ell= \max\left\{\bigceil{\frac{\log(\psi^{1/4})}{1-\rho^{1/2}}}-\kappa, 0\right\}$,
so $\rho^{(\kappa+\ell)/2} \leq 1/\psi^{1/4}$.
With this choice of $\ell$
and the observation that $\inf_{x \in [0,1]} \frac{1-\sqrt{x}}{1-x} = \frac{1}{2}$,
\begin{align*}
    \opnorm{\Gammadep(\Pxbar)} \leq 5\kappa + \frac{11\log{\psi}}{4(1-\rho^{1/2})} \leq 5\kappa + \frac{22\log{\psi}}{1-\rho}.
\end{align*}
The claim now follows.
\end{proof}

\subsection{Finishing the proof of \Cref{prop:ldsprop}}

For what follows, $c_i$ will denote universal positive constants whose values remain unspecified.

For any $\e>0$ and $r>0$, we now construct an $\e$-covering of $\partial B(r)$ with
$\scrF_\star$ the offset class of $\scrF$ from \eqref{eq:LDS_function_class}. To this end, we let $\{A_1,\dots,A_N\}$ be a $\delta$-cover of $\scrA \triangleq \{A \in \mathbb{R}^{\dx \times \dx} \mid \| A\|_F \leq B \}$
for $\delta$ to be specified. By a volumetric argument we may choose $\{A_1,\dots,A_N\}$ such that $N \leq \left(1 +\frac{2B}{\delta} \right)^{\dx^2}$. 
Now, any realization of $\{\bar{X}_t\}$ will have norm less than
$B_{\bar{X}}$, where $B_{\bar{X}}$ is given by \eqref{eq:LDS_constants} and satisfies
\begin{align*}
    B_{\bar{X}} \leq c_0\frac{\opnorm{H} \tau (\sqrt{\dx} + \sqrt{(1+\beta)\log{T}})}{1-\rho}.
\end{align*}
Let $A \in \scrA$, and let $A_i$ denote
an element in the covering satisfying $\norm{A - A_i}_F \leq \delta$. For any $x$ satisfying $\norm{x}_2 \leq B_{\bar{X}}$:
\begin{align*}
    \norm{ (A_i x - A_\star x) - (A x - A_\star x) }_F = \norm{(A_i - A)x}_2 \leq \norm{A_i-A}_F\norm{x}_2 \leq \delta B_{\bar{X}}.
\end{align*}
Thus, it suffices to take $\delta = \e/B_{\bar{X}}$
to construct an $\e$-covering of $\scrF_\star$ over $\{\bar{X}_t\}$, which shows that
$\calN_\infty(\scrF_\star, \e) \leq \left(1 + \frac{2BB_{\bar{X}}}{\e}\right)^{\dx^2}$.
Since $\partial B(r) \subset \scrF_\star$, we have the following inequality~\citep[see e.g.][Exercise 4.2.10]{vershynin2018}:
\begin{align*}
    \calN_\infty(\partial B(r), \e) \leq \calN_\infty(\scrF_\star, \e/2) \leq \left(1 + \frac{4B B_{\bar{X}}}{\e}\right)^{\dx^2}.
\end{align*}
By \Cref{prop:LDS_traj_hypercontractivity}, $(\scrF_\star, \Pxbar)$ is $(C_{\mathsf{LDS}},2)$-hypercontractive for all $T \geq \max\{6,2\kappa\}$, with
\begin{align*}
    C_{\mathsf{LDS}} = \frac{108 \tau^4 \opnorm{H}^4}{(1-\rho)^2\mu^2}.
\end{align*}
Also by \Cref{prop:LDS_ergodicity}, 
\begin{align*}
    \opnorm{\Gammadep(\Pxbar)}^2 \leq c_1\kappa^2 + \frac{c_2}{(1-\rho)^2} \log^2\left(  \frac{\tau^2}{4\mu}\left[ B_{\bar{X}}^2 + \frac{\dx \opnorm{H}^2}{1-\rho}\right] \right) \triangleq \gamma^2.
\end{align*}
Since $\scrF_\star$ is convex and contains the zero function, it is also star-shaped.
Furthermore, on the truncated process \eqref{eq:LDS_dynamics_truncated},
the class $\scrF_\star$ is $2BB_{\bar{X}}$-bounded.
Invoking \Cref{thm:themainthm}, we thus have for every $r>0$ that
\begin{align}
    \E \|\bar{f}- f_\star\|_{L^2}^2 \leq 8 \E \bar{\sfM}_T(\mathscr{F}_\star)   +r^2+ 4B^2B_{\bar{X}}^2  \left(1 +\frac{4\sqrt{8} B B_{\bar{X}}}{r} \right)^{\dx^2} \exp \left( \frac{-T}{8C_{\mathsf{LDS}} \gamma^2 } \right). \label{eq:lds_truncated_LSE}
\end{align}
Here, the notation $\E \bar{\sfM}_T(\scrF_\star)$ is meant to emphasize
that the offset complexity is with respect to the truncated process
$\Pxbar$ and \emph{not} the original process $\Px$.
We now set $r^2 = \opnorm{H}^2 \dx^2/T$,
and compute a $T_0$ such that the third term in \eqref{eq:lds_truncated_LSE} is also bounded by $\opnorm{H}^2 \dx^2/T$.
To do this, it suffices to compute $T_0$ such that for all
$T \geq T_0$:
\begin{align*}
    T \geq c_3 C_{\mathsf{LDS}} \gamma^2 \dx^2 \log\left( \frac{T B B_{\bar{X}}}{\opnorm{H} \sqrt{\dx}} \right).
\end{align*}
Thus it suffices to set $T_0$ as (provided that $\beta$ is at most polylogarithmic in the problem constants---we later make such a choice):
\begin{align}
    T_0 = c_4 \frac{\tau^4 \opnorm{H}^4 \dx^2}{(1-\rho)^2\mu^2}\left[ \kappa^2 + \frac{1}{(1-\rho)^2}\right] \mathrm{polylog}\left( B, \dx, \tau, \opnorm{H}, \frac{1}{\mu}, \frac{1}{1-\rho},\right).
    \label{eq:lds_T0}
\end{align}
We do not attempt to compute the exact power of the 
polylog term; it can in principle be done via 
\citet[Lemma F.2]{du2021bilinear}.

Next, by \eqref{eq:lds_fhat_error_term},
$\E \norm{\widehat{f} - f_\star}_{L^2}^2 \ind\{\calE^c\} \leq \frac{4\sqrt{3}B^2 \opnorm{H}^2 \tau^2 \dx}{(1-\rho) T^{\beta/2}}$. Thus we also need to set $T_0$ large enough so that this term is bounded by $\opnorm{H}^2 \dx^2/T$. To do this, it suffices to constrain $\beta > 2$ and set
$T_0 \geq c_5 \left[ \frac{B^2 \tau^2}{1-\rho} \right]^{\frac{1}{\beta/2-1}}$. Hence, setting $\beta = \max\{4, c_6 \log{B}\}$ implies that 
\eqref{eq:lds_T0} suffices.

Let us now upper bound $\E \bar{\sfM}_T(\scrF_\star)$ by
$\E \sfM_T(\scrF_\star)$ plus $\opnorm{H}^2\dx^2/T$.
Recall the definition of $\calE$ from \eqref{eq:LDS_good_truncation_event}. 
We first write:
\begin{align*}
    \E \bar{\sfM}_T(\scrF_\star) &= \E \bar{\sfM}_T(\scrF_\star) \ind\{\calE\} + \E \bar{\sfM}_T(\scrF_\star) \ind\{\calE^c\} \\
    &= \E \sfM_T(\scrF_\star) \ind\{\calE\} + \E \bar{\sfM}_T(\scrF_\star) \ind\{\calE^c\} \\
    &\leq \E \sfM_T(\scrF_\star) + \E \bar{\sfM}_T(\scrF_\star) \ind\{\calE^c\}.
\end{align*}
The last inequality holds since it can be checked that
$\sfM_T(\scrF_\star) \geq 0$ (i.e., by lower bounding the supremum with the zero function which is in $\scrF_\star$ 
since $\scrF$ contains $f_\star$).
Furthermore, an elementary linear algebra calculation yields that we can upper bound $\bar{\sfM}_T(\scrF_\star)$ deterministically by:
\begin{align*}
    \bar{\sfM}_T(\scrF_\star) \leq \frac{4}{T} \bignorm{ \left(\left( \sum_{t=0}^{T-1} \bar{X}_t\bar{X}_t^\T \right)^{\dag}\right)^{1/2} \sum_{t=0}^{T-1} \bar{X}_t \bar{V}_t^\T H^\T }_F^2 \leq \frac{4}{T} \sum_{t=0}^{T-1} \norm{H \bar{V}_t}^2_2
\end{align*}
Here, the $\dag$ notation refers to the Moore-Penrose pseudo-inverse.
Therefore taking expectations:
\begin{align*}
    \E \bar{M}_T(\scrF_\star) \ind\{\calE^c\} &\leq \frac{4}{T} \sum_{t=0}^{T-1} \E \norm{H \bar{V}_t}_2^2 \ind\{\calE^c\} \leq \frac{4}{T} \sum_{t=0}^{T-1} \E \norm{H V_t}_2^2 \ind\{\calE^c\} \stackrel{(a)}{\leq} \frac{4}{T^{1 + \beta/2}} \sum_{t=0}^{T-1} \sqrt{\E\norm{H V_t}_2^4} \\
    &\stackrel{(b)}{\leq} \frac{4\sqrt{3}}{T^{1+\beta/2}} \sum_{t=0}^{T-1} \E\norm{H V_t}_2^2 = \frac{4\sqrt{3}\norm{H}_F^2}{T^{\beta/2}} \leq \frac{4\sqrt{3} \dx \opnorm{H}^2}{T^{\beta/2}}. 
\end{align*}
Here, (a) is Cauchy-Schwarz, and (b) follows from \Cref{prop:gaussian_fourth_moment_bound}.
This last term will be bounded by $\opnorm{H}^2 \dx^2/T$
as soon as $T \geq 4\sqrt{3}$, since we set $\beta \geq 4$.
The claim now follows. \hfill $\blacksquare$

\section{General linearized model dynamics}
\label{sec:appendix:glm}

\subsection{Comparison to existing results}
\label{sec:appendix:glm:comparison}
We first compare the results of \Cref{thm:glmthm} to
the existing bounds from \citet{sattar2020non,foster2020learning,kowshik2021near}.
Before doing so, we note that these existing results bound the loss
of specific gradient based algorithms.
On the other hand, \Cref{thm:glmthm} 
directly applies to the empirical risk minimizer of the square loss.
In general, the LSE optimization problem specialized to this setting is non-convex due to the composition of the square loss with the link function $\sigma$.
However, we believe it should be possible to show that the quasi-Newton method described
in \citet[Algorithm 1]{kowshik2021near} can be used to optimize the empirical
risk to precision of order $\opnorm{H}^2 \dx^2/T$, in which case a 
simple modification of \Cref{thm:themainthm} combined with the current analysis
in \Cref{thm:glmthm} would apply to bound the
excess risk of the final iterate of this algorithm. 
This is left to future work.


For our comparison, we will ignore all logarithmic factors, and assume any necessary 
burn-in times, remarking that the existing results all 
prescribe sharper burn-in times than \Cref{thm:glmthm}.
First, we compare with \citet[Corollary 6.2]{sattar2020non}.
In doing so, we will assume that $H = (1+\sigma) I$ for some $\sigma > 0$,
since this is the setting they study.
When $H$ is diagonal, 
\eqref{eq:GLM_parameter_rate} is actually invariant to the noise scale $\sigma$ 
which is the correct behavior: 
$\E\norm{\widehat{A}-A_\star}_F^2 \leq \tilde{O}(1) \frac{\dx^2}{\zeta^2 T}$.
On the other hand, \citet[Corollary 6.2]{sattar2020non} gives a high probability bound of
$\norm{\widehat{A} - A_\star}_F^2 \leq \tilde{O}(1) \frac{\sigma^2 \dx^2}{\zeta^4 (1-\rho)^3 T}$. Thus, \eqref{eq:GLM_parameter_rate} improves on this rate by not only a factor of $1/\zeta^2$, but also in moving the $1/(1-\rho)$ dependence into the log.
We note that their result seems to improve as $\sigma \to 0$, but
the probability of success also tends to $0$ as $\sigma \to 0$.

Next, we turn our attention to \citet[Theorem 2, fast rate]{foster2020learning}.
This result actually gives both in-sample excess risk and parameter recovery bounds.
For simplicity, we only compare to the parameter recovery bounds, as this is their sharper result. Their result yields a high probability bound that
$\norm{\widehat{A} - A_\star}_F^2 \leq \tilde{O}(1) \frac{\opnorm{H}^2 \opnorm{P_\star} \dx^2}{\zeta^4 (1-\rho) T}$. Again, we see \eqref{eq:GLM_parameter_rate} improve this rate by a factor of $1/\zeta^2$, and moves the dependence on $\opnorm{P_\star}$ and
$1/(1-\rho)$ into the logarithm. 
We note again, this rate seems to improve as $\opnorm{H} \to 0$, but the number of iterations $m$ of GLMtron needed tends to $\infty$ as
$\opnorm{H} \to 0$.
We conclude by noting that the rate of \citet{foster2020learning} does not 
have any burn-in times.

Finally, we compare to \citet[Theorem 1]{kowshik2021near}. We will assume that
$H = \sigma I$ again, as this is the setting of their work.
Their parameter recovery bound states that with high probability,
$\norm{\widehat{A}-A_\star}_F^2 \leq \tilde{O}(1) \frac{\sigma^2 \dx^2}{\zeta^2 T}$, which matches \eqref{eq:GLM_parameter_rate} up to the log factors. As noted previously, their logarithmic dependencies are sharper than ours. Furthermore, their result can also handle
the unstable regime when $\rho \leq 1 + O(1/T)$, which ours cannot.
However, we note that \Cref{thm:glmthm} also bounds $L^2$ excess risk
with logarithmic dependence on $1/(1-\rho)$, which is not an immediate consequence of
parameter error bounds. Indeed, a na{\"{i}}ve upper bound using the $1$-Lipschitz property of the link function yields:
$\norm{\widehat{f}-f_\star}_{L^2}^2 \leq \opnorm{\widehat{A}-A_\star}^2 \frac{1}{T}\sum_{t=0}^{T-1}\E\norm{X_t}_2^2 \lesssim \opnorm{\widehat{A}-A_\star}^2 \frac{1}{(1-\rho)^2}$.
Hence, even if the parameter error only depends logarithmically on $1/(1-\rho)$, 
it does not immediately translate over to excess risk.


\subsection{Proof of \Cref{thm:glmthm}}
We first turn our attention to controlling the
states $\{X_t\}_{t \geq 0}$ in expectation.
Note that the Lyapunov assumption in \Cref{assumption:glm}
implies that for every $x \in \R^{\dx}$:
\begin{align}
    \norm{\sigma(A_\star x)}^2_{P_\star} \leq \rho \norm{x}^2_{P_\star} \label{eq:glm_one_step_contraction},
\end{align}
and hence the function $x \mapsto \norm{x}_{P_\star}^2$ is a Lyapunov function (recall that $P_\star \succcurlyeq I$)
which certifies exponential stability to the origin
for the deterministic system $x_+ = \sigma(A_\star x)$.
\begin{proposition}
\label{prop:glm_X_fourth_moment}
Consider the GLM process $\{X_t\}_{t \geq 0}$ from \eqref{eq:GLM_dynamics}.
Under \Cref{assumption:glm}:
\begin{align}
    \sup_{t \in \N} \E\norm{X_t}_2^4 \leq B_X^4, \:\: B_X \triangleq \frac{12 \sqrt{2} \opnorm{H} \opnorm{P_\star}^{1/2} \sqrt{\dx}}{1-\rho}. \label{eq:glm_B_X}
\end{align}
\end{proposition}
\begin{proof}
For any $a,b \in \R$ and $\e > 0$,
$(a+b)^4 \leq (1+\e)^3 a^4 + (1+1/\e)^3 b^4$. 
Hence for any $\e > 0$:
\begin{align}
    \E\norm{X_t}_{P_\star}^4 &= \E\norm{\sigma(A_\star X_t) + HV_t}_{P_\star}^4 \nonumber \\
    &\leq (1+\e)^3 \E\norm{\sigma(A_\star X_t)}_{P_\star}^4 + (1+1/\e)^3 \E\norm{HV_t}_{P_\star}^4 \nonumber \\
    &\leq (1+\e)^3 \rho^2 \E\norm{X_t}_{P_\star}^4 + (1+1/\e)^3 \E\norm{HV_t}_{P_\star}^4 &&\text{using \eqref{eq:glm_one_step_contraction}}. \label{eq:glm_fourth_moment_step_a}
\end{align}
Now, we first assume that $\rho \in [1/2, 1)$.
For any $\e \in [0, 1]$, we have 
that $(1+\e)^3 \leq 1 + 12\e$.
Choosing $\e = \frac{1-\rho^2}{24\rho^2}$, we have that $\e \leq 1$, and therefore continuing from \eqref{eq:glm_fourth_moment_step_a}:
\begin{align*}
    \E\norm{X_t}_{P_\star}^4 &\leq (1+12\e)\E\norm{X_t}_{P_\star}^4 + (1+1/\e)^3 \E\norm{HV_t}_{P_\star}^4 \\
    &=\frac{1 + \rho^2}{2} \E\norm{X_t}_{P_\star}^4 + \frac{24^3}{(1-\rho^2)^3}\E\norm{HV_t}_{P_\star}^4 \\
    &\leq \frac{1 + \rho^2}{2} \E\norm{X_t}_{P_\star}^4 + \frac{3 \cdot 24^3}{(1-\rho^2)^3}(\E\norm{HV_t}_{P_\star}^2)^2 &&\text{using \Cref{prop:gaussian_fourth_moment_bound}} \\
    &\leq \frac{1 + \rho^2}{2} \E\norm{X_t}_{P_\star}^4 + \frac{3 \cdot 24^3}{(1-\rho^2)^3}\tr(H^\T P_\star H)^2.
\end{align*}
Unrolling this recursion yields:
\begin{align*}
    \E\norm{X_t}_{P_\star}^4 \leq \frac{6 \cdot 24^3}{(1-\rho^2)^4}\tr(H^\T P_\star H)^2.
\end{align*}
We now handle the case when $\rho \in [0, 1/2)$.
Setting $\e = 2^{1/3}-1$ and starting from \eqref{eq:glm_fourth_moment_step_a}:
\begin{align*}
    \E\norm{X_t}_{P_\star}^4 &\leq \frac{1}{2}\E\norm{X_t}_{P_\star}^4 + 125 \E\norm{H V_t}_{P_\star}^4 \\
    &\leq \frac{1}{2} \E\norm{X_t}_{P_\star}^4 + 375 \tr(H^\T P_\star H)^2.
\end{align*}
Unrolling this recursion yields:
\begin{align*}
    \E\norm{X_t}^4_{P_\star} \leq 750\tr(H^\T P_\star H)^2.
\end{align*}
The claim now follows by taking the maximum of these two bounds
and using the inequalities
$\tr(H^\T P_\star H) \leq \opnorm{H}^2 \opnorm{P_\star} \dx$ and
$1-\rho^2 \geq 1-\rho$.
\end{proof}

This proof proceeds quite similarly to the linear dynamical
systems proof given in \Cref{sec:appendix:lds}.
We start again by defining the truncated GLM dynamics:
\begin{align}
    \bar{X}_{t+1} = \sigma(A_\star \bar{X}_t) + H \bar{V}_t, \:\: \bar{X}_0 = H \bar{V}_0, \:\: \bar{V}_t = V_t \ind\{\norm{V_t}_2 \leq R\}. \label{eq:GLM_dynamics_truncated}
\end{align}
We set $R = \sqrt{\dx} + \sqrt{2(1+\beta)\log{T}}$ where $\beta \geq 2$ is a free parameter.
Define the event $\calE$ as:
\begin{align}
    \calE := \left\{ \max_{0 \leq t \leq T-1} \norm{V_t}_2 \leq R \right\}. \label{eq:GLM_good_truncation_event}
\end{align}
Note that by the setting of $R$, we have $\Pr(\calE^c) \leq 1/T^{\beta}$ using standard Gaussian concentration results plus a union bound. Furthermore on $\calE$, the original GLM process
driven by Gaussian noise \eqref{eq:GLM_dynamics}
coincides with the truncated process \eqref{eq:GLM_dynamics_truncated}.
Let $\widehat{f}$ denote the LSE on the original process
\eqref{eq:GLM_dynamics_truncated}, and
let $\bar{f}$ denote the LSE on the truncated process \eqref{eq:GLM_dynamics_truncated}.
Hence:
\begin{align}
    \E\norm{\widehat{f} - f_\star}_{L^2}^2 &= \E \norm{\widehat{f} - f_\star}_{L^2}^2 \ind\{\calE\} + \E \norm{\widehat{f} - f_\star}_{L^2}^2 \ind\{\calE^c\} \nonumber \\
    &\leq \E\norm{\bar{f} - f_\star}_{L^2}^2 + \E \norm{\widehat{f} - f_\star}_{L^2}^2 \ind\{\calE^c\} \label{eq:glm_truncate_decomp}.
\end{align}
Let us now control the error term 
$\E \norm{\widehat{f} - f_\star}_{L^2}^2 \ind\{\calE^c\}$.
Write $\widehat{f}(x) = \sigma(\widehat{A} x)$, and put $\widehat{\Delta} = \widehat{A}-A_\star$.
We have:
\begin{align}
    \E \norm{\widehat{f} - f_\star}_{L^2}^2 \ind\{\calE^c\} &= \frac{1}{T} \sum_{t=0}^{T-1} \E\norm{\sigma(\widehat{A} X_t) - \sigma(A_\star X_t) }_2^2 \ind\{\calE^c\} \stackrel{(a)}{\leq} \frac{1}{T}\sum_{t=0}^{T-1} \E\norm{\widehat{\Delta} X_t}_2^2\ind\{\calE^c\} \nonumber \\
    &\stackrel{(b)}{\leq} \frac{4B^2}{T} \sum_{t=0}^{T-1} \E\norm{X_t}_2^2 \ind\{\calE^c\} 
    \stackrel{(c)}{\leq} \frac{4B^2}{T^{1 + \beta/2}}\sum_{t=0}^{T-1} \sqrt{\E\norm{X_t}_2^4} 
    \stackrel{(d)}{\leq} \frac{4B^2 B_X^2}{T^{\beta/2}}. \label{eq:glm_fhat_lse_error_term}
\end{align}
Here, (a) follows since $\sigma$ is $1$-Lipschitz,
(b) uses the definition of $\scrF$ in \eqref{eq:GLM_function_class}, (c) follows by 
Cauchy-Schwarz, and (d) uses \Cref{prop:glm_X_fourth_moment}.

The remainder of the proof is to bound the
LSE error $\E\norm{\bar{f} - f_\star}_{L^2}^2$.
First, we establish an almost sure bound on
$\{\bar{X}_t\}_{t \geq 0}$.
%
\begin{proposition}
\label{prop:glm_truncated_state_bounds}
Consider the truncated GLM process \eqref{eq:GLM_dynamics_truncated}.
Under \Cref{assumption:glm}, the process $\{\bar{X}_t\}_{t \geq 0}$ satisfies:
\begin{align}
    \sup_{t \in \N} \norm{\bar{X}_t}_{P_\star} \leq \frac{2\opnorm{P_\star}^{1/2} \opnorm{H} (\sqrt{\dx} + \sqrt{2(1+\beta)\log{T}})}{1-\rho} \triangleq B_{\bar{X}}. \label{eq:glm_B_Xbar}
\end{align}
\end{proposition}
\begin{proof}
By triangle inequality and \eqref{eq:glm_one_step_contraction}:
\begin{align*}
    \norm{\bar{X}_{t+1}}_{P_\star} &= \norm{ \sigma(A_\star \bar{X}_t) + H \bar{V}_t }_{P_\star} \leq \norm{\sigma(A_\star \bar{X}_t)}_{P_\star} + \norm{H \bar{V}_t}_{P_\star} \\
    &\leq \rho^{1/2} \norm{\bar{X}_t}_{P_\star} + \norm{H \bar{V}_t}_{P_\star} \leq \rho^{1/2} \norm{\bar{X}_t}_{P_\star} + \opnorm{P_\star^{1/2} H} R.
\end{align*}
Unrolling this recursion,
and using the fact that $\inf_{x \in [0, 1]} \frac{1-\sqrt{x}}{1-x} = 1/2$ yields the result.
\end{proof}

We next establish uniform bounds for the covariance matrices
of the truncated process.
\begin{proposition}
\label{prop:glm_trunc_cov_bounds}
Suppose $T \geq 4$.
Consider the truncated GLM process \eqref{eq:GLM_dynamics_truncated},
and let the covariance matrices for the process $\{\bar{X}_t\}_{t \geq 0}$ be denoted as $\bar{\Gamma}_t \triangleq \E[ \bar{X}_t \bar{X}_t^\T ]$.
Under \Cref{assumption:glm}:
\begin{align*}
    \frac{1}{2} HH^\T \preccurlyeq \bar{\Gamma}_t \preccurlyeq B_{\bar{X}}^2 \cdot I.
\end{align*}
\end{proposition}
\begin{proof}
The upper bound is immediate from \Cref{prop:glm_truncated_state_bounds},
since $\E[\bar{X}_t\bar{X}_t^\T] \preccurlyeq \E[\norm{\bar{X}_t}_2^2] I \preccurlyeq B_{\bar{X}}^2 I$.
For the lower bound, it is immediate when $t=0$
using \Cref{prop:trunc_gauss_approximate_isotropic}.
On the other hand, for $t \geq 1$, since $\bar{V}_t$ is zero-mean:
\begin{align*}
    \E[ \bar{X}_t \bar{X}_t^\T] &= \E[(\sigma(A_\star \bar{X}_{t-1}) + H\bar{V}_{t-1})(\sigma(A_\star \bar{X}_{t-1}) + H\bar{V}_{t-1})^\T] \\
    &= \E[\sigma(A_\star \bar{X}_{t-1}) \sigma(A_\star \bar{X}_{t-1})^\T] + \E[ H \bar{V}_{t-1} \bar{V}_{t-1}^\T H^\T] 
    \succcurlyeq \E[ H \bar{V}_{t-1} \bar{V}_{t-1}^\T H^\T] \succcurlyeq \frac{1}{2}HH^\T.
\end{align*}
The last inequality again holds from \Cref{prop:trunc_gauss_approximate_isotropic}.
\end{proof}

\subsubsection{Trajectory hypercontractivity for truncated GLM}

For our purposes, the link function assumption in
\Cref{assumption:glm} ensures the following
approximate isometry inequality which holds for all $x \in \R^{\dx}$ and all
matrices $A,A' \in \R^{\dx \times \dx}$:
\begin{align}
    \zeta^2 \norm{Ax-A'x}_2^2 \leq \norm{\sigma(Ax)-\sigma(A'x)}_2^2 \leq \norm{Ax-A'x}_2^2. \label{eq:glm_approx_isometry}
\end{align}
This inequality is needed to establish trajectory hypercontractivity for $\scrF_\star$.
\begin{proposition}
\label{prop:glm_traj_hyp}
Suppose that $T \geq 4$.
Fix any matrix $A \in \R^{\dx \times \dx}$.
Under \Cref{assumption:glm}, the truncated process \eqref{eq:GLM_dynamics_truncated} satisfies:
\begin{align}
    \frac{1}{T} \sum_{t=0}^{T-1} \E\norm{\sigma(A \bar{X}_t) - \sigma(A_\star \bar{X}_t)}_2^4 \leq \frac{4B_{\bar{X}}^4}{\sigma_{\min}(H)^4 \zeta^4}  \left(\frac{1}{T}\sum_{t=0}^{T-1}\E\norm{\sigma(A \bar{X}_t) - \sigma(A_\star \bar{X}_t)}_2^2\right)^2.
\end{align}
Hence, the function class $\scrF_\star$ with $\scrF$ defined in
\eqref{eq:GLM_function_class} satisfies the $(C_{\mathsf{GLM}}, 2)$-trajectory
hypercontractivity condition
with $C_{\mathsf{GLM}} = \frac{4B_{\bar{X}}^4}{\sigma_{\min}(H)^4 \zeta^4}$.
\end{proposition}
\begin{proof}
Put $\Delta \triangleq A - A_\star$ and $M \triangleq \Delta^\T \Delta$.
We have:
\begin{align*}
    \E \norm{\Delta \bar{X}_t}_2^4 &= \E[ \bar{X}_t^\T M \bar{X}_t \bar{X}_t^\T M \bar{X}_t ] \\
    &\leq B_{\bar{X}}^2 \tr(M^2 \bar{\Gamma}_t) &&\text{using \Cref{prop:glm_truncated_state_bounds}} \\
    &\leq B_{\bar{X}}^2 \opnorm{M} \tr(M \bar{\Gamma}_t) &&\text{H{\"{o}}lder's inequality}\\
    &\leq B_{\bar{X}}^2 \tr(M) \tr(M \bar{\Gamma}_t) &&\text{since $M$ is positive semidefinite} \\
    &\leq B_{\bar{X}}^4 \tr(M)^2 &&\text{using \Cref{prop:glm_trunc_cov_bounds}} \\
    &\leq \frac{B_{\bar{X}}^4}{\lambda_{\min}(HH^\T)^2} \tr(M HH^\T)^2.
\end{align*}
On the other hand, by \Cref{prop:glm_trunc_cov_bounds}:
\begin{align*}
    \E\norm{\Delta \bar{X}_t}_2^2 = \tr(M \bar{\Gamma}_t) \geq \frac{1}{2}\tr(M HH^\T).
\end{align*}
Combining these bounds yields:
\begin{align*}
    \frac{1}{T}\sum_{t=0}^{T-1}\E\norm{\Delta \bar{X}_t}_2^4 \leq \frac{B_{\bar{X}}^4}{\lambda_{\min}(HH^\T)^2} \tr(M HH^\T)^2 \leq \frac{4B_{\bar{X}}^4}{\lambda_{\min}(HH^\T)^2}   \left( \frac{1}{T}\sum_{t=0}^{T-1} \E\norm{\Delta \bar{X}_t}_2^2 \right)^2.
\end{align*}
The claim now follows via the approximate isometry inequality \eqref{eq:glm_approx_isometry}.
\end{proof}

\subsubsection{Bounding the dependency matrix for truncated GLM}

We will use the result in \Cref{prop:wasserstein_bounds_tv_when_gaussian}
to bound the total-variation distance by the $1$-Wasserstein distance.
This is where the non-degenerate noise assumption in \Cref{assumption:glm}
is necessary.

The starting point is the observation that the diagonal Lyapunov function
in \Cref{assumption:glm} actually yields
\emph{incremental stability}~\citep{tran2016incremental} in addition to Lyapunov stability. In particular, let $\{a_i\}$ denote the rows of $A_\star$.
For any $x,x' \in \R^{\dx}$:
\begin{align}
    \norm{\sigma(A_\star x) - \sigma(A_\star x')}^2_{P_\star} &= \sum_{i=1}^{\dx} (P_\star)_{ii} (\sigma(\ip{a_i}{x}) - \sigma(\ip{a_i}{x'}))^2 \nonumber \\
    &\leq \sum_{i=1}^{\dx} (P_\star)_{ii} (\ip{a_i}{x} - \ip{a_i}{x'})^2 \nonumber \\
    &= (x-x')^\T A_\star^\T P_\star A_\star (x-x') \nonumber \\
    &\leq \rho \norm{x-x'}_{P_\star}^2. \label{eq:glm_incremental_stability}
\end{align}

This incremental stability property allows us to control the dependency matrix as follows.
\begin{proposition}
\label{prop:glm_ergodicity}
Consider the truncated GLM process $\{\bar{X}_t\}_{t \geq 0}$
from \eqref{eq:GLM_dynamics_truncated}.
Let $\Pxbar$ denote the joint distribution of $\{\bar{X}_t\}_{t=0}^{T-1}$.
Under \Cref{assumption:glm} and when $B \geq 1$,
we have that:
\begin{align*}
    \opnorm{\Gammadep(\Pxbar)} \leq \frac{22}{1-\rho} \log\left( \frac{B \sqrt{\dx}(B_{\bar{X}} + B_X)}{2\sigma_{\min}(H)}\right).
\end{align*}
\end{proposition}
\begin{proof}
Let $\{X_t\}_{t \geq 0}$ denote the original GLM dynamics from \eqref{eq:GLM_dynamics}.
Fix indices $t \geq 0$ and $k \geq 1$.
We construct a coupling of $(\sfP_{X_{t+k}}(\cdot \mid X_t=x), \sfP_{X_{t+k}})$ as follows.
Let $\{V_t\}_{t \geq 0}$ be \iid\ $N(0, I)$.
Let $\{Z_s\}_{s \geq t}$ be the process such that
$Z_{t} = x$, and follows the GLM dynamics \eqref{eq:GLM_dynamics} using the noise $\{V_t\}_{t \geq 0}$
(we do not bother defining $Z_{t'}$ for $t' < t$ since we do not need it).
Similarly, let $\{Z'_s\}_{s \geq 0}$ be the process following the GLM dynamics \eqref{eq:GLM_dynamics}
using the same noise $\{V_t\}_{t \geq 0}$.
Now we have:
\begin{align*}
    \E\norm{Z_{t+k} - Z'_{t+k}}_{P_\star} &= \E\norm{\sigma(A_\star Z_{t+k-1}) - \sigma(A_\star Z'_{t+k-1})}_{P_\star} \\
    &\leq \rho^{1/2} \E\norm{Z_{t+k-1} - Z'_{t+k-1}}_{P_\star} &&\text{using \Cref{eq:glm_incremental_stability}}.
\end{align*}
We now unroll this recursion down to $t$:
\begin{align*}
    \E\norm{Z_{t+k} - Z'_{t+k}}_{P_\star} \leq \rho^{k/2} \E\norm{Z_t - Z'_t}_{P_\star} = \rho^{k/2} \E\norm{x - Z'_t}_{P_\star}.
\end{align*}
Since $P_\star \succcurlyeq I$, this shows that:
\begin{align*}
    W_1(\sfP_{X_{t+k}}(\cdot \mid X_t = x), \sfP_{X_{t+k}}) \leq \rho^{k/2} (\norm{x}_{P_\star} + \E\norm{X_{t}}_{P_\star}) \leq \rho^{k/2}(\norm{x}_{P_\star} + B_X),
\end{align*}
where the last inequality follows from \Cref{prop:glm_X_fourth_moment} and Jensen's inequality.

Next, it is easy to see the map $x \mapsto \sigma(A_\star x)$
is $\opnorm{A}$-Lipschitz.
Furthermore, since $H$ is full rank
by \Cref{assumption:glm}, then for any $t$ and $k \geq 1$ both
$\sfP_t$ and 
$\sfP_{X_{t+k}}(\cdot \mid X_t=x)$
are absolutely continuous
w.r.t.\ the Lebesgue measure in $\R^{\dx}$.
Using \Cref{prop:wasserstein_bounds_tv_when_gaussian}, we have for any 
$k \geq 2$:
\begin{align*}
    \tvnorm{\sfP_{X_{t+k}}(\cdot \mid X_t = x) - \sfP_{X_{t+k}}} &\leq \frac{\opnorm{A_\star} \sqrt{\tr((HH^\T)^{-1})}}{2}  W_1(\sfP_{X_{t+k-1}}(\cdot \mid X_t = x), \sfP_{X_{t+k-1}}) \\
    &\leq \frac{\opnorm{A_\star} \sqrt{\tr((HH^\T)^{-1})}}{2} \rho^{(k-1)/2} (\norm{x}_{P_\star} + B_X).
\end{align*}
Using \Cref{prop:final_gamma_bound} to bound $\opnorm{\Gammadep(\Pxbar)}$ (which is valid because we constrained $\beta \geq 2$),
and \Cref{prop:glm_truncated_state_bounds} to bound $x \in \bar{\sfX}_t$,
for any $\ell \geq 1$:
\begin{align*}
    \opnorm{\Gammadep(\sfP_{\bar{X}})} &\leq 3 + \sqrt{2} \sum_{k=1}^{T-1} \max_{t=0,\dots,T-1-k} \esssup_{x \in \bar{\sfX}_t} \sqrt{\tvnorm{\sfP_{X_{t+k}}(\cdot \mid X_t=x) - \sfP_{X_{t+k}}}} \\
    &\leq 3 + \sqrt{2}\ell + \left[ \frac{\opnorm{A_\star} \sqrt{\tr((HH^\T)^{-1})} (B_{\bar{X}} + B_X)}{2} \right]^{1/2} \sum_{k=\ell+1}^{T-1} \rho^{(k-1)/4} \\
    &\stackrel{(a)}{\leq} 5\ell + \left[ \frac{B \sqrt{\dx} (B_{\bar{X}} + B_X)}{2\sigma_{\min}(H)} \right]^{1/2} \frac{\rho^{\ell/4}}{1-\rho^{1/4}}.
\end{align*}
Above, (a) uses the bounds $\opnorm{A_\star} \leq B$ and 
$\tr(HH^{-1}) \leq \dx/\sigma_{\min}(H)^2$.
Now put $\psi \triangleq \frac{B \sqrt{\dx} (B_{\bar{X}} + B_X)}{2\sigma_{\min}(H)}$.
We choose $\ell = \bigceil{\frac{\log(\sqrt{\psi})}{1-\rho^{1/4}}}$ so that
$\rho^{\ell/4} \leq 1/\sqrt{\psi}$.
This yields:
\begin{align*}
    \opnorm{\Gammadep(\sfP_{\bar{X}})} &\leq \frac{11\log{\psi}}{2(1-\rho^{1/4})} \stackrel{(a)}{\leq} \frac{22 \log{\psi}}{1-\rho} = \frac{22}{1-\rho} \log\left( \frac{B \sqrt{\dx}(B_{\bar{X}} + B_X)}{2\sigma_{\min}(H)}\right).
\end{align*}
Above, (a) follows from 
$\inf_{x \in [0,1]} \frac{1-x^{1/4}}{1-x} = 1/4$.
\end{proof}

\subsubsection{Finishing the proof of \Cref{thm:glmthm}}

Below, we let $c_i$ be universal positive constants
that we do not track precisely.

For any $\e>0$ we now construct an $\e$-covering of $\scrF_\star \setminus B(r)$, with $\scrF_\star$ the offset class of $\scrF$
from \eqref{eq:GLM_function_class}. 
Note that we are not covering $\partial B(r)$ since the class
$\scrF_\star$ is not star-shaped. However, an inspection
of the proof of \Cref{thm:lucemthm} shows that 
one can remove the star-shaped assumption by instead
covering the set $\scrF_\star \setminus B(r)$.
To this end, we let $\{A_1,\dots,A_N\}$ be a $\delta$-cover of $\scrA \triangleq \{A \in \mathbb{R}^{\dx \times \dx} \mid \|A\|_F \leq B \}$,
for a $\delta$ to be specified.
By a volumetric argument we may choose $\{A_1,\dots,A_N\}$ such that $N \leq \left(1 +\frac{2B}{\delta} \right)^{\dx^2}$. 
Now, any realization of $\{\bar{X}_t\}$ will have norm less than
$B_{\bar{X}}$ from \eqref{eq:glm_B_Xbar},
where $B_{\bar{X}}$ is bounded by:
\begin{align*}
    B_{\bar{X}} \leq \frac{c_0\opnorm{P_\star}^{1/2} \opnorm{H} (\sqrt{\dx} + \sqrt{(1+\beta)\log{T}})}{1-\rho}.
\end{align*}
Now fix any $A \in \scrA$,
and let $A_i$ be an element in the $\delta$-cover satisfying $\norm{A - A_i}_F \leq \delta$.
We observe that for any $x$ satisfying $\norm{x}_2 \leq B_{\bar{X}}$:
\begin{align*}
    \norm{(\sigma(A_ix) - \sigma(A_\star x)) - (\sigma(A x) - \sigma(A_\star x))}_2 &= \norm{\sigma(A_i x) - \sigma(A x)}_2 \leq \norm{(A_i - A) x}_2 \\
    &\leq \norm{A_i-A}_F\norm{x}_2 \leq \delta B_{\bar{X}}.
\end{align*}
Thus, it suffices to take $\delta = \e /B_{\bar{X}}$
to construct the $\e$ cover of $\scrF_\star$,
i.e., $\calN_\infty(\scrF_\star, \e) \leq \left(1 +\frac{2BB_{\bar{X}}}{\e}\right)^{\dx^2}$. This then implies~\citep[see e.g.][Example 4.2.10]{vershynin2018}:
\begin{align*}
    \calN_\infty(\scrF_\star \setminus B(r), \e) \leq \calN_\infty(\scrF_\star, \e/2) \leq \left(1 +\frac{4BB_{\bar{X}}}{\e}\right)^{\dx^2}.
\end{align*}
Next, by \Cref{prop:glm_traj_hyp}, $(\scrF_\star, \Pxbar)$ is $(C_{\mathsf{GLM}},2)$-hypercontractive for all $T \geq 4$,
where 
\begin{align*}
    C_{\mathsf{GLM}} \leq \frac{4B_{\bar{X}}^4}{\sigma_{\min}(H)^4 \zeta^4} \leq \frac{c_1 \opnorm{P_\star}^2 \mathrm{cond}(H)^4 (\dx^2 + ((1+\beta)\log{T})^2) }{  \zeta^4 (1-\rho)^4 }.
\end{align*}
Furthermore, by \Cref{prop:glm_ergodicity}:
\begin{align*}
    \opnorm{\Gammadep(\Pxbar)}^2 \leq \frac{c_2}{(1-\rho)^2} \log^2\left( \frac{B \sqrt{\dx}(B_{\bar{X}} + B_X)}{2\sigma_{\min}(H)}\right) \triangleq \gamma^2.
\end{align*}
The class $\scrF_\star$ is $2BB_{\bar{X}}$-bounded on \eqref{eq:GLM_dynamics_truncated}.
Invoking \Cref{thm:themainthm}, for every $r>0$:
\begin{align}
    \E \|\bar{f}- f_\star\|_{L^2}^2 \leq 8 \E \bar{\sfM}_T(\mathscr{F}_\star) + r^2 + 4B^2B_{\bar{X}}^2 \left(1 +\frac{4\sqrt{8} B B_{\bar{X}}}{r} \right)^{\dx^2} \exp \left( \frac{-T  }{8C_{\mathsf{GLM}} \gamma^2 } \right). \label{eq:glm_truncated_LSE}
\end{align}
Here, the notation $\E \bar{\sfM}_T(\scrF_\star)$ is meant to emphasize
that the offset complexity is with respect to the truncated process
$\Pxbar$ and \emph{not} the original process $\Px$.
We now set $r^2 = \opnorm{H}^2 \dx^2/T$,
and compute a $T_0$ such that the third term in \eqref{eq:glm_truncated_LSE} is also bounded by $\opnorm{H}^2 \dx^2/T$.
To do this, it suffices to compute $T_0$ such that for all
$T \geq T_0$:
\begin{align*}
    T \geq c_3 C_{\mathsf{GLM}} \gamma^2 \dx^2 \log\left( \frac{T B B_{\bar{X}}}{\opnorm{H} \sqrt{\dx}} \right).
\end{align*}
It suffices to take (assuming $\beta$ is at most polylogarithmic in any problem constants):
\begin{align}
    T_0 \geq c_4
    \frac{\opnorm{P_\star}^{2} \mathrm{cond}(H)^4  \dx^4}{\zeta^4 (1-\rho)^6} \mathrm{polylog}\left( B, \dx, \opnorm{P_\star}, \mathrm{cond}(H), \frac{1}{\zeta}, \frac{1}{1-\rho} \right).
    \label{eq:glm_T0_bound}
\end{align}
Again, we do not attempt to compute the exact power of the 
polylog term, but note it can in principle be done via 
\citet[Lemma F.2]{du2021bilinear}.

Next, from \eqref{eq:glm_fhat_lse_error_term} we have that
the error term 
$\E \norm{\widehat{f} - f_\star}_{L^2}^2 \ind\{\calE^c\} \leq \frac{4B^2 B_X^2}{T^{\beta/2}}$.
Thus if we further constrain $\beta > 2$ and require $T_0 \geq c_5 \left[ \frac{B^2 \opnorm{P_\star}}{(1-\rho)^2} \right]^{\frac{1}{\beta/2-1}}$,
then 
$\E \norm{\widehat{f} - f_\star}_{L^2}^2 \ind\{\calE^c\} \leq \frac{\opnorm{H}^2 \dx^2}{T}$.
Note that setting $\beta = \max\{3, c_6 \log{B}\}$ suffices.

To finish the proof, it remains to bound $\E \bar{\sfM}_T(\scrF_\star)$.
Now, unlike the linear dynamical systems case,
there is no closed-form expression for $\E \bar{M}_T(\scrF_\star)$.
Hence, we will bound it via the chaining bound \eqref{eq:theformofthebound}. Before we can use the result, however, we need to verify that
the truncated noise process $\{H \bar{V}_t\}_{t \geq 0}$ is 
a sub-Gaussian MDS. The MDS part is clear since $\bar{V}_{t} \perp \bar{V}_{t'}$ for $t \neq t'$, and $\bar{V}_t$ is zero-mean.
Furthermore, \Cref{prop:truncated_sub_gaussian_bound}
yields that $H \bar{V}_t$ is a $4\opnorm{H}^2$-sub-Gaussian random vector. To finish the proof we need to bound the martingale offset complexity.

\paragraph{Bounding the Martingale Offset Complexity}
\label{sec:boundoffsetglm}

We will use the chaining bound \eqref{eq:theformofthebound} in particular. Thus,  we need to bound the covering number of $\scrF_\star$. Define
\begin{align*}
\mathbb{M}_{\dx,\dx}(B) = \{A \in \mathbb{R}^{\dx\times \dx} \mid \| A \|_F \leq B \}.
\end{align*}
By a standard volumetric argument we have that $\log \mathcal{N}(\mathbb{M}_{\dx,\dx}, \| \cdot \|_F , \e) \leq    \dx^2 \log (1+2B/\e)$. Let now $\{A_1,\dots, A_N\}$ be an optimal $\e$-cover of $\mathbb{M}_{\dx,\dx}$. Then for every $A \in \mathbb{M}_{\dx,\dx}$  we can find $A_i \in \{A_1,\dots A_N\}$ such that
\begin{align*}
\| \sigma (A \: \cdot \:) -\sigma(A_i \: \cdot \:)\|_\infty \leq  \sup_{x\in \sfX} \|(A - A_i)x\|_2 \leq  \|A - A_i\|_F \sup_{x\in \sfX} \| x\|_2 \leq  \e B_{\bar X}
\end{align*}
where we used the estimate $B_{\bar X}$ of \eqref{eq:glm_B_Xbar} and the fact that $\sigma$ is $1$-Lipschitz.

Hence any $\e$-covering of $\mathbb{M}_{\dx,\dx}(B)$ induces a $B_{\bar X}\e$-covering of $\scrF_\star$ and we  have established the upper bound
\begin{equation}
\begin{aligned}
\label{eq:glm_coverest}
 \log \mathcal{N}(\scrF_\star, \|\cdot\|_\infty, \e) &\leq \log \mathcal{N}(\mathbb{M}_{\dx,\dx},  \| \cdot \|_F , \e/B_{\bar X}B)\\
 &\leq    \dx^2 \log (1+2B_{\bar X}B/\e).
\end{aligned}
\end{equation}

Since the truncated process $H\bar V_t$ is $4\opnorm{H}$ sub-Gaussian, we may invoke \eqref{eq:theformofthebound} with $\delta=\gamma$. For any $\gamma>0$ we have for some constant $c'>0$:

\begin{equation}
    \begin{aligned}\label{eq:glm_martest}
      \E \bar{\sfM}_T(\scrF_\star)& \leq c' \left(\frac{ \opnorm{H}\log \mathcal{N}_\infty(\scrF_\star,\gamma)}{T}  + \gamma \sqrt{\dx\opnorm{H}}+\gamma^2\right)\\
      &\leq  c' \left(\frac{ \opnorm{H} \dx^2 \log (1+2B_{\bar X}B/\gamma)}{T}  + \gamma\sqrt{\dx\opnorm{H}} +\gamma^2\right)
    \end{aligned}
\end{equation}
where we substituted the estimate \eqref{eq:glm_coverest} into \eqref{eq:theformofthebound}. To arrive at the desired conclusion \eqref{eq:GLM_final_rate}, it suffices to take $\gamma^2 =  \frac{\dx \opnorm{H}}{T}$ and impose $T\geq T_0$.

\end{document}